\documentclass[twoside,11pt]{article}

\usepackage[preprint]{jmlr2e}

\usepackage{graphicx}
\usepackage{amsmath}
\usepackage{amssymb}
\usepackage{booktabs}
\usepackage{tabularx}
\usepackage{float}
\usepackage{multirow}
\usepackage{subcaption}
\usepackage{placeins}
\usepackage{enumitem}
\usepackage[ruled,vlined,noline]{algorithm2e}
\usepackage{xcolor}
\usepackage{mathrsfs}

\usepackage{lastpage}
\jmlrheading{XX}{2025}{1-\pageref{LastPage}}{11/25; Revised X/25}{X/25}{25-0000}{Rui Wu, Lizheng Wang, and Yongjun Li}
\ShortHeadings{The Causal Round Trip}{Wu, Wang, and Li}
\firstpageno{1}

\newtheorem{assumption}{Assumption}

\begin{document}

\title{The Causal Round Trip: Generating Authentic Counterfactuals by Eliminating Information Loss}

\author{
    \name Rui Wu \email wurui22@mail.ustc.edu.cn \\
    \addr School of Management, University of Science and Technology of China \\
    96 Jinzhai Road, Hefei, 230026, Anhui, China
    \AND
    \name Lizheng Wang \email lzwang@ustc.edu.cn \\
    \addr School of Management, University of Science and Technology of China \\
    96 Jinzhai Road, Hefei, 230026, Anhui, China
    \AND
    \name Yongjun Li \thanks{Corresponding author.} \email lionli@ustc.edu.cn \\
    \addr School of Management, University of Science and Technology of China \\
    96 Jinzhai Road, Hefei, 230026, Anhui, China
}

\editor{Editor Name}

\maketitle

\begin{abstract}%
    Judea Pearl's vision of Structural Causal Models (SCMs) as engines for counterfactual reasoning hinges on faithful abduction: the precise inference of latent exogenous noise. For decades, operationalizing this step for complex, non-linear mechanisms has remained a significant computational challenge. The advent of diffusion models, powerful universal function approximators, offers a promising solution. However, we argue that their standard design, optimized for perceptual generation over logical inference, introduces a fundamental flaw for this classical problem: an inherent information loss we term the \textbf{Structural Reconstruction Error (SRE)}. To address this challenge, we formalize the principle of \textbf{Causal Information Conservation (CIC)} as the necessary condition for faithful abduction. We then introduce \textbf{BELM-MDCM}, the first diffusion-based framework engineered to be causally sound by eliminating SRE by construction through an \textbf{analytically invertible} mechanism. To operationalize this framework, a \textbf{Targeted Modeling} strategy provides structural regularization, while a \textbf{Hybrid Training Objective} instills a strong causal inductive bias. Rigorous experiments demonstrate that our Zero-SRE framework not only achieves state-of-the-art accuracy but, more importantly, enables the high-fidelity, individual-level counterfactuals required for deep causal inquiries. Our work provides a foundational blueprint that reconciles the power of modern generative models with the rigor of classical causal theory, establishing a new and more rigorous standard for this emerging field.
\end{abstract}


\section{Introduction}\label{sec:intro}

The fundamental challenge of causal inference, as articulated by \citet{rubin1974estimating}, is our inability to simultaneously observe an individual's potential outcomes. Generating authentic counterfactuals is thus the field's grand challenge. Structural Causal Models (SCMs), introduced by \citet{pearl2009causality}, provide the formal language for this pursuit. An SCM posits that an outcome $V_i$ is generated by a function of its parents $\mathbf{Pa}_i$ and a unique exogenous noise variable $U_i$. This noise, $U_i$, represents the \textbf{primordial causal information}—the collection of unobserved factors unique to an individual. This concept aligns directly with the long-standing focus in econometrics on \textbf{unobserved individual heterogeneity}, a central challenge in structural modeling for decades \citep{heckman2001micro}. Pearl's framework for causal reasoning, the Abduction-Action-Prediction cycle, hinges on the fidelity of the first step: abduction. To answer any "what if" question, one must first perfectly infer this primordial information $U_i$ from an observed outcome $v_i$. For decades, while this theoretical blueprint was clear, its practical realization for complex, non-linear mechanisms remained a major computational hurdle, often addressed in econometrics through strong parametric assumptions or linear approximations \citep{angrist2008mostly}.

The advent of deep generative models, particularly diffusion models \citep{ho2020denoising}, offers a powerful new hope for bridging this gap. As near-universal function approximators, they possess the expressive power to learn the complex, non-linear functions that have long challenged classical methods \citep{chao2023interventional, sanchez2022dcms}. However, this promise is shadowed by a critical, yet overlooked, "impedance mismatch." These models were engineered for perceptual tasks like image synthesis, where visual plausibility is paramount, not for the logical rigor demanded by causal abduction. We argue that their standard design, which relies on approximate inversion schemes like DDIM \citep{song2021denoising}, is fundamentally at odds with the strict requirements of this classical causal problem.

In this work, we diagnose and resolve this conflict. We begin by giving the classic requirement for faithful abduction a modern name: \textbf{Causal Information Conservation (CIC)}\footnote{In this work, 'Causal Information Conservation' is defined operationally as the lossless, deterministic recovery of the exogenous noise variable $U$. Its novelty lies in its application as a design principle and diagnostic tool for the diffusion model paradigm in causality, rather than as a formal information-theoretic quantity. Connecting this operational principle to formal measures, such as mutual information, is a compelling avenue for future research.}. Our core contribution is the identification that standard diffusion models systematically violate this principle due to an inherent algorithmic flaw. We formalize this flaw as the \textbf{Structural Reconstruction Error (SRE)}—a quantifiable information loss that imposes a hard theoretical ceiling on the fidelity of any counterfactual generated by such methods. The SRE is not an estimation error to be solved with more data, but a structural defect in the tool itself.

To solve the long-standing challenge of operationalizing faithful abduction, we introduce \textbf{BELM-MDCM}. It is not merely a new model, but the first diffusion-based framework re-engineered from first principles to be causally sound. Architected around an \textbf{analytically invertible} sampler \citep{liu2024belm}, it is the first \textbf{Zero-SRE causal framework by construction}. This design choice reconciles the expressive power of modern diffusion models with the logical rigor of Pearl's causal theory, ensuring the abduction step is lossless. Our primary contributions are therefore:

\begin{enumerate}[label=(\roman*)]
    \item \textbf{Diagnosing a Fundamental Barrier in a Classic Problem.} We are the first to identify that standard diffusion models, when applied to the classic problem of SCM abduction, suffer from a structural flaw we term the \textbf{Structural Reconstruction Error (SRE)}, which violates the foundational principle of \textbf{Causal Information Conservation}.

    \item \textbf{Proposing the First Causally-Sound Diffusion Framework.} We introduce BELM-MDCM, the first framework to eliminate SRE by design. By leveraging an analytically invertible mechanism, it ensures that the power of diffusion models can be applied to causality without compromising the integrity of the abduction process.

    \item \textbf{Developing a Principled Methodology to Operationalize the Framework.} To make our Zero-SRE framework practical and robust, we introduce two synergistic innovations: a \textbf{Targeted Modeling} strategy to manage complexity and a \textbf{Hybrid Training Objective} to provide a strong causal inductive bias, both supported by our theoretical analysis.
\end{enumerate}

Through a comprehensive experimental evaluation, we demonstrate that BELM-MDCM not only sets a new state-of-the-art in estimation accuracy but, more critically, unlocks the generation of authentic individual-level counterfactuals for deep causal inquiries. By providing a foundational blueprint that resolves a core tension between modern machine learning and classical causal theory, our work establishes a new, more rigorous standard for this research direction.

\subsection{The Inversion Challenge in Diffusion-Based Causality}
Diffusion models \citep{ho2020denoising} are powerful generative models that learn to reverse a fixed, gradual noising process. They train a neural network, $\epsilon_\theta(\mathbf{x}_t, t)$, to predict the noise component of a corrupted sample $\mathbf{x}_t$ by optimizing a simple mean-squared error objective:
\begin{equation}\label{eq:simple_loss_main}
    \begin{split}
        L_{\text{simple}}(\theta) = \mathbb{E}_{t, \mathbf{x}_0, \boldsymbol{\epsilon}} \bigg[ \Big\| \boldsymbol{\epsilon} - \epsilon_\theta\big(\sqrt{\bar{\alpha}_t}\mathbf{x}_0
        + \sqrt{1-\bar{\alpha}_t}\boldsymbol{\epsilon}, t\big) \Big\|^2 \bigg]
    \end{split}
\end{equation}
where $\bar{\alpha}_t$ defines the noise schedule and $\boldsymbol{\epsilon} \sim \mathcal{N}(\mathbf{0}, \mathbf{I})$. This trained network is then used to iteratively denoise a variable from pure noise back to a clean sample. A standard deterministic method for this generative process is the Denoising Diffusion Implicit Model (DDIM) \citep{song2021denoising}:
\begin{equation}\label{eq:ddim_reverse_main}
    \begin{split}
        \mathbf{x}_{t-1} = \sqrt{\bar{\alpha}_{t-1}} \left( \frac{\mathbf{x}_t - \sqrt{1-\bar{\alpha}_t}\epsilon_\theta(\mathbf{x}_t,t)}{\sqrt{\bar{\alpha}_t}} \right) 
        + \sqrt{1-\bar{\alpha}_{t-1}} \cdot \epsilon_\theta(\mathbf{x}_t, t)
    \end{split}
\end{equation}
However, causal abduction requires the inverse operation: encoding an observed data point $\mathbf{x}_0$ into its latent noise code $\mathbf{x}_T$. Standard frameworks \citep{chao2023interventional} use the DDIM inversion, which only approximates this path:
\begin{equation}\label{eq:ddim_inversion_main}
    \begin{split}
        \mathbf{x}_{t+1} = \sqrt{\bar{\alpha}_{t+1}} \left( \frac{\mathbf{x}_t - \sqrt{1-\bar{\alpha}_t}\epsilon_\theta(\mathbf{x}_t,t)}{\sqrt{\bar{\alpha}_t}} \right) 
        + \sqrt{1-\bar{\alpha}_{t+1}} \cdot \epsilon_\theta(\mathbf{x}_t, t)
    \end{split}
\end{equation}
This inversion is approximate because it relies on the noise prediction $\epsilon_\theta(\mathbf{x}_t, t)$ remaining constant across the step, which introduces discretization errors that accumulate \citep{liu2022pseudo}. This structural flaw, which we term the \textbf{Structural Reconstruction Error (SRE)}, systematically corrupts the inferred exogenous noise $U_i$. The initial error in the abduction step then propagates through the entire Abduction-Action-Prediction cycle, compromising the fidelity of the final counterfactual.

\subsection{Our Solution: A Zero-SRE Causal Framework}
To eliminate SRE by construction, we build our framework upon an analytically invertible sampler: the \textbf{B}idirectional \textbf{E}xplicit \textbf{L}inear \textbf{M}ulti-step (BELM) sampler \citep{liu2024belm}. BELM overcomes the "memoryless" limitation of single-step samplers like DDIM by using a history of noise predictions, a principle grounded in classical theory for solving ODEs \citep{hairer2006solving}.

Specifically, we employ a second-order BELM. During decoding, it computes a more stable effective noise, $\boldsymbol{\epsilon}_{\text{eff}}$, using predictions from the current and previous timesteps:
\begin{equation}\label{eq:belm_eff_noise_reverse}
    \boldsymbol{\epsilon}_{\text{eff}} = \frac{3}{2}\epsilon_\theta(\mathbf{x}_t, t) - \frac{1}{2}\epsilon_\theta(\mathbf{x}_{t+1}, t+1)
\end{equation}
This improved estimate is then used in a DDIM-like update. The key innovation is that the corresponding encoding process is constructed to be the exact algebraic inverse of this decoding process, guaranteeing that the round-trip is lossless, i.e., $\mathbf{H}(\mathbf{T}(\mathbf{x}_0)) = \mathbf{x}_0$. While the original work on BELM focused on general generative tasks, \textbf{we are the first to identify, leverage, and theoretically justify its analytical invertibility as the key to satisfying the principle of Causal Information Conservation for rigorous counterfactual generation.} Our choice of a second-order BELM represents a deliberate trade-off, providing substantial accuracy gains over single-step methods while maintaining practical efficiency \citep{liu2024belm}, making it ideal for our causal framework.

\subsection{Methodological Gaps in Applying Invertible SCMs}
However, achieving high-fidelity causal inference requires more than a simple substitution of one sampler for another. The principle of analytical invertibility, while theoretically sound, exposes new challenges in practical SCM implementation that our framework is designed to address. 

\paragraph{The Challenge of Model Specification: Targeted Modeling.}
A key decision in SCM construction is assigning a causal mechanism to each node. Naively applying a complex, computationally expensive BELM-based diffusion model to \textit{every} node in the causal graph is suboptimal. This motivates our \textbf{Targeted Modeling} strategy, where model complexity is treated as a resource to be allocated judiciously across the graph.

\paragraph{The Challenge of Downstream Tasks: Hybrid Training.}
The second challenge arises from a fundamental mismatch in objectives. A diffusion model is trained on a generative objective, $L_{\text{diffusion}}(\theta)$, while a downstream predictive task is optimized using a discriminative loss, $L_{\text{task}}(\phi)$. These two objectives are not aligned. This "objective mismatch" motivates our \textbf{Hybrid Training} strategy, which seeks to unify these two goals.

\section{Theoretical Analysis: An Operator-Theoretic Framework}\label{sec:theory}
To formalize our thesis that \textbf{Causal Information Conservation} is paramount and its violation via \textbf{Structural Reconstruction Error} is a fundamental barrier, we develop a rigorous operator-theoretic framework. This perspective is essential for analyzing the fidelity of the \textbf{causal mapping process itself}, moving beyond simple prediction errors. We present the first formal analysis that decomposes the counterfactual error in diffusion-based causal models to explicitly isolate the SRE, proving how our Zero-SRE design eliminates this critical structural limitation.

Our analysis first establishes the conditions for perfect counterfactual generation (\S\ref{sec:theory}.1-\S\ref{sec:theory}.3) and proves that standard methods produce a non-zero SRE, which our sampler eliminates by construction (Proposition \ref{prop:ddim_error}-\ref{prop:belm_invertibility}; \S\ref{sec:theory}.4). The centerpiece is a novel error decomposition theorem that isolates the SRE, motivating our Zero-SRE design (\S\ref{sec:theory}.5-\S\ref{sec:theory}.7). We conclude with learnability guarantees and a discussion of implications for advanced causal tasks like transportability (\S\ref{sec:theory}.8-\S\ref{sec:transportability}).

\subsection{Problem Formulation and Causal Operators}
Let $(\Omega, \mathcal{F}, P)$ be a probability space. We consider endogenous variables $\mathbf{V}$ as elements of the Hilbert space of square-integrable random variables, $\mathcal{X} := L^2(\Omega, \mathbb{R}^d)$. Unless otherwise specified, all vector norms $\| \cdot \|$ in the subsequent analysis refer to the standard Euclidean ($L_2$) norm.

\begin{definition}[Functional SCM Operator]
A Structural Causal Model is defined by a set of unknown, true functional operators $\{\mathbf{F}_i\}_{i=1}^d$, where each $\mathbf{F}_i: \mathcal{X}^{\text{pa}_i} \times \mathcal{U}_i \to \mathcal{X}_i$ is a map such that $V_i := \mathbf{F}_i(\mathbf{Pa}_i, U_i)$, with $\mathbf{Pa}_i$ being the set of parent random variables and $U_i$ an exogenous noise variable. We establish the convention that the corresponding lowercase bold letter, $\mathbf{pa}_i$, denotes a specific vector of observed values for these parents.
\end{definition}

Our goal is to learn a model parameterized by $\theta$ that approximates this SCM. Our model consists of a pair of conditional operators for each variable $V_i$:
\begin{enumerate}
    \item A \textbf{decoder (generative) operator} $\mathbf{H}_\theta: \mathcal{U} \times \mathcal{X}^p \to \mathcal{X}$, which aims to approximate $\mathbf{F}$.
    \item An \textbf{encoder (inference) operator} $\mathbf{T}_\theta: \mathcal{X} \times \mathcal{X}^p \to \mathcal{U}$, which aims to perform abduction by inferring the latent noise.
\end{enumerate}
These operators are realized by solving the probability flow ODE (Appendix \ref{sec:appendix_theory}). The decoder $\mathbf{H}_\theta$ solves the ODE from $t=T$ to $t=0$, while the encoder $\mathbf{T}_\theta$ solves it from $t=0$ to $t=T$. Our BELM sampler is a high-fidelity numerical solver designed such that these forward and backward operations are exact algebraic inverses.

\subsection{Identifiability and Exact Counterfactual Generation}
We adapt principles from identifiable generative modeling \citep{chao2023interventional} to formalize the conditions for exact counterfactuals. This requires assuming the SCM is invertible with respect to its noise term, a condition discussed in Section \ref{sec:assumptions}.

\begin{theorem}[Identifiability via Statistical Independence]\label{thm:identifiability}
Given an SCM operator $X := \mathbf{F}(\mathbf{Pa}, U)$ where $U \perp\!\!\!\perp \mathbf{Pa}$ and $\mathbf{F}$ is invertible w.r.t. $U$. If a learned encoder $\mathbf{T}_\theta$ (with sufficient capacity) yields a latent representation $Z = \mathbf{T}_\theta(X, \mathbf{Pa})$ that is statistically independent of the parents $\mathbf{Pa}$, then $Z$ is an isomorphic representation of the exogenous noise $U$.
\end{theorem}

\subsection{Geometric Inductive Bias for Identifiability}\label{sec:geometric_bias}
The score-matching objective's geometric inductive biases strengthen our identifiability argument. We leverage the principle of \textbf{implicit regularization}, where optimizers favor "simpler" functions \citep{hochreiter1997flat, neyshabur2018pac}.\footnote{We adopt the principle of simplicity bias, a cornerstone of modern deep learning theory that, while empirically supported, remains an active and not yet universally proven area of research. Our conclusions are conditioned on its validity, as discussed further in Section \ref{sec:assumptions}.} This suggests the model learns the most parsimonious geometric transformation required to explain the data.

Considering the local geometry of the data density $p(\mathbf{x})$ provides powerful intuition. In a local region $\mathcal{R}$, if the data is isotropic (spherically symmetric), the simplest score function is a radial vector field, yielding a conformal map. If the structure is simply anisotropic (e.g., ellipsoidal), the model is biased towards learning a local affine map. This refines the notion of a purely conformal bias and leads to the following proposition.

\begin{proposition}[Implicit Bias towards Simple Geometric Maps]\label{prop:conformal_bias}
Assume (A1) the true data density $p(\mathbf{x})$ is smooth ($C^2$) and (A2) the optimization process has a simplicity bias (e.g., favoring low-complexity solutions, see Appendix \ref{sec:appendix_conformal_proof}).
\begin{enumerate}[label=(\roman*)]
    \item If there exists a local region $\mathcal{R}$ where $p(\mathbf{x})$ is isotropic, the optimal learned score function is a radial vector field, and the flow map it generates is a \textbf{conformal map} on $\mathcal{R}$.
    \item If we relax the condition to a local region $\mathcal{R}$ where $p(\mathbf{x})$ has an ellipsoidal structure, the optimal learned score function is normal to the ellipsoidal iso-contours, and the flow map it generates is a \textbf{local affine transformation} on $\mathcal{R}$.
\end{enumerate}
\end{proposition}
The formal argument is detailed in Appendix~\ref{sec:appendix_conformal_proof}. This proposition is significant: it suggests that the model defaults to learning the most parsimonious, well-behaved, and locally invertible map that can explain the data's geometry. This bias is crucial for the abduction step, as it prevents the pathological distortions that would corrupt the inferred causal noise $U$.

\begin{theorem}[Operator Isomorphism Guarantees Exact Counterfactuals]\label{thm:correctness}
Let the conditions of Theorem \ref{thm:identifiability} hold. If the learned operator pair$(\mathbf{T}_\theta, \mathbf{H}_\theta)$ constitutes a conditional isomorphism (i.e., $\mathbf{H}_\theta(\mathbf{T}_\theta(\cdot, \mathbf{pa}), \mathbf{pa}) = \mathbf{I}$, the identity operator), then the model's prediction under an intervention $do(\mathbf{Pa} := \boldsymbol{\alpha})$ is exact.
\end{theorem}
\begin{proof}
A full proof, covering cases for different dimensions of the exogenous noise variable, is provided in Appendix \ref{sec:appendix_identifiability_proofs}.
\end{proof}

\subsection{Analysis of Inversion Fidelity}
We now formally analyze the inversion error. We prove that standard approximate schemes produce a non-zero SRE (Proposition \ref{prop:ddim_error}), whereas our chosen sampler eliminates it by construction (Proposition \ref{prop:belm_invertibility}).

\begin{proposition}[Structural Error of Approximate Inversion]\label{prop:ddim_error}
Let $\mathbf{T}_{\text{DDIM}}$ be the operator for one step of DDIM inversion from $\mathbf{x}_t$ to $\mathbf{x}_{t+1}$, and $\mathbf{H}_{\text{DDIM}}$ be the generative step operator from $\mathbf{x}_{t+1}$ to $\mathbf{x}_t$. The single-step reconstruction error is non-zero and of second order in the time step $\Delta t$:
$$ (\mathbf{H}_{\text{DDIM}} \circ \mathbf{T}_{\text{DDIM}})(\mathbf{x}_t) - \mathbf{x}_t = \mathcal{O}((\Delta t)^2) $$
This error accumulates over the full trajectory, leading to a non-zero Structural Reconstruction Error.
\end{proposition}
\begin{proof}
See Appendix \ref{sec:appendix_inversion_and_non_invertible} for a rigorous proof.
\end{proof}

\begin{proposition}[Analytical Invertibility of the Sampler]\label{prop:belm_invertibility}
Let $\mathbf{T}_{\text{BELM}}$ and $\mathbf{H}_{\text{BELM}}$ be the operators corresponding to the full-trajectory BELM sampler for inference and generation, respectively. For a fixed noise prediction network $\epsilon_\theta$, the operators are exact algebraic inverses:
$$ \mathbf{H}_{\text{BELM}} \circ \mathbf{T}_{\text{BELM}} = \mathbf{I} $$
\end{proposition}
\begin{proof}
The proof follows from the algebraic construction of the BELM update rules, as detailed in Appendix \ref{sec:appendix_inversion_and_non_invertible}.
\end{proof}

\subsection{Error Decomposition for Counterfactual Estimation}
This brings us to our central theoretical result: an error decomposition theorem that rigorously partitions the total counterfactual error. This decomposition isolates the SRE and mathematically demonstrates why its elimination is critical.

\begin{definition}[Counterfactual Error Components]
We formally define the two primary sources of error in counterfactual estimation for the invertible case:
\begin{enumerate}
    \item The \textbf{Structural Reconstruction Error} ($E_{SR}$) measures the information loss from the model's abduction-action cycle on a given sample $X$:
    $$ E_{SR}(X) := \| (\mathbf{H}_\theta \circ \mathbf{T}_\theta - \mathbf{I})X \|^2 $$
    \item The \textbf{Latent Space Invariance Error} ($E_{LSI}$) measures the failure of the learned latent space to remain invariant under interventions on parent variables:
    $$ E_{LSI} := \| \mathbf{T}_\theta(X, \mathbf{Pa}) - \mathbf{T}_\theta(X_{\boldsymbol{\alpha}}^{\text{true}}, \boldsymbol{\alpha}) \|^2 $$
\end{enumerate}
\end{definition}

\begin{theorem}[Counterfactual Error Bound]\label{thm:error_bound}
Let a model be defined by $(\mathbf{T}_\theta, \mathbf{H}_\theta)$ and the true SCM by $\mathbf{F}$. Assume the decoder $\mathbf{H}_\theta$ is $L_\mathcal{H}$-Lipschitz. The expected squared error of the model's counterfactual prediction $\hat{X}_{\boldsymbol{\alpha}}$ is bounded by the expectation of the two error components:
$$
\mathbb{E}\left[\| \hat{X}_{\boldsymbol{\alpha}} - X_{\boldsymbol{\alpha}}^{\text{true}} \|^2\right] \le 2 \mathbb{E}\left[E_{SR}(X_{\boldsymbol{\alpha}}^{\text{true}})\right] + 2 L_\mathcal{H}^2 \mathbb{E}\left[E_{LSI}\right]
$$
\end{theorem}
\begin{proof}
The proof is in Appendix \ref{sec:appendix_main_proofs}.
\end{proof}

\begin{remark}[Elimination of Structural Error]
By Proposition \ref{prop:belm_invertibility}, the \textbf{Structural Reconstruction Error} for BELM-MDCM is identically zero. This is the central theoretical advantage of our framework. It disentangles the error sources, allowing us to isolate the entire modeling challenge to learning a high-quality score function ($\epsilon_\theta$) without the confounding factor of an imperfect inversion algorithm. Any remaining error is now purely a function of statistical estimation, not a structural bias of the model itself.
\end{remark}

\begin{proposition}[Bound on Latent Space Invariance Error]\label{prop:lsi_bound}
We assume the learned score network, $\boldsymbol{\epsilon}_\theta$, is Lipschitz continuous, ensuring the existence and uniqueness of the probability flow ODE solution via the Picard-Lindelöf theorem. Under standard integrability conditions (Fubini's theorem), the Latent Space Invariance Error is bounded by the expected score-matching loss:
$$
\mathbb{E}\left[E_{LSI}\right] \le C' \cdot \mathbb{E}\left[\| \boldsymbol{\epsilon}_\theta - \boldsymbol{\epsilon}^* \|^2\right]
$$
for some constant $C'$, where $\boldsymbol{\epsilon}^*$ is the true score function.
\end{proposition}
\begin{proof}
The proof is in Appendix \ref{sec:appendix_main_proofs}.
\end{proof}
This proposition formally establishes that by eliminating structural error, the causal fidelity of BELM-MDCM is directly and provably controlled by its ability to accurately learn the data's score function.

\subsection{Decomposing Error: A Motivation for Empirical Validation}\label{sec:error_decomp_motivation}
The error decomposition in Theorem \ref{thm:error_bound} provides a clear strategy for empirical validation by isolating two distinct error sources: the \textbf{Structural Reconstruction Error ($E_{SR}$)} and the \textbf{Latent Space Invariance Error}. While developing a single score combining these is future work, these components directly motivate our empirical investigations. Our ablation study (Section \ref{sec:ablation}) is designed to measure the impact of a non-zero $E_{SR}$, while our stress-test (Section \ref{sec:exp_many_to_one}) probes robustness when latent space invariance is challenged by a non-invertible SCM.

\subsection{Theoretical Roles of Targeted Modeling and Hybrid Training}
With algorithmic error eliminated by our Zero-SRE design, the challenge becomes minimizing the modeling error ($E_{LSI}$). Our two methodological innovations, Targeted Modeling and Hybrid Training, are principled strategies for this purpose.

\paragraph{Targeted Modeling as Formal Complexity Control.}
Our Targeted Modeling strategy acts as a form of structural regularization. The finite sample bound in Theorem \ref{thm:finite_sample_specific} is governed by the Rademacher complexity $\mathfrak{R}_n(\mathcal{F}_\Theta)$ of the entire SCM's hypothesis space. By assigning low-complexity models to a subset of nodes, we directly constrain the overall complexity.

\begin{remark}[Effect on Generalization Bound]
Our Targeted Modeling strategy is formally justified as a complexity control mechanism. The Rademacher complexity of a composite SCM is bounded by the sum of the complexities of its individual mechanisms \citep{mohri2018foundations}. By strategically substituting a high-complexity diffusion model $\mathcal{F}_{\text{diff}}$ with a lower-complexity alternative $\mathcal{F}_{\text{simple}}$ for non-critical nodes, Targeted Modeling directly minimizes this upper bound. This leads to a tighter generalization bound and improves the statistical efficiency of the overall SCM.
\end{remark}

\paragraph{Hybrid Training as a Weighted Score-Matching Objective.}
The Hybrid Training Objective, $L_{\text{total}} = L_{\text{diffusion}} + \lambda \cdot L_{\text{task}}$, imparts a crucial inductive bias for learning a \textit{causally salient} score function. The task-specific loss acts as a \textit{conductor's baton}, forcing the model to prioritize learning an accurate score function in regions of the data manifold most critical to the causal question. We formalize this by proposing that the auxiliary loss implicitly implements a weighted score-matching objective.

\begin{proposition}[Hybrid Objective as a Weighted Score-Matching Regularizer]\label{prop:hybrid_objective}
The auxiliary task loss $L_{\text{task}}$ provides a lower bound for the model's error, weighted by a function reflecting the causal salience of the data manifold. Minimizing the hybrid objective $L_{\text{total}}$ is thereby equivalent to solving a weighted score-matching problem that prioritizes accuracy in causally salient regions, leading to a smaller effective Latent Space Invariance Error. (A rigorous proof is provided in Appendix \ref{sec:appendix_hybrid_proofs}.)
\end{proposition}
This proposition formally grounds our hybrid training strategy, revealing that the task-specific loss intelligently forces the diffusion model to prioritize accuracy in regions of the data manifold most critical to the causal question. This reinforces the CIC principle by avoiding information loss where it matters most, effectively implementing the simplicity bias principle from Section \ref{sec:geometric_bias}.

We can deepen this insight by analyzing its information-theoretic implications.
\begin{proposition}[Disentanglement via Hybrid Objective]
\label{prop:disentanglement}
Information-theoretically, the hybrid objective provides a strong inductive bias towards learning a disentangled latent representation. It encourages a "division of labor" where the task-specific component explains variance from the parents $\mathbf{Pa}$, while the diffusion component's latent code $Z = \mathbf{T}_\theta(V, \mathbf{Pa})$ models the residual information. This implicitly pushes $Z$ towards being independent of $\mathbf{Pa}$, a crucial step towards satisfying the identifiability conditions.
\end{proposition}
\begin{proof}
A detailed information-theoretic argument is provided in Appendix \ref{sec:appendix_hybrid_proofs}.
\end{proof}

\subsection{BELM-MDCM as a Unifying Framework}
The principle of Causal Information Conservation also unifies our framework with classical models. Simpler models like Additive Noise Models (ANMs) can be seen as special cases where this principle is met trivially, positioning our work as a generalization of established causal principles. For instance, in a classic ANM \citep{hoyer2009nonlinear}, $V_i = f_i(\mathbf{Pa}_i) + U_i$, the noise is recovered by a direct, lossless inversion: $U_i = V_i - f_i(\mathbf{Pa}_i)$. Our framework generalizes this principle to arbitrarily complex, non-additive mechanisms, offering a flexible, non-parametric extension to classical structural equation models \citep{wooldridge2010econometric}. The importance of noise distributions, particularly non-Gaussianity, for identifiability in linear models is also a well-established principle \citep{shimizu2006linear}.

\subsection{Learnability and Statistical Guarantees}
We now provide finite-sample learnability guarantees for our SCM framework.

\begin{proposition}[Asymptotic Consistency]\label{prop:consistency}
Under standard regularity conditions, as the number of data samples $n \to \infty$ and model capacity $N \to \infty$, the learned operators $(\hat{\mathbf{T}}_n, \hat{\mathbf{H}}_n)$ are consistent estimators of the ideal operators $(\mathbf{T}^*, \mathbf{H}^*)$: $\hat{\mathbf{T}}_n \xrightarrow{p} \mathbf{T}^*$ and $\hat{\mathbf{H}}_n \xrightarrow{p} \mathbf{H}^*$.
\end{proposition}

\begin{theorem}[Finite Sample Bound for Causal Diffusion SCMs]\label{thm:finite_sample_specific}
Let an SCM consist of $d$ endogenous nodes, with a causal graph having a maximum in-degree of $d_{in}^{max}$. Assume each causal mechanism is implemented by a score network $\epsilon_\theta$ that is an $L$-layer MLP with ReLU activations, and the spectral norm of each weight matrix is bounded by $B$. Let the input space be appropriately normalized. Let the loss function be bounded by $M$. Then, for the parameters $\hat{\theta}_n$ learned from $n$ samples, the excess risk is bounded with probability at least $1-\delta$:
$$ R(\hat{\theta}_n) - R(\theta^*) \le C \cdot \frac{d \cdot L \cdot B^L \cdot \sqrt{d_{in}^{max} + d_{embed} + 1}}{\sqrt{n}} + M\sqrt{\frac{\log(1/\delta)}{2n}} $$
where $C$ is a constant independent of the network architecture and sample size, and $d_{embed}$ is the dimension of the time embedding.
\end{theorem}
\begin{proof}
The proof, which combines the sub-additivity of Rademacher complexity over the SCM with standard bounds for deep neural networks \citep{bartlett2017spectrally, neyshabur2018pac}, is detailed in Appendix~\ref{sec:appendix_learning_proofs_specific}.
\end{proof}

\begin{remark}[Interpretation of the Bound]
This refined bound quantitatively links the generalization error to:
\begin{enumerate}[label=(\roman*), topsep=0pt, noitemsep]
    \item \textbf{Causal Complexity ($d \cdot \sqrt{d_{in}^{max}}$):} The error scales with the number of causal mechanisms ($d$) and the graph's complexity ($d_{in}^{max}$), formalizing the intuition that more complex causal systems are harder to learn.
    \item \textbf{Network Complexity ($L \cdot B^L$):} The error scales with the depth and spectral norm of the score networks. This provides direct theoretical grounding for our Targeted Modeling strategy, as using simpler models tightens this generalization bound.
\end{enumerate}
\end{remark}

\subsection{Implications for Causal Transportability}
\label{sec:transportability}

Causal Information Conservation also provides a foundation for \textbf{transportability}—applying knowledge from a source domain $\mathcal{S}$ to a target domain $\mathcal{T}$ \citep{pearl2014transportability}. Transportability requires separating invariant causal knowledge from domain-specific mechanisms. By losslessly recovering the exogenous noise $U$ (the invariant "causal essence"), our framework achieves this separation by design; the decoders $\mathbf{H}_\theta$ represent the domain-specific mechanisms. This insight is formalized in the following theorem.

\begin{theorem}[Condition for Lossless Causal Transport]
\label{thm:transport}
Let a source domain $\mathcal{S}$ and a target domain $\mathcal{T}$ be described by SCMs $\mathcal{M}^\mathcal{S}$ and $\mathcal{M}^\mathcal{T}$, respectively. Assume the following conditions hold:
\begin{enumerate}[label=(\roman*), noitemsep]
    \item \textbf{Shared Structure:} Both domains share the same causal graph $\mathcal{G}$ and the same exogenous noise distributions $\{p_i(U_i)\}$. The domains differ only in a subset of causal mechanisms $\mathcal{K}_{\text{changed}}$.
    \item \textbf{Noise Independence:} The exogenous noise variables $\{U_i\}_{i=1}^d$ are mutually independent.
    \item \textbf{Information Conservation:} A model $(\mathbf{T}_\theta, \mathbf{H}_\theta)$ trained on data from $\mathcal{S}$ satisfies the Causal Information Conservation principle, achieving zero Structural Reconstruction Error.
\end{enumerate}
Then, causal knowledge can be losslessly transported from $\mathcal{S}$ to $\mathcal{T}$ by re-learning only the operators $\{\mathbf{T}_{\theta_k}, \mathbf{H}_{\theta_k}\}$ corresponding to the changed mechanisms $k \in \mathcal{K}_{\text{changed}}$, while directly reusing all operators for invariant mechanisms.
\end{theorem}
\begin{proof}
The proof is provided in Appendix \ref{sec:appendix_transport_proof}.
\end{proof}

\subsection{Discussion of Assumptions}\label{sec:assumptions}
Our framework rests on several key assumptions, which we now critically examine.

Our geometric inductive bias argument (Proposition \ref{prop:conformal_bias}) rests on the principle of simplicity bias. While this principle is a cornerstone of modern deep learning theory with substantial empirical backing, it remains an active area of research and is not a universally proven theorem. Our conclusions are therefore conditioned on the validity of this powerful but conjectural assumption.

The cornerstone of our identifiability theory (Theorem \ref{thm:identifiability}) is the SCM's invertibility with respect to its noise term $U$. This is a strong assumption; when violated (e.g., by a many-to-one function), the abduction task becomes ill-posed.

To address this foundational challenge, we provide an exhaustive theoretical treatment in \textbf{Appendix~\ref{sec:appendix_inversion_and_non_invertible}}. There, we formalize the irreducible "representational error" and derive a tighter, more general error bound (Theorem~\ref{thm:non_inv_bound_tight}). More importantly, we propose a concrete mitigation strategy: a novel prior-matching regularizer (Definition~\ref{def:prior_reg_appendix}), theoretically shown to reduce the error by encouraging the learned encoder to approximate the ideal Maximum a Posteriori (MAP) solution (Proposition~\ref{prop:regularizer_guarantee}). This highlights a primary contribution: even in the challenging non-invertible case, BELM-MDCM's zero-SRE design eliminates the \textit{algorithmic} error, thereby isolating the more fundamental \textit{representational} challenge. Our stress-test in Section \ref{sec:exp_many_to_one} empirically confirms this advantage, while validating our regularizer provides a clear direction for future work.

Our identifiability proof is dimension-dependent, leveraging Liouville's theorem for $d \ge 3$ and requiring stronger assumptions like asymptotic linearity for the special case of $d=2$. Other assumptions, such as Lipschitz continuity of the score network, are mild regularity conditions standard in deep generative model analysis and can be encouraged through architectural choices like spectral normalization.

\section{Architectural Design and Training}
\label{sec:arch}

The BELM-MDCM architecture embodies our core principles through a non-monolithic, theoretically-motivated design. Its central philosophy is \textbf{Targeted Modeling}: judiciously allocating the expressive power of our Zero-SRE \texttt{CausalDiffusionModel} to nodes of causal interest (e.g., \textbf{Treatment T}, \textbf{Outcome Y}), while using simpler, efficient mechanisms for confounders, as illustrated in Figure~\ref{fig:targeted_modeling}. This strategy provides practical complexity control, tightening the generalization bound as established in Theorem \ref{thm:finite_sample_specific}.

\begin{figure}[htbp]
    \centering
    \includegraphics[width=0.7\textwidth]{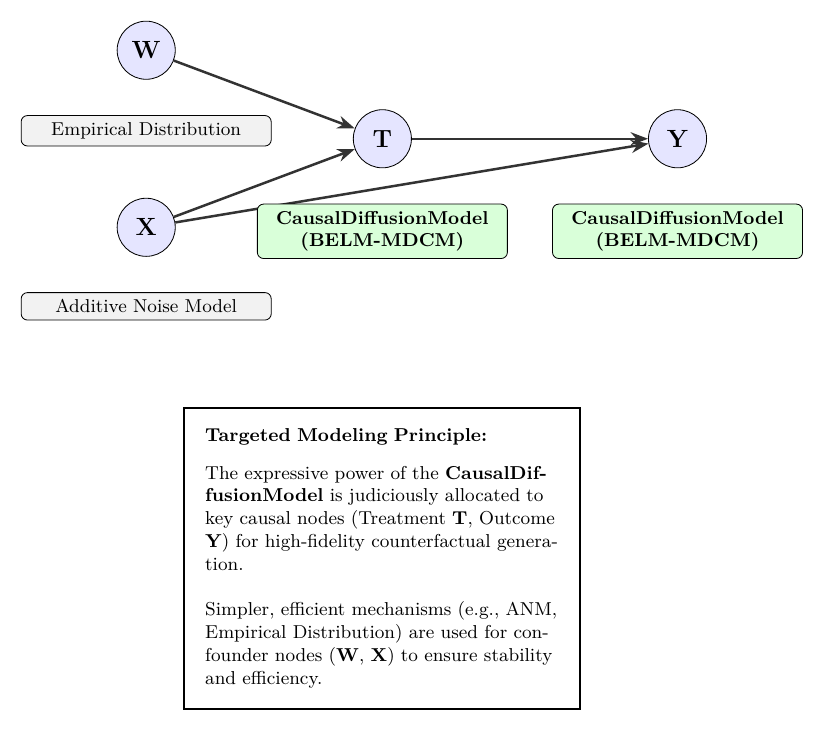}
    \caption{Illustration of the Targeted Modeling Principle. The expressive \texttt{CausalDiffusionModel} is judiciously allocated to key causal nodes (Treatment T, Outcome Y) for high-fidelity counterfactual generation. Simpler, efficient mechanisms (e.g., ANM, Empirical Distribution) are used for confounder nodes (W, X) to ensure stability and efficiency.}
    \label{fig:targeted_modeling}
\end{figure}
\begin{figure}[htbp]
    \centering
    \includegraphics[width=0.8\textwidth]{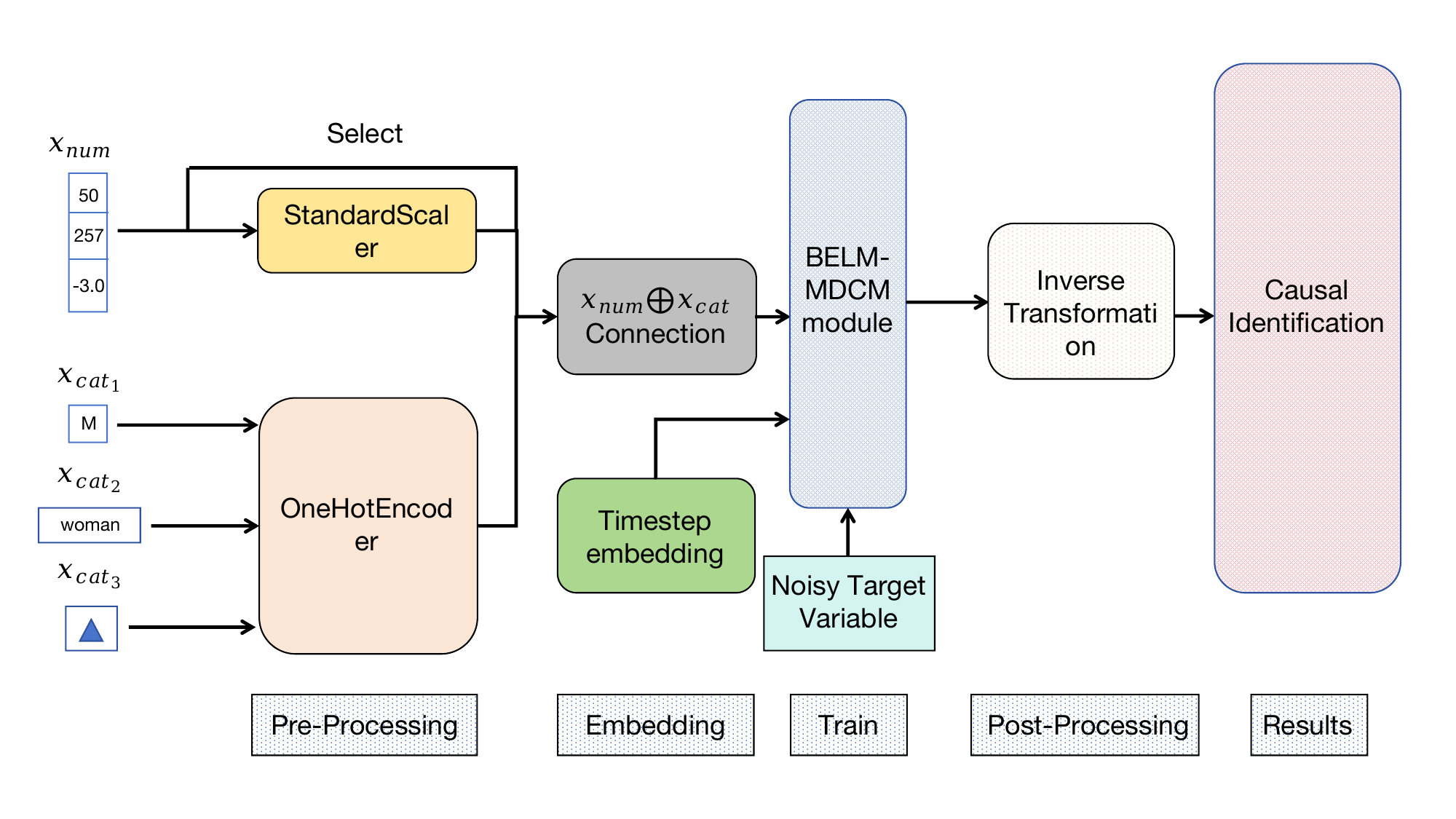}
    \caption{The detailed internal architecture of the \texttt{CausalDiffusionModel}. This diagram illustrates the end-to-end workflow of the causal mechanism designed for key nodes like Treatment T and Outcome Y, detailing the pre-processing, embedding, training, and post-processing stages.}
    \label{fig:model_structure}
\end{figure}
The internal architecture of the \texttt{CausalDiffusionModel} itself, depicted in Figure~\ref{fig:model_structure}, is engineered to learn the complex, non-linear mapping $v_i := f_i(\mathbf{pa}_i, u_i)$ with high fidelity.
\FloatBarrier

\subsection{Mechanism for Exogenous Nodes}
Exogenous nodes (without parents in the causal graph $\mathcal{G}$) are modeled non-parametrically via the \textbf{Empirical Distribution} of the observed data. This approach avoids distributional assumptions and provides a robust foundation for the Structural Causal Model (SCM).

\subsection{Mechanism for Endogenous Nodes: The CausalDiffusionModel}
For endogenous nodes $V_i$, particularly those central to the causal query (treatment, outcome, key mediators), we employ our bespoke \texttt{CausalDiffusionModel} to learn the functional mapping $v_i := f_i(\mathbf{pa}_i, u_i)$.

\subsubsection{Conditioning via Parent Node Transformation}
The denoising process is conditioned on the parent nodes $\mathbf{pa}_i$, which are transformed into a fixed-dimensional conditioning vector $\mathbf{c} \in \mathbb{R}^{d_c}$. A \texttt{ColumnTransformer} handles heterogeneous data types: continuous parents are standardized (\textbf{StandardScaler}) to unify scales, while categorical parents are one-hot encoded (\textbf{OneHotEncoder}) to prevent artificial ordinality. The resulting vectors are concatenated into $\mathbf{c}$, which remains constant for a given sample's diffusion trajectory.

\subsubsection{The Denoising Process}
The core of the \texttt{CausalDiffusionModel} is a denoising network $\epsilon_\theta(v_t, t, \mathbf{c})$, implemented as a \textbf{Residual MLP} \citep{he2016deep}. It takes as input the noisy variable $v_t$, a sinusoidal \textbf{Time Embedding} of timestep $t$, and the conditioning vector $\mathbf{c}$. Before the diffusion process, the target variable $V_i$ is also preprocessed (standardized for continuous values or label-encoded for categorical ones). The denoising process is driven by the BELM sampler, ensuring a mathematically exact and stable inversion path as established in Section \ref{sec:theory}.

\subsubsection{Hybrid Training Objective}\label{sec:hybrid_objective}
We introduce a \textbf{Hybrid Training Objective} to reconcile generative fidelity with predictive accuracy. As established in our theoretical analysis (Proposition \ref{prop:hybrid_objective}), this is more than a standard multi-task learning scheme; it acts as a powerful inductive bias, creating a weighted score-matching objective that prioritizes accuracy in causally salient regions of the data manifold. The total loss is a linearly weighted combination:
\begin{equation}
    L_{\text{total}} = L_{\text{diffusion}} + \lambda \cdot L_{\text{task}}
\end{equation}
where $L_{\text{diffusion}}$ is the noise prediction error (Equation \ref{eq:simple_loss_main}). The auxiliary loss $L_{\text{task}}$ is a Mean Squared Error for continuous nodes ($L_{\text{regression}}$) and a Cross-Entropy loss for discrete nodes ($L_{\text{classification}}$).

\subsubsection{Decoding and Counterfactual Generation}
For generation, the BELM sampler produces an output in the normalized space. This is then mapped back to the original data domain using the inverse transformations of the pre-fitted preprocessors (\texttt{StandardScaler} for continuous, \texttt{LabelEncoder} for categorical). For categorical outputs, the continuous value is rounded and clipped to the valid class range before the inverse mapping, ensuring that generated (counterfactual) data is interpretable and resides in the correct space.

\section{New Evaluation Metrics for Generative Causal Models}
\label{sec:new_metrics}

The principle of Causal Information Conservation demands new evaluation dimensions that traditional metrics like ATE and PEHE cannot capture. An accurate ATE score, for instance, could arise from a model with high SRE where individual errors fortuitously cancel out at the population level. To move beyond mere outcome accuracy and directly assess a model's adherence to our foundational principle, we propose a new, theoretically-grounded evaluation framework.

\subsection{Causal Information Conservation Score (CIC-Score)}
\label{subsec:cic_score}

The \textbf{Causal Information Conservation Score (CIC-Score)} is a direct empirical diagnostic for the Structural Reconstruction Error. It quantifies a framework's adherence to the CIC principle by disentangling algorithmic information loss (from an imperfect inversion process) from modeling error (from the statistical challenge of learning the true causal mechanism). We define the score, bounded in $[0, 1]$, using an exponential formulation:
\[
\text{CIC-Score} = \exp\left( - \left( \delta_U + \delta_{\text{SRE}} \right) \right)
\]
The error components are designed to isolate distinct failure modes:

\begin{itemize}[leftmargin=*, noitemsep]
    \item $\delta_U$, the \textbf{Relative Noise Recovery Error}, quantifies the \textit{modeling error}. It measures how well the trained network approximates the true score function, reflected in the fidelity of the recovered noise $\hat{U}$ versus the ground-truth $U_{\text{true}}$:
    \[
    \delta_U = \frac{\mathbb{E}[\|\hat{U}_{\text{scaled}} - U_{\text{true, scaled}}\|^2]}{\mathbb{E}[\|U_{\text{true, scaled}}\|^2]}
    \]
    This term captures all inaccuracies from finite data and imperfect optimization.

    \item $\delta_{\text{SRE}}$, the \textbf{Normalized Structural Error}, exclusively quantifies the \textit{algorithmic error} inherent to the inversion process itself. Its definition is model-dependent to allow for fair comparisons:
    \begin{itemize}
        \item For frameworks with analytical invertibility (e.g., our BELM-MDCM, ANMs), the algorithm introduces no information loss, so we set $\delta_{\text{SRE}} \equiv 0$ by construction. Any observed reconstruction error is a symptom of modeling error, already captured by $\delta_U$.
        \item For frameworks relying on approximate inversion (e.g., DDIM), $\delta_{\text{SRE}}$ is empirically measured to quantify this inherent algorithmic flaw:
         \[
         \delta_{\text{SRE}} = \frac{\mathbb{E}[\|(\mathbf{H}_\theta \circ \mathbf{T}_\theta - \mathbf{I})X\|^2]}{\mathbb{E}[\|X\|^2]}
         \]
    \end{itemize}
\end{itemize}
This principled decomposition allows the CIC-Score to fairly assess different frameworks by isolating structural design advantages from the universal challenge of model training.

\subsection{Causal Mechanism Fidelity Score (CMF-Score)}
\label{subsec:cmf_score}

A generative causal model's core promise is to learn true causal \textit{mechanisms}, not just outcomes. Naïve metrics like pairwise correlations fail to capture the non-linear, multi-variable, and directional nature of causality. We therefore propose the \textbf{Causal Mechanism Fidelity (CMF)} score, a hierarchical framework with two levels of increasing rigor.

\subsubsection{Level 1 (Pragmatic): The Conditional Mutual Information Score (CMI-Score)}
The Conditional Mutual Information (CMI), $I(V_i; V_j | \mathbf{Pa}_j \setminus \{V_i\})$, is a non-parametric, non-linear measure of the direct influence a parent $V_i$ has on its child $V_j$ after accounting for all other parents. The CMI-Score evaluates whether this influence is preserved. For a single mechanism $V_j$, it is the average consistency across all parent-child edges:
\[
\text{CMI-Score}(V_j) = \frac{1}{|\mathbf{Pa}_j|} \sum_{V_i \in \mathbf{Pa}_j} \left( 1 - \frac{\left| I_{\text{obs}}(V_i; V_j | \cdot) - I_{\text{cf}}(V_i; V'_j | \cdot) \right|}{I_{\text{obs}}(V_i; V_j | \cdot) + \epsilon} \right)
\]
where $I_{\text{obs}}$ and $I_{\text{cf}}$ are the CMI values from observational and counterfactual data, respectively. The final CMI-Score is the average over all SCM mechanisms.

\subsubsection{Level 2 (Gold Standard): The Kernelized Mechanism Discrepancy (KMD) Score}
To rigorously compare entire conditional \textit{distributions}, we use the Maximum Mean Discrepancy (MMD) \citep{gretton2012kernel}, a kernel-based statistical test for distributional equality. The KMD-Score applies this test to the conditional distributions $p(V_j | \mathbf{Pa}_j)$ that define each causal mechanism, measuring the discrepancy between the learned and observed conditionals. The final score is mapped to a similarity measure in $[0, 1]$:
\[
\text{KMD-Score} = \exp(-\gamma \cdot \mathbb{E}_{\mathbf{pa}_j \sim p(\mathbf{Pa}_j)}[\text{MMD}(p(V_j|\mathbf{pa}_j), p_\theta(V_j|\mathbf{pa}_j))])
\]
where $\gamma$ is a scaling parameter. A score of 1 indicates that the learned conditional mechanism is statistically indistinguishable from the observed one.

\paragraph{Complementary and Validated Evaluation Metrics.}
Our proposed metrics complement, rather than replace, traditional ones like ATE and PEHE. They evaluate distinct facets of performance: while ATE/PEHE measure \textbf{outcome accuracy}, the CMF-Score assesses \textbf{mechanism fidelity}. This distinction is critical, as a model can achieve a high ATE via fortuitous error cancellation despite failing to learn the true data-generating process. To ensure our metrics are practically reliable, we conducted a controlled micro-simulation study, detailed in \textbf{Appendix \ref{sec:appendix_metric_validation}}. The results provide strong empirical evidence for their validity and complementary roles: the CIC-Score acts as a high-sensitivity SRE detector; the CMI-Score robustly tracks the fidelity of causal associations; and the KMD-Score serves as a final arbiter of distributional similarity. This validation confirms that our evaluation framework offers a more complete, nuanced, and reliable assessment of generative causal models.
\FloatBarrier

\section{Experiments}\label{sec:experiments}

Our empirical evaluation is designed as a comprehensive test of our central thesis: that eliminating SRE is a necessary condition for generating authentic counterfactuals and unlocks analytical capabilities beyond the reach of conventional methods. We structured the study as a four-act narrative to rigorously test our claims. \textbf{Act I} establishes our model's state-of-the-art predictive fidelity on standard benchmarks. \textbf{Act II} provides a deep diagnostic analysis, using our proposed metrics as empirical evidence for the destructive effect of SRE. \textbf{Act III} showcases the unique capabilities unlocked by an information-conserving framework. Finally, \textbf{Act IV} validates the framework's robustness through a series of stress tests and a full ablation study.

\paragraph{Evaluation Protocol.}
For a rigorous evaluation, we employ two complementary protocols. This distinction is crucial, as it separates the assessment of our methodology's peak performance from the diagnostic analysis of its components.
\begin{enumerate}[label=(\alph*), topsep=0pt, noitemsep]
    \item \textbf{Ensemble Evaluation for SOTA Performance:} To benchmark against state-of-the-art methods (specifically, ITE estimation in Section~\ref{sec:ite_estimation}), we adopt the standard Deep Ensemble methodology. We train N=5 independent models and report the final metric (e.g., PEHE) on the ensembled prediction.
    \item \textbf{Individual Model Evaluation for Diagnostic Analysis:} In all other experiments where the goal is a fair, apples-to-apples architectural comparison or stability assessment, we report the \textbf{mean and standard deviation of metrics from individual model instances} across N=5 runs. This isolates the effect of design choices from the gains of ensembling.
\end{enumerate}
We estimate the Average Treatment Effect (ATE) throughout our experiments using a standard counterfactual imputation procedure, the pseudo-code for which is detailed in Algorithm~\ref{alg:ate_estimation_appendix} in Appendix~\ref{sec:appendix_algorithm}.
\paragraph{Baseline Estimators.}
The Directed Acyclic Graphs (DAGs) for our experiments are shown in Figure~\ref{fig:all_dags_combined}. We benchmark BELM-MDCM against a suite of baselines from the DoWhy library \citep{sharma2022dowhy}, spanning classical statistical methods to state-of-the-art machine learning estimators to ensure a comprehensive comparison.

\begin{figure*}[htbp]
    \centering
    \begin{subfigure}[b]{0.3\textwidth}
        \centering
        \includegraphics[width=\linewidth]{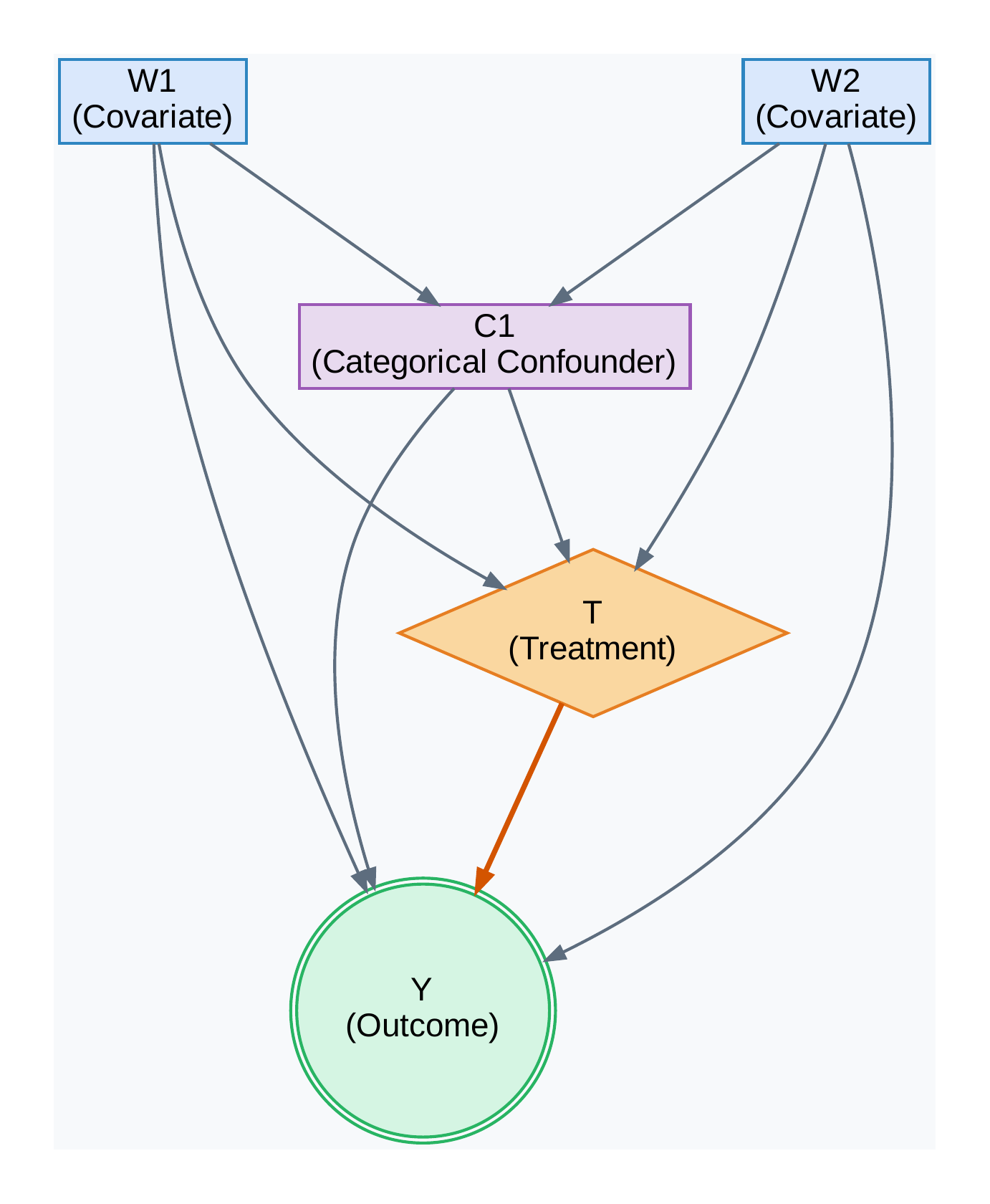}
        \caption{PSM Failure Scenario}
        \label{fig:sub_dag_psm_fail}
    \end{subfigure}
    \hfill
    \begin{subfigure}[b]{0.3\textwidth}
        \centering
        \includegraphics[width=\linewidth]{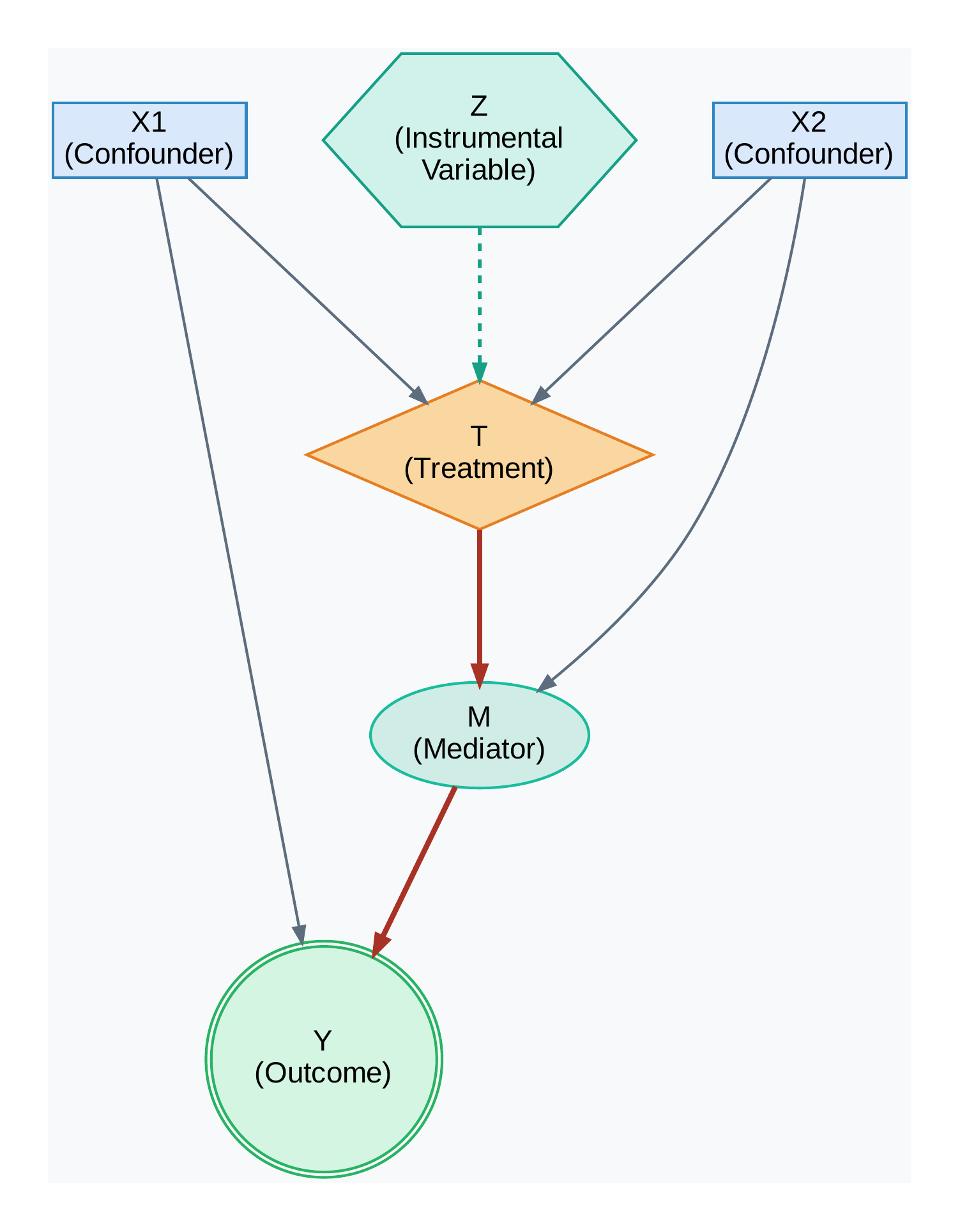}
        \caption{Ablation Study Scenario}
        \label{fig:sub_dag_advanced_mediation}
    \end{subfigure}
    \hfill
    \begin{subfigure}[b]{0.25\textwidth}
        \centering
        \includegraphics[width=\linewidth]{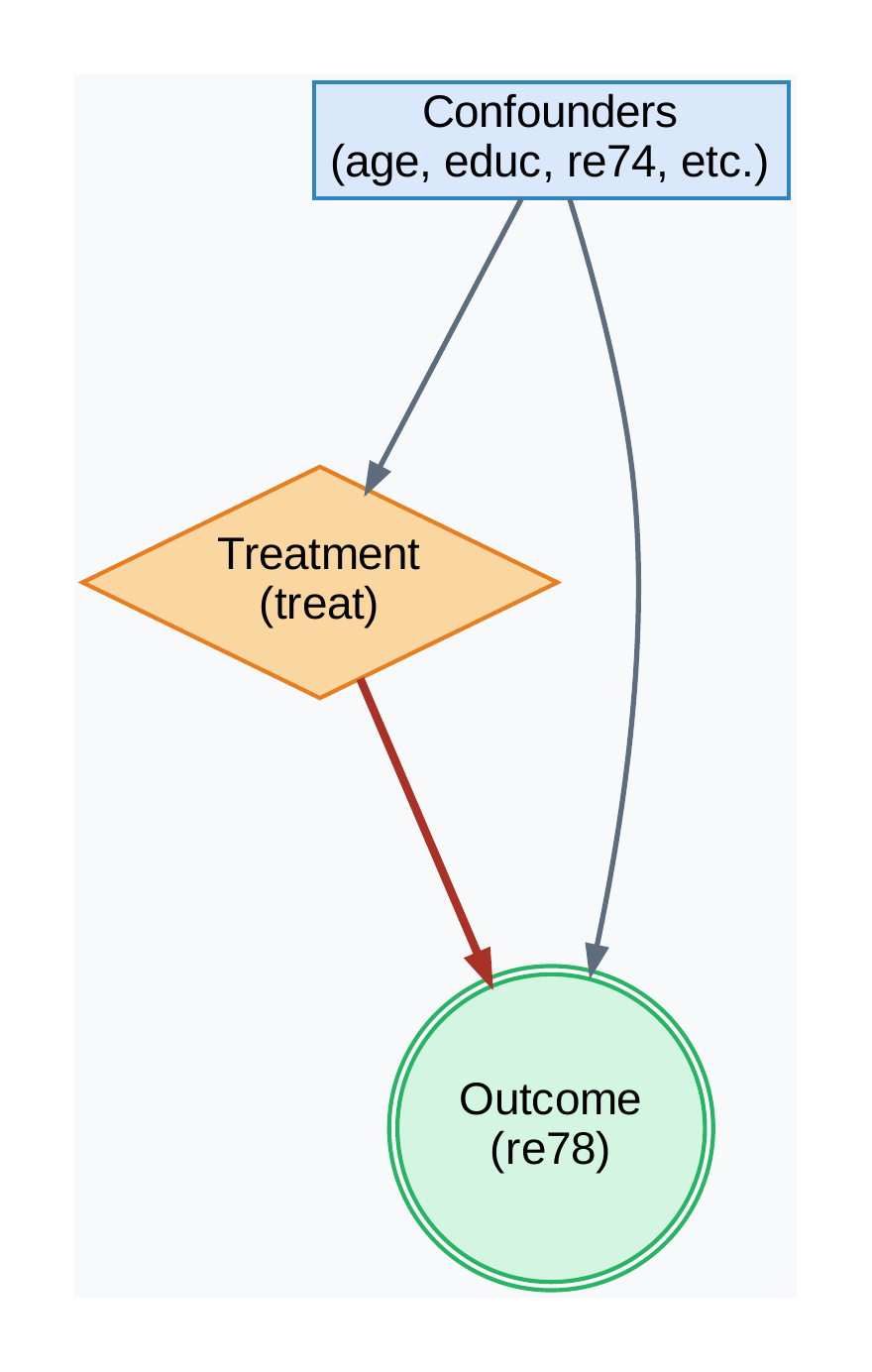}
        \caption{Lalonde Confounding Structure}
        \label{fig:sub_dag_lalonde}
    \end{subfigure}
    \caption{Directed Acyclic Graphs (DAGs) for key experiments. (a) A structure designed to challenge propensity score methods. (b) A mediation structure used for the ablation study. (c) The standard confounding structure assumed for both Lalonde-based experiments.}
    \label{fig:all_dags_combined}
\end{figure*}
\FloatBarrier

\subsection{Act I: Establishing State-of-the-Art Predictive Fidelity}
We first establish that our principled design achieves superior predictive fidelity on standard causal inference benchmarks.

\subsubsection{Robustness in Non-Linear Confounding Scenarios}\label{sec:psm_fail}
We tested our model in a challenging synthetic scenario (Figure~\ref{fig:sub_dag_psm_fail}) designed with highly non-linear confounding to cause propensity-based methods to fail. Table~\ref{tab:psm_fail_results} shows the results. While Causal Forest is exceptionally accurate on this specific DGP, our \textbf{BELM-MDCM} framework secures its position as the second most accurate method, delivering a highly stable and competitive ATE estimate. Crucially, it significantly outperforms the entire suite of propensity-based methods and powerful estimators like DML in accuracy. The high standard deviation of DML highlights its unreliability in this context, validating our model as a robust estimator where traditional approaches are compromised.

\begin{table}[hbt!]
    \small
    \centering
    \caption{ATE Estimation on the PSM Failure Scenario (True ATE = 5000). We report the mean ATE and standard deviation across multiple runs.}
    \label{tab:psm_fail_results}
    \begin{tabularx}{\linewidth}{X l r}
        \toprule
        \textbf{Method} & \textbf{Mean ATE $\pm$ Std Dev} & \textbf{Absolute Error} \\
        \midrule
        \textbf{BELM-MDCM} & \textbf{5266.87 $\pm$ 197.14} & \textbf{266.87} \\
        \midrule
        Causal Forest & 4895.77 $\pm$ 69.26 & 104.23 \\
        Propensity Score Stratification & 5309.38 $\pm$ 185.36 & 309.38 \\
        Linear Regression & 5348.82 $\pm$ 23.23 & 348.82 \\
        Propensity Score Matching & 5353.93 $\pm$ 191.36 & 353.93 \\
        Inverse Propensity Weighting & 5385.68 $\pm$ 52.03 & 385.68 \\
        Double Machine Learning & 4285.63 $\pm$ 550.97 & 714.37 \\
        \bottomrule
    \end{tabularx}
\end{table}

\subsubsection{Accuracy and Robustness on Real-World Observational Data}\label{sec:real_world_lalonde}
We next evaluated our framework on the canonical Lalonde dataset \citep{lalonde1986evaluating}, a challenging real-world benchmark with a known RCT ground truth. Table~\ref{tab:lalonde_ate_results} demonstrates the comprehensive superiority of our \textbf{BELM-MDCM} framework. It achieved a mean ATE estimate of \textbf{1567.36 $\pm$ 201.62}, the lowest error among all methods that correctly identified the treatment effect's positive direction. More critically, the results highlight a stark contrast in reliability. Classical methods failed entirely, while the powerful Causal Forest baseline suffered from extreme instability (Std Dev of \textbf{785.59}). In contrast, \textbf{BELM-MDCM} exhibited remarkable robustness, with a standard deviation approximately four times lower. This outstanding performance on a canonical benchmark validates that our framework delivers accurate estimates with the consistency essential for trustworthy causal inference.

\begin{table}[hbt!]
    \small
    \centering
    \caption{ATE Estimation Stability on the Lalonde Dataset (RCT Benchmark ATE $\approx 1794$). Results for all models are reported as Mean $\pm$ Standard Deviation across 5 independent runs.}
    \label{tab:lalonde_ate_results}
    \begin{tabularx}{\linewidth}{X l r}
        \toprule
        \textbf{Method} & \textbf{ATE (Mean $\pm$ Std)} & \textbf{Abs. Error (Mean)} \\
        \midrule
        \textbf{BELM-MDCM} & \textbf{1567.36 $\pm$ 201.62} & \textbf{226.64} \\
        \midrule
        Causal Forest & 1085.30 $\pm$ 785.59 & 708.70 \\
        Linear Regression & 46.33 $\pm$ 76.80 & 1747.67 \\
        Propensity Score Matching & -3.96 $\pm$ 118.37 & 1797.96 \\
        Propensity Score Stratification & -35.54 $\pm$ 81.44 & 1829.54 \\
        Propensity Score Weighting & -122.55 $\pm$ 50.51 & 1916.55 \\
        Double Machine Learning & nan $\pm$ nan & nan \\
        \bottomrule
    \end{tabularx}
\end{table}

\subsubsection{High-Fidelity ITE Estimation and Stability Analysis}\label{sec:ite_estimation}
\noindent\textbf{Objective.} We evaluate performance at the individual level using a semi-synthetic version of the Lalonde dataset. This experiment leverages real-world covariates and assumes the causal structure depicted in Figure~\ref{fig:sub_dag_lalonde}. To rigorously assess both accuracy and reliability, we follow our Individual Model Evaluation protocol, reporting the mean and standard deviation of performance across 5 independent runs for each method.

\noindent\textbf{Results.} The PEHE results, presented in Table~\ref{tab:lalonde_pehe_results}, confirm the exceptional fidelity and robustness of our framework. \textbf{BELM-MDCM} achieves the lowest average PEHE score of \textbf{537.84} and demonstrates remarkable stability with the lowest standard deviation of just \textbf{60.11}. This performance is closely followed by Causal Forest. However, the results also highlight the instability of other meta-learners; X-Learner, in particular, exhibits extremely high variance, with a standard deviation more than three times larger than its competitors, rendering its single-run estimates unreliable. This highlights the dual advantage of our framework: superior accuracy combined with consistent, trustworthy performance. Figure~\ref{fig:ite_accuracy_scatter} provides visual confirmation, showing the tight clustering of our model's ensembled ITE estimates around the ground truth.

\begin{figure}[hbt!]
    \centering
    \includegraphics[width=0.75\linewidth]{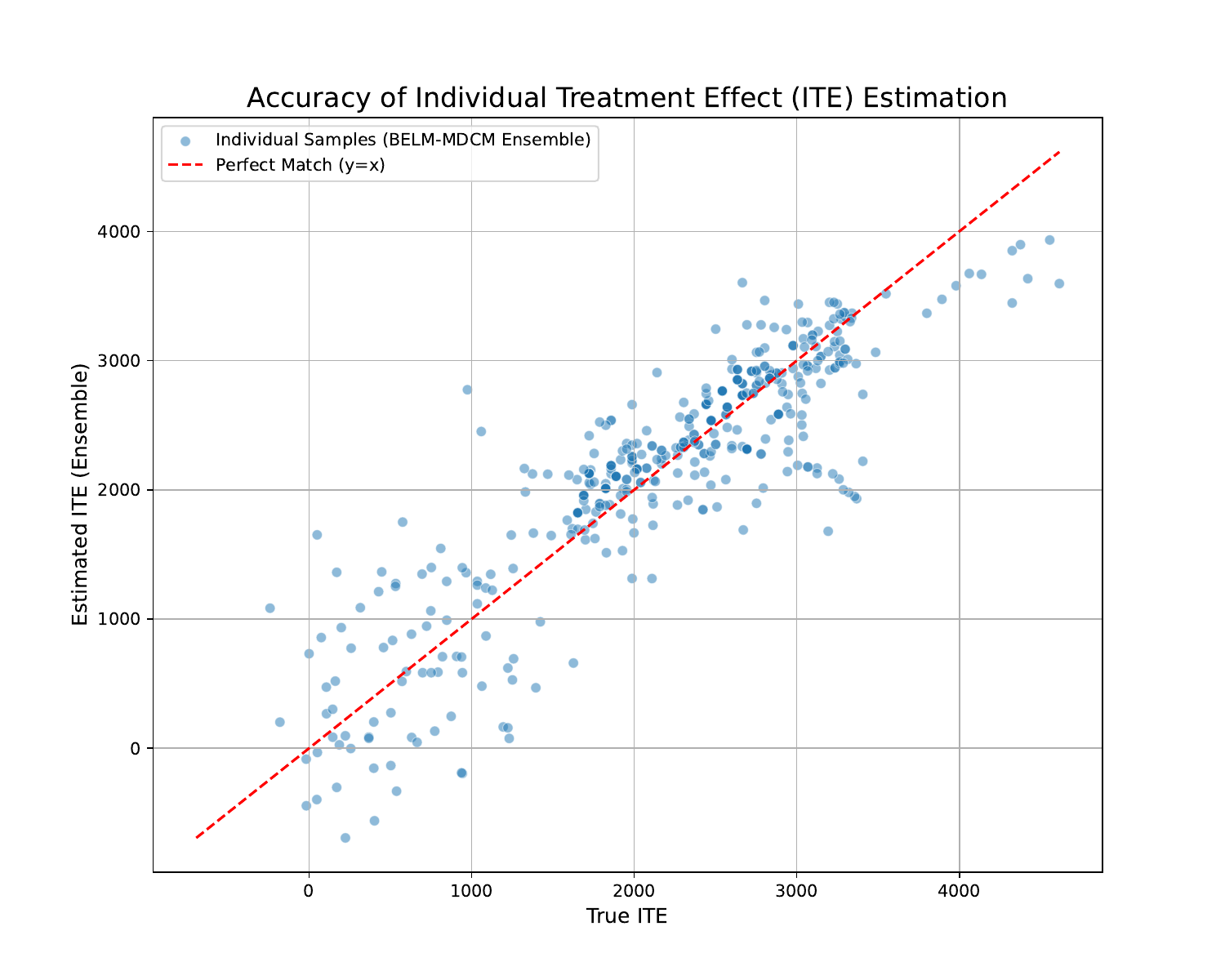}
    \caption{Accuracy of Individual Treatment Effect (ITE) Estimation on the semi-synthetic Lalonde dataset. The plot shows the ensembled estimated ITE from our model versus the true ITE. The tight clustering of our model's estimates (blue dots) around the perfect-match line (red dash) visually demonstrates its low PEHE score.}
    \label{fig:ite_accuracy_scatter}
\end{figure}
\begin{table}[hbt!]
    \small
    \centering
    \caption{ITE Estimation Accuracy (PEHE) on the Semi-Synthetic Lalonde Dataset. Results are reported as Mean $\pm$ Standard Deviation across 5 independent runs. Lower is better.}
    \label{tab:lalonde_pehe_results}
    \begin{tabularx}{\linewidth}{X r}
        \toprule
        \textbf{Method} & \textbf{PEHE Score (Mean $\pm$ Std)} \\
        \midrule
        \textbf{BELM-MDCM} & \textbf{537.84 $\pm$ 60.11} \\
        \midrule
        Causal Forest & 563.90 $\pm$ 73.66 \\
        S-Learner & 816.26 $\pm$ 79.17 \\
        X-Learner & 1546.38 $\pm$ 679.09 \\
        \bottomrule
    \end{tabularx}
\end{table}
\begin{table}[hbt!]
    \small
    \centering
    \caption{Causal Mechanism Fidelity (CMI-Score) on the Semi-Synthetic Lalonde Dataset. Results are reported as Mean $\pm$ Standard Deviation across 5 runs. Higher is better.}
    \label{tab:lalonde_cmi_results}
    \begin{tabularx}{\linewidth}{X r}
        \toprule
        \textbf{Method} & \textbf{CMI-Score (Mean $\pm$ Std)} \\
        \midrule
        S-Learner & 0.9905 $\pm$ 0.0062 \\
        \textbf{BELM-MDCM} & \textbf{0.9824 $\pm$ 0.0092} \\
        Causal Forest & 0.9786 $\pm$ 0.0099 \\
        X-Learner & 0.9782 $\pm$ 0.0145 \\
        T-Learner & 0.9555 $\pm$ 0.0113 \\
        \bottomrule
    \end{tabularx}
\end{table}

\FloatBarrier

\subsection{Act II: Uncovering the Accuracy-Invertibility Trade-off}
\label{sec:killer_experiment}
We now conduct the pivotal experiment of our study: a deep diagnostic analysis using our novel CIC-Score to reveal the trade-off between predictive accuracy and mechanism invertibility. This provides the core empirical evidence for our thesis by comparing three paradigms: our BELM-MDCM (Learned Invertibility), a DDIM variant (Flawed Invertibility), and a classic RF-ANM (Assumed Invertibility).

The results in Table~\ref{tab:golden_table} decisively validate our framework's principles. Our \textbf{BELM-MDCM} is the clear leader, achieving the lowest PEHE score (\textbf{1071.95}) with high stability. Critically, its CIC-Score of \textbf{0.3679} is orders of magnitude higher than the alternatives, proving its unique ability to learn an invertible mapping that conserves causal information. In stark contrast, the \textbf{DDIM-MDCM} model exemplifies the failure predicted by our theory: its near-zero CIC-Score confirms a near-total collapse of causal information due to SRE, leading to unreliable predictions (high PEHE and variance). The classical \textbf{RF-ANM}, while structurally invertible, lacks the capacity to learn the true mechanism, resulting in a zero CIC-Score and poor accuracy. This "Golden Table" experiment underscores that both structural integrity \textit{and} powerful modeling capacity are essential for high-fidelity causal inference.

\begin{table}[hbt!]
    \small
    \centering
    \caption{The "Ultimate Golden Table": A comparative analysis of model classes on predictive accuracy (PEHE) and structural integrity (CIC-Score). This table includes the NF-SCM baseline, which empirically validates the \textit{likelihood-fidelity dilemma}. Results are reported as Mean $\pm$ Standard Deviation across 5 runs. Lower PEHE is better; higher CIC-Score is better.}
    \label{tab:golden_table}
    \begin{tabularx}{\linewidth}{X r r}
        \toprule
        \textbf{Model} & \textbf{PEHE (Mean $\pm$ Std)} & \textbf{CIC-Score (Mean $\pm$ Std)} \\
        \midrule
        RF-ANM & 1533.18 $\pm$ 134.24 & 0.0000 $\pm$ 0.0000 \\
        DDIM-MDCM & 2085.98 $\pm$ 788.12 & 0.0065 $\pm$ 0.0130 \\
        NF-SCM & \textbf{442229.96 $\pm$ 66963.73} & 0.1572 $\pm$ 0.0232 \\
        \midrule
        \textbf{BELM-MDCM} & \textbf{1071.95 $\pm$ 152.11} & \textbf{0.3679 $\pm$ 0.0000} \\
        \bottomrule
    \end{tabularx}
\end{table}
\begin{figure}[htbp]
    \centering
    \includegraphics[width=0.75\textwidth]{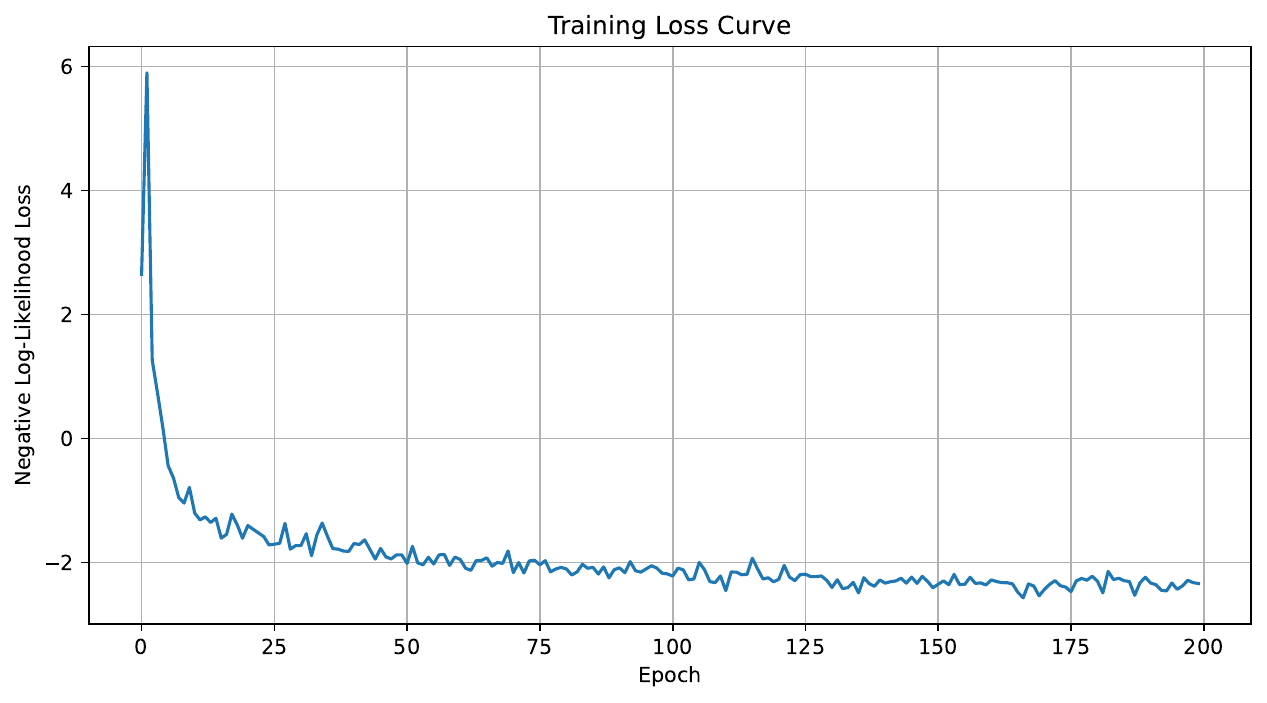}
    \caption{The training loss curve for the Conditional Normalizing Flow (NF) baseline. The smooth, stable convergence to a low negative log-likelihood value indicates a successful statistical training run. However, this did not correspond to learning the true causal mechanism, as evidenced by its extremely high PEHE score.}
    \label{fig:nf_loss_curve}
\end{figure}

\paragraph{The Likelihood-Fidelity Dilemma: Why Natively Invertible Models Can Fail.}
To rigorously test the limits of models that are natively invertible, we conducted a comprehensive stability analysis on a Conditional Normalizing Flow (NF) baseline, a model class that satisfies the Causal Information Conservation principle by construction (SRE $\equiv 0$). Across five independent runs with different random seeds, the NF model consistently demonstrated successful statistical learning, with its training loss stably converging to a high log-likelihood in each instance (a representative example is shown in Figure~\ref{fig:nf_loss_curve}).

However, this statistical success was starkly contrasted by a \textbf{systematic and catastrophic failure} in the causal task. The model yielded an average PEHE score of \textbf{442,229.96 $\pm$ 66,963.73}, confirming that its generated counterfactuals were fundamentally incorrect. This consistent result provides decisive evidence for a critical challenge we term the \textit{likelihood-fidelity dilemma}: a model can perfectly learn to replicate a data distribution while remaining completely ignorant of the underlying causal mechanism.

The root of this dilemma is the fundamental mismatch between the optimization objective and the causal goal. The maximum likelihood objective incentivizes the NF to find \textit{any} invertible mapping that transforms the data to a simple base distribution. While an infinite number of such mappings may be statistically equivalent, \textbf{only one} corresponds to the true, unique causal data-generating process. Without a guiding signal, the NF is mathematically predisposed to learn a causally-incorrect "statistical shortcut." This finding powerfully underscores the contribution of our \textbf{Hybrid Training Objective}. It acts as the crucial causal inductive bias that resolves this dilemma, compelling the model to learn the unique, causally salient structure and enabling valid causal inference where pure likelihood-based methods, even those with zero SRE, are destined to fail.
\FloatBarrier

\subsection{Act III: Unlocking Deeper Causal Inquiry with a High-Fidelity Model}
An information-conserving model serves as a trustworthy "world model" for deep causal inquiry. We showcase three applications uniquely enabled by our framework's high-fidelity counterfactuals.

\paragraph{Heterogeneity Analysis: Conditional ATE (CATE).}
A reliable ITE model can act as a ``causal microscope.'' We use it to explore treatment effect heterogeneity by estimating the Conditional Average Treatment Effect (CATE) for subpopulations. By averaging the results from our five independently trained models, we obtain stable and robust estimates. Our model's mean CATE estimates track the true CATE trends with high fidelity across different education levels, a capability crucial for policy-making. For instance, for individuals with an education level of 3.0, the estimated CATE was \$2562.79 (true CATE: \$2280.79). For levels 8.0, 12.0, and 16.0, the estimates were \$2092.91 (true: \$2118.61), \$2253.76 (true: \$2384.44), and \$2434.89 (true: \$2490.21), respectively, demonstrating a close correspondence to the ground truth.

\begin{figure}[hbt!]
    \centering
    \includegraphics[width=0.7\linewidth]{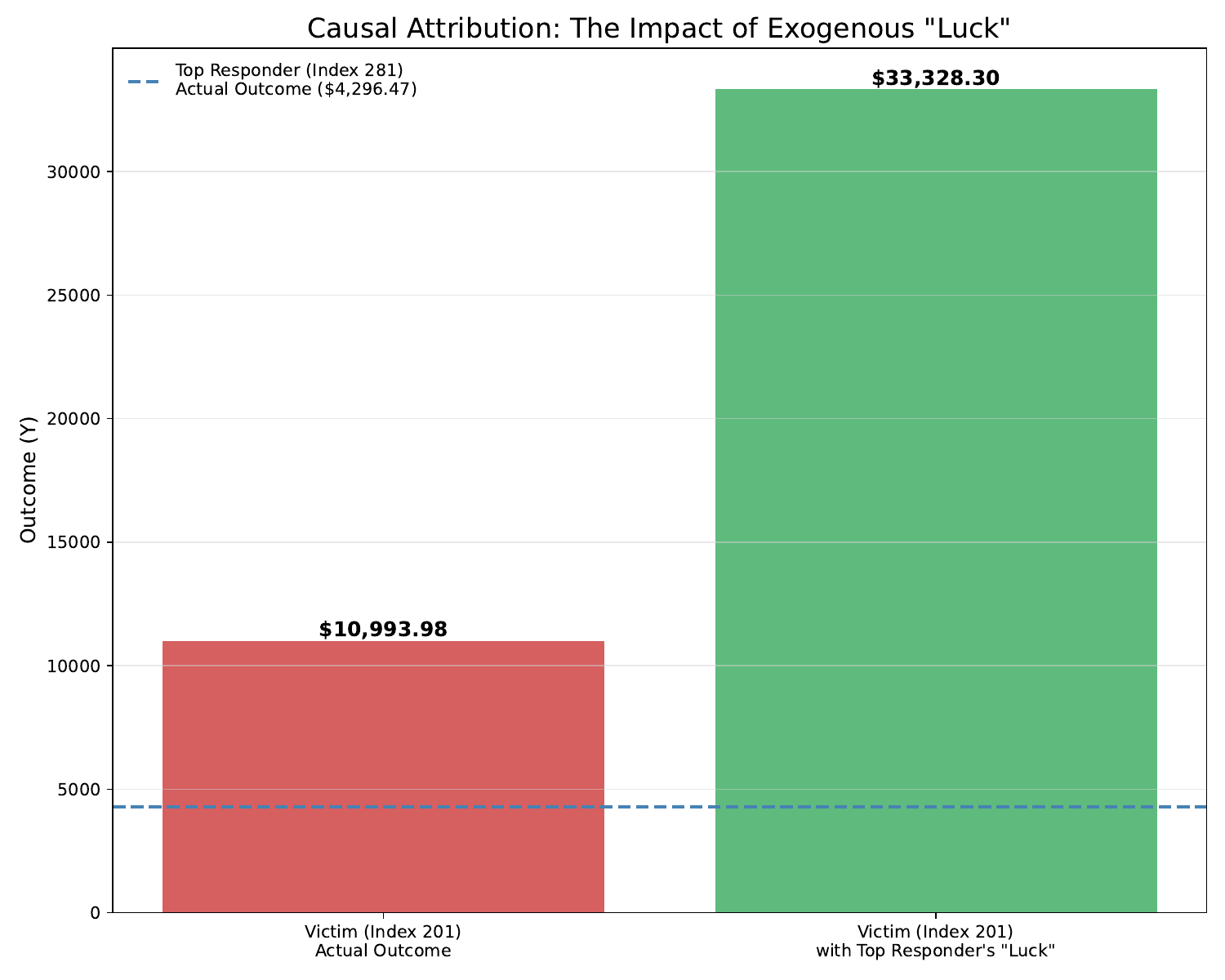}
    \caption{Causal Attribution Analysis. This chart shows the counterfactual outcome for the 'Victim' if they had possessed the individual-specific exogenous factors of the 'Top Responder'.}
    \label{fig:causal_attribution}
\end{figure}
\paragraph{Causal Attribution: Isolating the Effect of Exogenous Factors.}
We conducted a causal attribution experiment via a counterfactual intervention of the form $do(U_{\text{victim}} := u_{\text{responder}})$. Figure \ref{fig:causal_attribution} shows that our framework can losslessly recover these factors, revealing that unobserved exogenous 'luck'\footnote{In this context, the term 'luck' serves as an intuitive shorthand for the exogenous noise variable $U$. Within the Structural Causal Model (SCM) framework, $U$ represents all unobserved, individual-specific factors (e.g., intrinsic ability, random chance, measurement errors) that, together with the observed parent variables ($\mathbf{Pa}$), determine the final outcome for an individual.} had a massive and stable causal effect, averaging a \textbf{+22,334.32} change in the 'Victim's' outcome. This capacity for reliable attribution is a unique advantage of our information-conserving framework.
\begin{figure}[hbt!]
    \centering
    \begin{subfigure}{0.48\textwidth}
        \includegraphics[width=\linewidth]{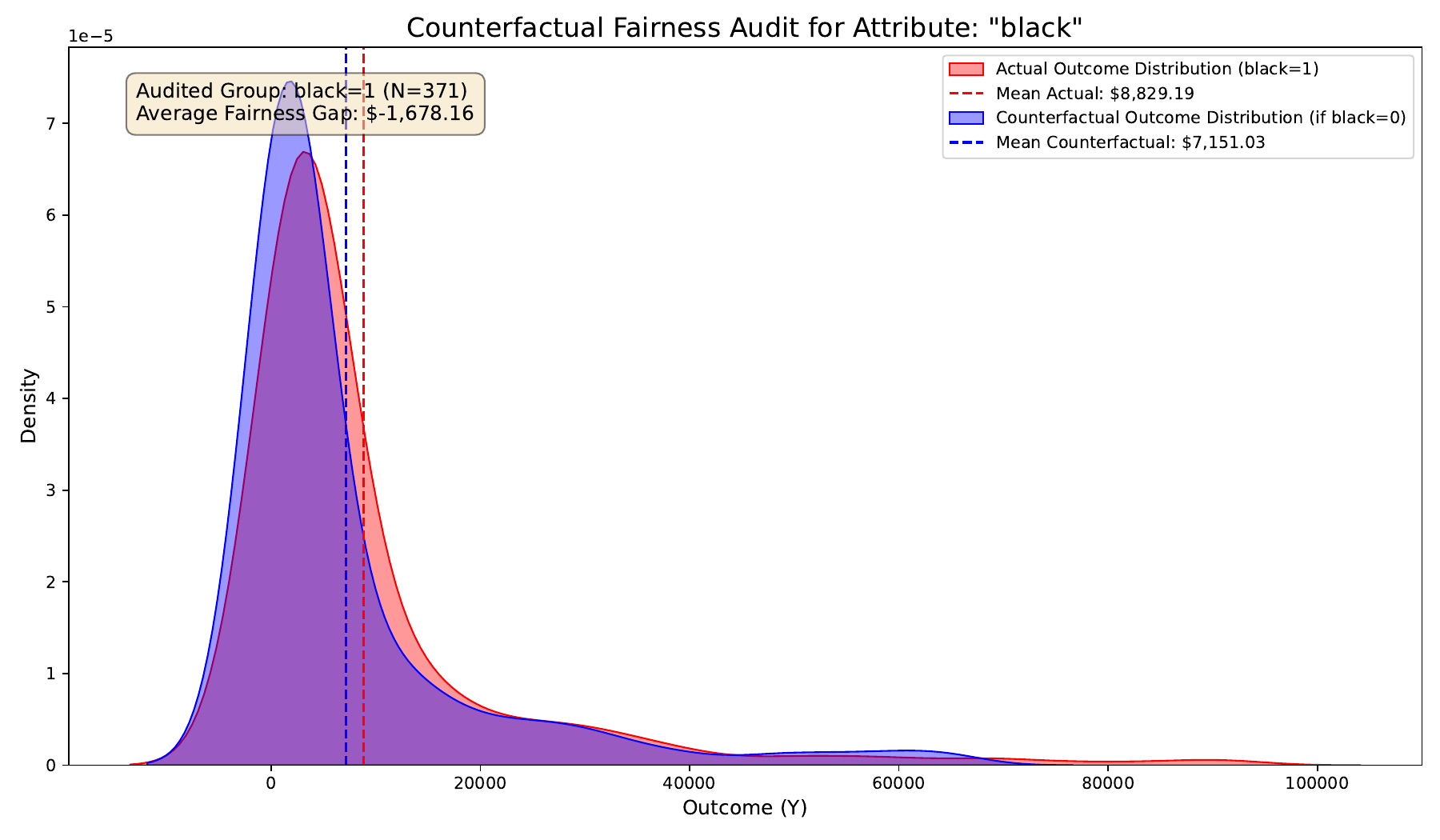}
        \caption{Fairness audit for attribute 'black'.}
        \label{fig:fairness_audit_black}
    \end{subfigure}
    \hfill
    \begin{subfigure}{0.48\textwidth}
        \includegraphics[width=\linewidth]{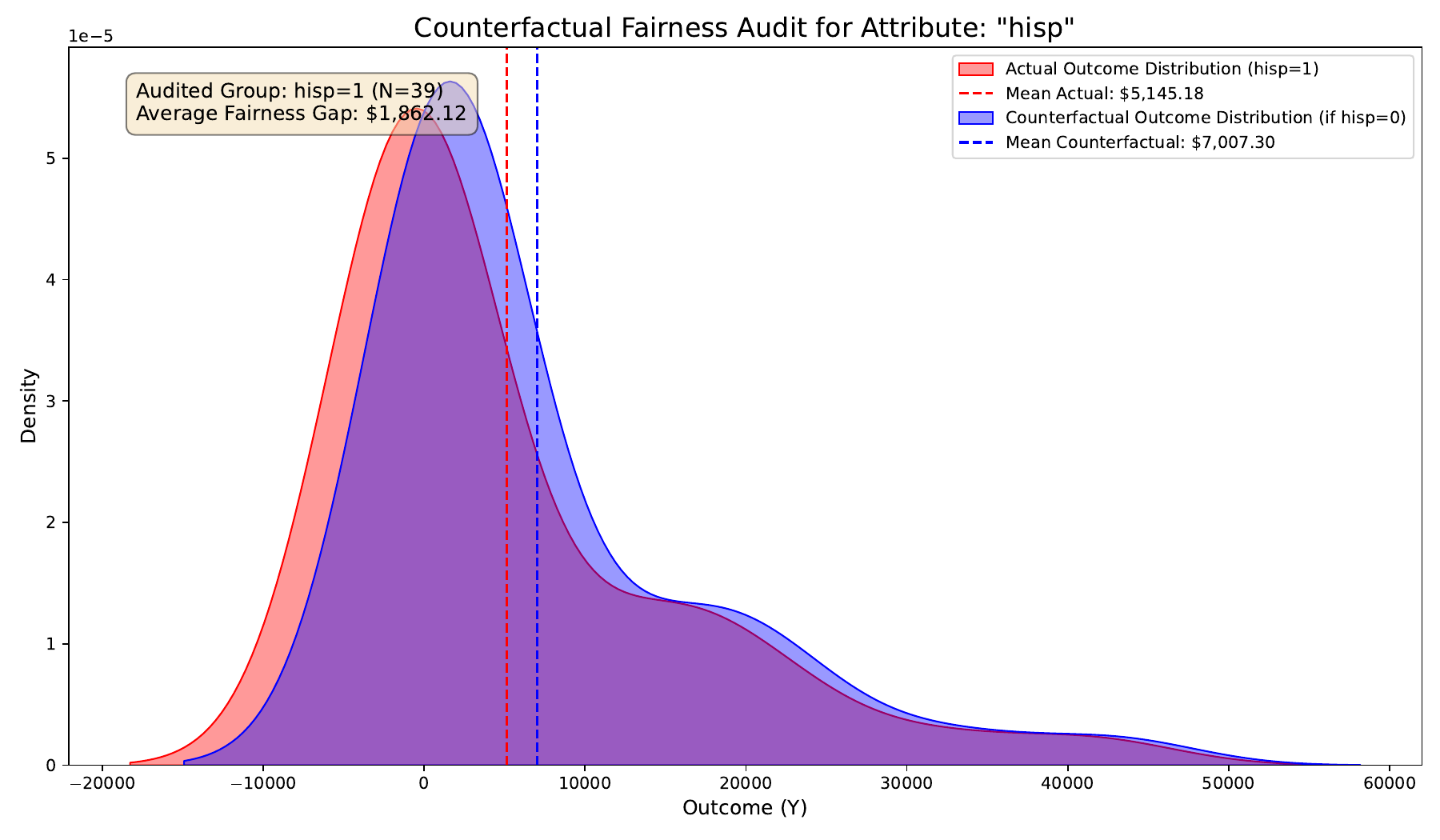}
        \caption{Fairness audit for attribute 'hisp'.}
        \label{fig:fairness_audit_hisp}
    \end{subfigure}
    \caption{Counterfactual fairness audits reveal significant outcome disparities based on sensitive attributes. The plots show the distribution of actual outcomes for each group versus the distribution of their counterfactual outcomes had their sensitive attribute been different.}
    \label{fig:fairness_audits}
\end{figure}
\paragraph{Counterfactual Fairness Audit.}
Finally, we applied our framework to a counterfactual fairness audit. Only a model that faithfully represents the data generating process can reliably answer questions about fairness. Figure \ref{fig:fairness_audits} reveals significant and stable disparities: our model estimates an average fairness gap of \textbf{-1,678.16} for the 'black' attribute and \textbf{+1,862.12} for the 'hisp' attribute, demonstrating its capacity as a powerful tool for ethical audits.
\FloatBarrier

\subsection{Act IV: Final Validation: Stress Tests and Ablation Study}
We conclude by subjecting the framework to two final tests: a stress test on a non-invertible SCM and a comprehensive ablation study.

\subsubsection{Stress Test on a Non-Invertible SCM}
\label{sec:exp_many_to_one}
We tested our framework's robustness when the theoretical assumption of an invertible SCM is violated, using a DGP where $Y \propto U_Y^2$. The results in Table~\ref{tab:many_to_one_results} and Figure~\ref{fig:many_to_one_chart} decisively validate our hypothesis. On the definitive metric of individual-level fidelity (PEHE), our zero-SRE BELM framework achieves an error of \textbf{0.77}, a \textbf{44\% reduction} compared to the SRE-prone DDIM model. This empirically proves that even when the true SCM is non-invertible, eliminating algorithmic SRE provides a substantial advantage. This result is fully consistent with our theoretical analysis in Appendix~\ref{sec:appendix_inversion_and_non_invertible} (Theorem~\ref{thm:non_inv_bound_tight}), where the total error is decomposed into algorithmic, modeling, and representational errors. By eliminating the algorithmic error ($E_{SR} \equiv 0$), our framework's performance approaches the theoretical limit set by the other two components.

Interestingly, the DDIM sampler achieves a slightly higher KMD-Score. We hypothesize that its inherent inversion noise acts as a form of implicit regularization, making the marginal generated distribution appear closer to the truth, even while individual-level counterfactuals are less accurate. This highlights the important distinction between distributional fidelity and individual-level causal accuracy.

\begin{figure}[hbt!]
    \centering
    \includegraphics[width=0.8\linewidth]{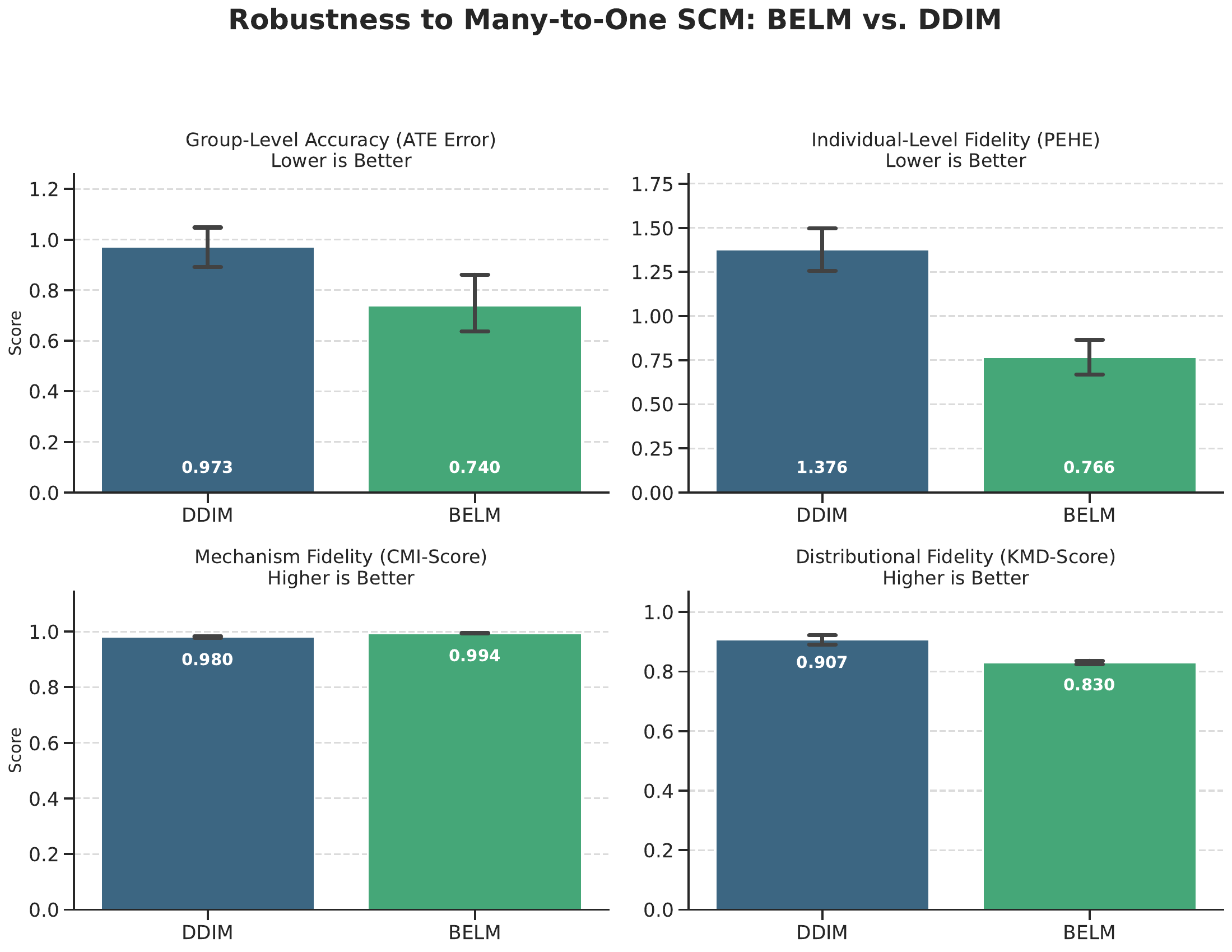}
    \caption{Robustness comparison on the many-to-one SCM. Our zero-SRE framework (BELM) demonstrates significantly superior performance on PEHE.}
    \label{fig:many_to_one_chart}
\end{figure}

\begin{table}[hbt!]
    \small
    \centering
    \caption{Performance on the non-invertible SCM ($Y \propto U^2$). Results are averaged over 5 runs ($\pm$ std). Lower PEHE is better.}
    \label{tab:many_to_one_results}
    \begin{tabularx}{\linewidth}{l*{2}{>{\centering\arraybackslash}X}}
        \toprule
        \textbf{Sampler} & \textbf{PEHE(Mean $\pm$ Std)} & \textbf{KMD-Score(Mean $\pm$ Std)} \\
        \midrule
        \textbf{BELM (Zero SRE)} & \textbf{0.77 $\pm$ 0.16} & 0.830 $\pm$ 0.009 \\
        DDIM (with SRE) & 1.38 $\pm$ 0.19 & \textbf{0.907 $\pm$ 0.023} \\
        \bottomrule
    \end{tabularx}
\end{table}
\FloatBarrier

\subsubsection{Ablation Study: Deconstructing the Framework's Success} \label{sec:ablation}
We conducted a comprehensive ablation study on a challenging synthetic dataset (Figure~\ref{fig:sub_dag_advanced_mediation}) to validate the contribution of each core component. The findings, presented in Table~\ref{tab:ablation_results_final} and Figure~\ref{fig:ablation_plot}, provide unequivocal evidence for our integrated design. The full BELM-MDCM model establishes the gold standard for both accuracy and stability. The study reveals three critical insights:
\begin{itemize}[leftmargin=*, noitemsep]
    \item \textbf{Decisive Role of the Hybrid Objective:} Removing it causes a catastrophic decline in performance (400\%+ error increase), demonstrating it is a core driver of the causal inductive bias, not merely a fine-tuning mechanism.
    \item \textbf{Critical Importance of Targeted Modeling:} Removing it leads to a complete collapse in model stability (Std Dev explodes from 4.57 to 138.01), validating our theoretical analysis on complexity control. A judicious allocation of model complexity is paramount for reproducibility.
    \item \textbf{Robust Advantage of Exact Invertibility:} Replacing BELM with DDIM leads to a clear degradation in both accuracy and stability, confirming that SRE from approximate inversion systematically erodes the quality of the final causal estimate.
\end{itemize}

\begin{figure}[htbp]
    \centering
    \includegraphics[width= 0.8\linewidth]{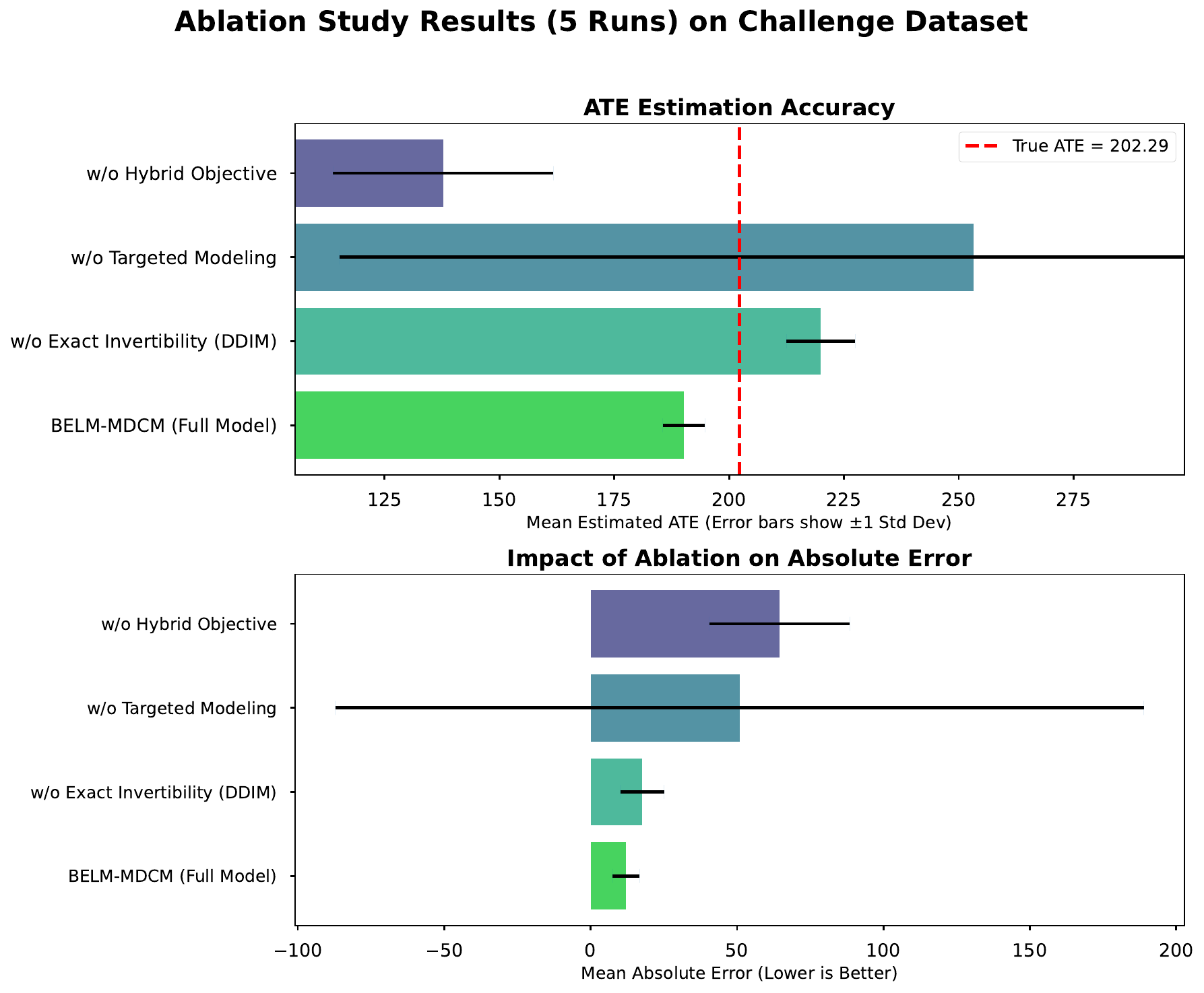}
    \caption{Visualization of the ablation study results. The top panel shows the mean estimated ATE relative to the true value (red dashed line), with error bars indicating $\pm1$ standard deviation. The bottom panel highlights the mean absolute error for each configuration. The full BELM-MDCM model is demonstrably the most accurate and stable.}
    \label{fig:ablation_plot}
\end{figure}

\begin{table}[htbp]
    \small
    \centering
    \caption{Ablation study results on a challenging synthetic dataset (True ATE = 202.29), validating the necessity of each framework component. The mean and standard deviation are computed over 5 runs.}
    \label{tab:ablation_results_final}
    \begin{tabularx}{\linewidth}{X r r r}
        \toprule
        \textbf{Model Configuration} & \textbf{Mean Est. ATE ($\pm$ Std)} & \textbf{Abs. Error} \\
        \midrule
        \textbf{BELM-MDCM (Full Model)} & \textbf{190.16 $\pm$ 4.57} & \textbf{12.14} \\
        \midrule
        w/o Exact Invertibility (DDIM) & 219.98 $\pm$ 7.48 & 17.68 \\
        w/o Hybrid Objective & 137.77 $\pm$ 23.99 & 64.53 \\
        w/o Targeted Modeling & 253.20 $\pm$ 138.01 & 50.90 \\
        \bottomrule
    \end{tabularx}
\end{table}
\paragraph{Conclusion.} The ablation study confirms that our framework's three core design principles: \textbf{analytical invertibility, a hybrid training objective, and targeted modeling}—work in synergy. The removal of any single component creates a significant vulnerability, validating the integrity and effectiveness of our integrated architectural design.
\FloatBarrier

\section{Causal Information Conservation as a Unifying Principle}\label{sec:unifying_principle}

The principle of Causal Information Conservation extends beyond a foundation for our model; it offers a unifying lens—a new taxonomy—for analyzing the suitability of \textbf{any} generative model for individual-level causal inference. Applying this principle and its operational metric, SRE, allows us to situate our work and clarify its unique advantages within the broader landscape.

\subsection{Normalizing Flows: Natively Information-Conserving}

Normalizing Flows (NFs) \citep{dinh2014nice, dinh2017density, kingma2018glow} are designed around an invertible mapping. By construction, their \textbf{Structural Reconstruction Error} is identically zero, making NFs a native implementation of the Causal Information Conservation principle. However, their \textbf{strengths come with limitations}: the requirement of a tractable Jacobian imposes heavy architectural constraints, which can limit expressive power and introduce strong topological assumptions on the data manifold.

\subsection{VAEs and GANs: Architecturally High-SRE}

Variational Autoencoders (VAEs) \citep{kingma2014autoencoding} and Generative Adversarial Networks (GANs) \citep{goodfellow2014generative} are fundamentally ill-suited for this task. Their \textbf{Structural Reconstruction Error} is large and non-zero due to a fundamental \textit{architectural mismatch}, as their separate encoder and decoder networks lack any structural guarantee of being inverses. Furthermore, their \textbf{sources of information loss} are inherent to their design; optimization objectives like the ELBO's KL-divergence term actively encourage lossy compression, theoretically impeding the recovery of precise causal information.

\subsection{A Comparative Perspective on Diffusion Models}

Viewed through our principle, diffusion models occupy a unique and compelling space in this taxonomy. \textbf{Standard diffusion-based approaches}, using samplers like DDIM, aspire to conserve information, but their reliance on approximate inversion results in a non-zero SRE; they are "aspirational" but ultimately lossy. \textbf{In contrast, BELM-MDCM (this work)} achieves zero SRE by integrating an analytically invertible sampler, matching the theoretical purity of Normalizing Flows but without their rigid architectural constraints. Furthermore, unlike NFs trained with a generic likelihood objective, our framework's Hybrid Training objective provides a causally-oriented inductive bias. BELM-MDCM thus uniquely combines the rigorous invertibility of NFs with the modeling flexibility and task-specific power of the diffusion paradigm, making it ideally suited for principled, high-fidelity causal inference.

\section{Conclusion and Future Work}

This paper introduced \textbf{Causal Information Conservation} as a guiding principle for the emerging field of diffusion-based causal inference. Our primary contribution is not the concept of invertibility itself, but framing it as a foundational design requirement and identifying the \textbf{Structural Reconstruction Error (SRE)} as the precise, quantifiable cost of its violation.

Our proposed framework, \textbf{BELM-MDCM}, serves as a constructive proof of this principle. By being architected around analytical invertibility, it is the first to achieve zero SRE by design, shifting the focus from mitigating numerical errors to upholding a fundamental causal principle. This work provides a foundational blueprint and a more rigorous standard for applying the power of diffusion models to the profound challenges of causal inference, reconciling their flexibility with the logical rigor demanded by classical theory.

\subsection{Limitations and Future Work}\label{sec:limitations}
Our work highlights several avenues for future research:

\begin{itemize}
    \item \textbf{Handling Non-Invertible SCMs:} Our framework excels when the true SCM is invertible. While our stress test shows robustness when this is violated, developing models inherently resilient to such misspecifications is a key challenge. Empirically validating the proposed prior-matching regularizer (Appendix \ref{sec:appendix_inversion_and_non_invertible}) is a concrete next step.

    \item \textbf{Robustness to Graph Misspecification:} Like most SCM-based methods \citep{peters2017elements}, our framework assumes a correctly specified causal graph. Analyzing how structural errors (e.g., omitted confounders) propagate through our error decomposition framework is a significant future research direction.

    \item \textbf{Formalizing CIC within Information Theory:} Our work defines CIC as an operational principle. A compelling direction is to formalize it within a rigorous information-theoretic framework, for instance by proving that zero SRE maximizes the mutual information $I(U; \hat{U})$ between the true and recovered noise, connecting our work to rate-distortion theory.
    
    \item \textbf{Scalability and Generalizability:} While powerful, the BELM sampler is computationally intensive. Improving its efficiency for high-dimensional settings is an important practical challenge. Furthermore, extending the principles of CIC and zero-SRE modeling to other data modalities, such as time-series or images, represents an exciting frontier.
\end{itemize}

\acks{We thank the anonymous reviewers for their insightful feedback which significantly improved the clarity and rigor of this paper.}

\appendix

\section{Core Diffusion Model Equations}\label{sec:appendix_theory}
\renewcommand{\theequation}{A.\arabic{equation}}\setcounter{equation}{0}

This appendix provides the essential equations for the diffusion models referenced in this work. Diffusion models learn a data distribution by training a neural network $\boldsymbol{\epsilon}_\theta$ to reverse a fixed, gradual noising process.

The model is trained by optimizing a simplified score-matching objective \citep{ho2020denoising}:
\begin{equation}
    \begin{split}
        L_{\text{simple}}(\theta) = \mathbb{E}_{t, \mathbf{x}_0, \boldsymbol{\epsilon}} \bigg[ \Big\| \boldsymbol{\epsilon} - \boldsymbol{\epsilon}_\theta\big(\sqrt{\bar{\alpha}_t}\mathbf{x}_0
        + \sqrt{1-\bar{\alpha}_t}\boldsymbol{\epsilon}, t\big) \Big\|^2 \bigg]
    \end{split}
\end{equation}
where $\bar{\alpha}_t$ is a predefined noise schedule. For deterministic generation and inversion, we use the Denoising Diffusion Implicit Model (DDIM) update step \citep{song2021denoising}:
\begin{equation}
    \begin{split}
        \mathbf{x}_{t-1} = \sqrt{\bar{\alpha}_{t-1}} \left( \frac{\mathbf{x}_t - \sqrt{1-\bar{\alpha}_t}\boldsymbol{\epsilon}_\theta(\mathbf{x}_t,t)}{\sqrt{\bar{\alpha}_t}} \right) + \sqrt{1-\bar{\alpha}_{t-1}}\boldsymbol{\epsilon}_\theta(\mathbf{x}_t,t)
    \end{split}
\end{equation}
This process can be viewed as a discretization of a continuous-time probability flow Ordinary Differential Equation (ODE) \citep{song2021score}:
\begin{equation}
\label{eq:ode_appendix}
    d\mathbf{x} = \left[ \frac{1}{2}\frac{d \log \alpha(s)}{ds}\mathbf{x}(s) - \frac{1}{2} \frac{d \log(1-\alpha(s))}{ds} \frac{\sqrt{1-\alpha(s)}}{\sqrt{\alpha(s)}} \boldsymbol{\epsilon}_\theta(\mathbf{x}(s), s) \right] ds
\end{equation}
Within our theoretical framework, the encoder operator $\mathbf{T}_\theta$ corresponds to solving this ODE forward in time (from $s=0$ to $s=1$), while the decoder operator $\mathbf{H}_\theta$ corresponds to solving it backward in time (from $s=1$ to $s=0$).

\section{Detailed Proofs for Identifiability (Theorems \ref{thm:identifiability} \& \ref{thm:correctness})}\label{sec:appendix_identifiability_proofs}
\renewcommand{\theequation}{B.\arabic{equation}}\setcounter{equation}{0}

This appendix provides detailed, dimension-specific proofs for the identifiability of the exogenous noise $U$ and the subsequent correctness of counterfactual generation. The core challenge lies in showing that the statistical independence of the latent code $Z$ from the parents $\mathbf{Pa}$ is a sufficient condition to establish an isomorphic relationship between $Z$ and $U$. The mathematical tools required differ based on the dimensionality of $U$.

\subsection{The High-Dimensional Case (\texorpdfstring{$d \ge 3$}{d >= 3})}
    For cases where the exogenous noise $U$ has a dimensionality $d \ge 3$, the proof leverages Liouville's theorem on conformal mappings.

    \begin{proof}[Proof of Theorems \ref{thm:identifiability} and \ref{thm:correctness} for $d \ge 3$]
        We adapt the proof from identifiable generative modeling \citep{chao2023interventional} to our conditional operator setting.

        Let $\mathbf{Q}_{\mathbf{pa}}(U) := \mathbf{T}_\theta(\mathbf{F}(\mathbf{pa}, U), \mathbf{pa})$ be the composite function mapping the noise $U$ to the latent code $Z$ for a given context $\mathbf{pa}$. By assumption, $\mathbf{Q}_{\mathbf{pa}}$ is invertible and differentiable. The core assumption, $Z \perp\!\!\!\perp \mathbf{Pa}$, implies that the conditional density $p_Z(z | \mathbf{pa})$ must equal a marginal density $p_Z(z)$ that is independent of $\mathbf{pa}$.

        Using the change of variables formula, we relate the density of $Z$ to that of $U$:
        \begin{equation}
            p_Z(z | \mathbf{pa}) = \frac{p_U(\mathbf{Q}_{\mathbf{pa}}^{-1}(z))}{\left| \det J_{\mathbf{Q}_{\mathbf{pa}}}(\mathbf{Q}_{\mathbf{pa}}^{-1}(z)) \right|}
        \end{equation}
        where $J_{\mathbf{Q}_{\mathbf{pa}}}$ is the Jacobian of $\mathbf{Q}_{\mathbf{pa}}$. Since the left-hand side is independent of $\mathbf{pa}$ and $p_U(\cdot)$ is a fixed distribution, this imposes a strong constraint on the Jacobian determinant. Under regularity conditions, this implies that the Jacobian $J_{\mathbf{Q}_{\mathbf{pa}}}(u)$ must be a scaled orthogonal matrix, making $\mathbf{Q}_{\mathbf{pa}}$ a conformal map.

        By Liouville's theorem, for dimensions $d \ge 3$, any conformal map must be a Möbius transformation (a composition of translations, scalings, orthogonal transformations, and inversions). For the map to be well-behaved, it must exclude the inversion component, which would introduce singularities. This is consistent with the regularity of functions representable by neural networks, simplifying the map to an affine form:
        \begin{equation}
            \mathbf{Q}_{\mathbf{pa}}(u) = \mathbf{A}_{\mathbf{pa}}u + \mathbf{d}_{\mathbf{pa}}
        \end{equation}
        where $\mathbf{A}_{\mathbf{pa}}$ is a scaled orthogonal matrix. The argument then uses the independence of the distribution's moments and support to show that $\mathbf{A}_{\mathbf{pa}}$ and $\mathbf{d}_{\mathbf{pa}}$ must be constant w.r.t. $\mathbf{pa}$. This leads to the isomorphic relationship $\mathbf{T}_\theta(\mathbf{F}(\mathbf{Pa}, U), \mathbf{Pa}) = \mathbf{A}U + \mathbf{d} = g(U)$.

        The proof of counterfactual correctness follows directly, as detailed in \citet{chao2023interventional}. The conditional isomorphism $\mathbf{H}_\theta(\mathbf{T}_\theta(\cdot, \mathbf{pa}), \mathbf{pa}) = \mathbf{I}$ combined with the identifiability result $\mathbf{T}_\theta(\mathbf{F}(\mathbf{pa}, u), \mathbf{pa}) = g(u)$ implies that $\mathbf{H}_\theta(g(u), \mathbf{pa}) = \mathbf{F}(\mathbf{pa}, u)$. The decoder thus perfectly mimics the true causal mechanism, making an intervention exact: $\hat{X}_{\boldsymbol{\alpha}} = \mathbf{H}_\theta(g(u), \boldsymbol{\alpha}) = \mathbf{F}(\boldsymbol{\alpha}, u) = X_{\boldsymbol{\alpha}}^{\text{true}}$.
    \end{proof}

\subsection{The One-Dimensional Case (\texorpdfstring{$d=1$}{d=1})}
    For the one-dimensional case where Liouville's theorem does not apply, this section provides a dedicated proof leveraging properties of 1D functions and a uniform noise assumption from \citep{chao2023interventional} that does not sacrifice generality.

    We first establish a helper lemma characterizing a specific class of 1D functions.
    \begin{lemma}\label{lem:1d_characterization}
        For $U, Z \subset \mathbb{R}$, consider a family of invertible functions $q_{\mathbf{pa}}: U \to Z$ for $\mathbf{pa} \in \mathcal{X}_\text{pa} \subset \mathbb{R}^d$. The derivative expression $\frac{dq_{\mathbf{pa}}}{du}(q_{\mathbf{pa}}^{-1}(z))$ is a function of $z$ only, i.e., $c(z)$, if and only if $q_{\mathbf{pa}}(u)$ can be expressed as
        $$q_{\mathbf{pa}}(u) = q(u + r(\mathbf{pa}))$$
        for some function $r$ and invertible function $q$.
    \end{lemma}
    \begin{proof}
        The proof is provided in \citep{chao2023interventional}. The reverse direction follows from direct differentiation, while the forward direction uses the inverse function theorem to show that the inverses $s_{\mathbf{pa}}(z) = q_{\mathbf{pa}}^{-1}(z)$ must have the same derivative $1/c(z)$, implying they can only differ by an additive constant, which yields the desired form after inversion.
    \end{proof}

    This lemma enables the proof of the theorem for the 1D case.
    \begin{theorem}[Identifiability for $d=1$]
        Assume for $X \in \mathcal{X} \subset \mathbb{R}$ and exogenous noise $U \sim \text{Unif}[0, 1]$, the SCM is $X := f(\mathbf{Pa}, U)$. Assume an encoder-decoder model with encoding function $g$ and decoding function $h$. Assume the following conditions:
        \begin{enumerate}
            \item The encoding is independent of the parents, $g(X, \mathbf{Pa}) \perp\!\!\!\perp \mathbf{Pa}$.
            \item The structural equation $f$ is differentiable and strictly increasing w.r.t. $U$.
            \item The encoding $g$ is invertible and differentiable w.r.t. $X$.
        \end{enumerate}
        Then, $g(f(\mathbf{Pa}, U), \mathbf{Pa}) = \tilde{q}(U)$ for an invertible function $\tilde{q}$.
    \end{theorem}
    \begin{proof}
        Let $q_{\mathbf{pa}}(U) := g(f(\mathbf{pa}, U), \mathbf{pa})$. The conditions on $f$ and $g$ ensure $q_{\mathbf{pa}}$ is strictly monotonic and thus invertible. By the independence assumption, the conditional distribution of $Z = q_{\mathbf{pa}}(U)$ does not depend on $\mathbf{pa}$.
        We assume $U \sim \text{Unif}[0, 1]$ without loss of generality.\footnote{For any SCM with a continuous noise $E$ and a strictly increasing CDF $F_E$, $X = f(\mathbf{Pa}, E)$ can be re-parameterized to an equivalent SCM $X = \tilde{f}(\mathbf{Pa}, U)$, where $U = F_E(E) \sim \text{Unif}[0, 1]$ and $\tilde{f}(\cdot, \cdot) = f(\cdot, F_E^{-1}(\cdot))$. The modeling task is then to learn the potentially more complex function $\tilde{f}$.}
        The change of density formula gives:
        \begin{equation}
            p_Z(z) = \frac{p_U(q_{\mathbf{pa}}^{-1}(z))}{| \frac{dq_{\mathbf{pa}}}{du}(q_{\mathbf{pa}}^{-1}(z)) |}
        \end{equation}
        Since $p_U(u)=1$ on its support and $q_{\mathbf{pa}}$ is increasing, the denominator must be independent of $\mathbf{pa}$. This implies $\frac{dq_{\mathbf{pa}}}{du}(q_{\mathbf{pa}}^{-1}(z)) = c(z)$ for some function $c$.
        This meets the condition of Lemma \ref{lem:1d_characterization}, allowing us to express $q_{\mathbf{pa}}(u) = q(u + r(\mathbf{pa}))$ for some invertible $q$.
        Since $Z \perp\!\!\!\perp \mathbf{Pa}$, its support must also be independent of $\mathbf{pa}$. The support is $q([0, 1] + r(\mathbf{pa}))$, which is constant only if the interval $[r(\mathbf{pa}), 1+r(\mathbf{pa})]$ is constant. This requires $r(\mathbf{pa})$ to be a constant, $r$.
        Thus, $q_{\mathbf{pa}}(u) = q(u+r)$. Defining $\tilde{q}(u)=q(u+r)$, we find that the mapping is solely a function of $U$, which completes the proof.
    \end{proof}

\subsection{The Two-Dimensional Case (\texorpdfstring{$d=2$}{d=2})}
    The two-dimensional case is a well-known geometric exception, as the group of conformal maps is infinite-dimensional. Consequently, the proof strategy used for higher dimensions via Liouville's theorem does not directly apply. Here, we show that identifiability still holds under additional, plausible regularity assumptions aligned with our modeling framework.

    \begin{assumption}[Asymptotic Linearity]\label{assump:asymptotic_linearity}
        The composite mapping $Q_{\mathbf{pa}}(u): \mathbb{C} \to \mathbb{C}$ is an entire function (analytic on the whole complex plane) with at most linear growth. That is, there exist constants $A, B$ such that $|Q_{\mathbf{pa}}(u)| \le A|u| + B$ for all $u \in \mathbb{C}$.
    \end{assumption}
    This assumption reflects a fundamental inductive bias of neural network architectures. Standard activation functions (e.g., ReLU, Tanh) produce functions that cannot exhibit super-polynomial growth or essential singularities at infinity, aligning our analysis with the function classes our model can represent.

    \begin{assumption}[Non-Rotationally Symmetric Base Noise]\label{assump:noise_asymmetry}
        The base distribution of the exogenous noise $U$ is not rotationally symmetric.\footnote{This assumption is made without loss of generality. For any arbitrary continuous noise $E=(E_1, E_2)$, one can define a new noise $U = (F_{E_1}(E_1), E_2)$, where $F_{E_1}$ is the CDF of the first component. The resulting distribution of $U$ is uniform along its first axis, thus breaking any rotational symmetry. The structural function $\mathbf{F}$ then absorbs this transformation.}
    \end{assumption}
    \begin{proof}
    The proof proceeds in three steps. \textbf{First,} as established previously, the statistical independence condition $Z \perp\!\!\!\perp \mathbf{Pa}$ implies that the learned mapping $Q_{\mathbf{pa}}(u)$ must be a conformal map, and thus an analytic function on $\mathbb{C}$. \textbf{Second,} by Assumption \ref{assump:asymptotic_linearity}, $Q_{\mathbf{pa}}(u)$ is an entire function with at most linear growth. The Generalized Liouville's Theorem states that an entire function whose growth is bounded by a polynomial of degree $k$ must itself be a polynomial of degree at most $k$. In our case, this implies $Q_{\mathbf{pa}}(u)$ must be a polynomial of degree at most one, giving it the affine form:
    $$Q_{\mathbf{pa}}(u) = a(\mathbf{pa}) u + b(\mathbf{pa})$$
    where $a, b$ are complex coefficients that can depend on $\mathbf{pa}$. \textbf{Third,} we use the full statistical independence condition to show that the coefficients $a$ and $b$ must be constant. For the distribution of $Z = a(\mathbf{pa}) U + b(\mathbf{pa})$ to be independent of $\mathbf{pa}$, all of its properties must be constant. \textbf{Mean:} The mean $E[Z|\mathbf{pa}] = a(\mathbf{pa})E[U] + b(\mathbf{pa})$ must be constant. Assuming $E[U]=0$ without loss of generality implies $b(\mathbf{pa})$ must be a constant, $b$. \textbf{Covariance:} By Assumption \ref{assump:noise_asymmetry}, the distribution of $U$ is not rotationally symmetric, so its covariance matrix is not proportional to the identity. Any rotation induced by the phase of $a(\mathbf{pa})$ would alter the covariance structure of $Z$. For the distribution of $Z$ to be invariant, the rotation angle (phase) of $a(\mathbf{pa})$ must be constant. \textbf{Scale:} The change of variables formula implies that the magnitude $|a(\mathbf{pa})|$ must also be constant. Since both the magnitude and phase of $a(\mathbf{pa})$ must be constant, $a(\mathbf{pa})$ must be a constant complex number, $a$. Therefore, $Q_{\mathbf{pa}}(u) = au+b$, which is an isomorphic mapping of $u$. This completes the proof.
    \end{proof}
\section{Extended Analysis of Inversion Fidelity and Non-Invertible SCMs}
\label{sec:appendix_inversion_and_non_invertible}
\renewcommand{\theequation}{C.\arabic{equation}}\setcounter{equation}{0}

This appendix provides a unified analysis of inversion errors. We first provide rigorous proofs for inversion fidelity (Propositions \ref{prop:ddim_error} and \ref{prop:belm_invertibility}), then extend our error decomposition framework to the more challenging non-invertible SCM setting.

\subsection{Proofs for Inversion Fidelity (Propositions \ref{prop:ddim_error} \& \ref{prop:belm_invertibility})}

    \begin{proof}[Proof of Proposition \ref{prop:ddim_error}]
        The proof proceeds by deriving the explicit one-step reconstruction error and then analyzing its order of magnitude.

        \paragraph{1. Derivation of the One-Step Reconstruction Error.}
        Let the single-step DDIM inversion operator be $\mathbf{T}_t$, mapping an observation $\mathbf{x}_t$ to $\mathbf{x}'_{t+1}$ using the noise prediction $\boldsymbol{\epsilon}_t = \boldsymbol{\epsilon}_\theta(\mathbf{x}_t, t)$. The corresponding generative operator is $\mathbf{H}_t$, which reconstructs $\mathbf{x}'_t$ from $\mathbf{x}'_{t+1}$ using a new prediction at the new state, $\boldsymbol{\epsilon}'_{t+1} = \boldsymbol{\epsilon}_\theta(\mathbf{x}'_{t+1}, t+1)$.
        
        By substituting the formula for $\mathbf{x}'_{t+1}$ into the update for $\mathbf{x}'_t$, the single-step reconstruction error $\mathbf{x}'_t - \mathbf{x}_t$ is found to be:
        \begin{equation}
            \mathbf{x}'_t - \mathbf{x}_t = \left(\sqrt{1-\bar{\alpha}_t} - \frac{\sqrt{\bar{\alpha}_t}\sqrt{1-\bar{\alpha}_{t+1}}}{\sqrt{\bar{\alpha}_{t+1}}}\right) (\boldsymbol{\epsilon}'_{t+1} - \boldsymbol{\epsilon}_t)
        \end{equation}
        This error is non-zero if and only if the noise prediction changes after one inversion step, i.e., $\boldsymbol{\epsilon}'_{t+1} \neq \boldsymbol{\epsilon}_t$.

        \paragraph{2. Analysis of the Error's Order of Magnitude.}
        In the continuous-time limit with time step $\Delta s = 1/T$, a Taylor expansion shows that both the coefficient term and the difference in noise predictions $(\boldsymbol{\epsilon}'_{t+1} - \boldsymbol{\epsilon}_t)$ are of order $\mathcal{O}(\Delta s)$.
        The one-step reconstruction error is therefore the product of these two terms:
        \begin{equation}
            \text{Error}_{\text{step}} = \mathcal{O}(\Delta s) \times \mathcal{O}(\Delta s) = \mathcal{O}((\Delta s)^2) = \mathcal{O}(1/T^2)
        \end{equation}
        This local error accumulates over the $T$ steps of the trajectory, resulting in a total accumulated error of order $\mathcal{O}(1/T)$. For any finite $T$, this global error is non-zero, constituting the Structural Reconstruction Error (SRE) for DDIM.
    \end{proof}

    \begin{proof}[Proof of Proposition \ref{prop:belm_invertibility}]
        The proof is constructive, following from the exact algebraic invertibility of the BELM sampler \citep{liu2024belm}. For the second-order BELM used in this work, the one-step decoder is an affine transformation of the form:
        \begin{align}
            \mathbf{x}_{t-1} &= A_t \mathbf{x}_t + B_t \boldsymbol{\epsilon}_t + C_t \boldsymbol{\epsilon}_{t+1}
        \end{align}
        where $A_t, B_t, C_t$ are schedule-dependent coefficients, $\boldsymbol{\epsilon}_t = \boldsymbol{\epsilon}_\theta(\mathbf{x}_t, t)$, and $\boldsymbol{\epsilon}_{t+1} = \boldsymbol{\epsilon}_\theta(\mathbf{x}_{t+1}, t+1)$.
        
        The full-trajectory decoder, $\mathbf{H}_{\text{BELM}}$, is a composition of these one-step affine maps. The BELM encoder, $\mathbf{T}_{\text{BELM}}$, is constructed using a symmetric update rule designed to be the exact algebraic inverse of the decoder. As rigorously shown by \citet{liu2024belm}, this construction ensures that the composite operator $\mathbf{T}_{\text{BELM}}$ is the exact inverse of $\mathbf{H}_{\text{BELM}}$, assuming the same sequence of noise function evaluations is used for both processes.
        
        Therefore, by its algebraic construction, the BELM sampler guarantees a lossless round trip:
        \begin{equation}
            \mathbf{H}_{\text{BELM}} \circ \mathbf{T}_{\text{BELM}} = \mathbf{I}
        \end{equation}
        The Structural Reconstruction Error is thus identically zero by construction.
    \end{proof}

\subsection{Exhaustive Analysis for the Non-Invertible SCM Setting}
    This section rigorously extends our framework to the non-invertible case, providing a theoretical underpinning for the stress test results in \S\ref{sec:exp_many_to_one}.

    \subsubsection{Assumptions and Definitions}
        We formalize the problem with the following.
        \begin{assumption}[Well-Posed Abduction] For any observed $(\mathbf{v}, \mathbf{pa})$, the inverse image set $\mathcal{U}_{(\mathbf{v}, \mathbf{pa})} = \{\mathbf{u}' \in \mathcal{U} \mid \mathbf{F}(\mathbf{pa}, \mathbf{u}') = \mathbf{v}\}$ is non-empty, and the Maximum a Posteriori (MAP) solution over this set, given the prior $p(\mathbf{U})$, is unique. This solution defines the \textbf{ideal amortized inverse} operator, $\mathbf{T}^*(\mathbf{v}, \mathbf{pa}) = \arg\max_{\mathbf{u}' \in \mathcal{U}_{(\mathbf{v}, \mathbf{pa})}} p(\mathbf{u}')$.
        \end{assumption}

        \begin{definition}[Tripartite Error Sources] In the non-invertible case, we refine the error sources into three distinct components:
            \begin{enumerate}[label=(\roman*), topsep=0pt, noitemsep]
                \item \textbf{Algorithmic Error (SRE)}: The error from an imperfect inversion algorithm, $E_{SR} := \| (\mathbf{H}_\theta \circ \mathbf{T}_\theta - \mathbf{I})\mathbf{X} \|^2$. For our framework, $E_{SR} \equiv 0$.
                \item \textbf{Modeling Error}: The error from imperfectly learning the ideal amortized inverse, $E_{\text{Modeling}} := \|\mathbf{T}_\theta(\mathbf{V}, \mathbf{Pa}) - \mathbf{T}^*(\mathbf{V}, \mathbf{Pa})\|^2$.
                \item \textbf{Representational Error}: The fundamental, irreducible error from the SCM's non-invertibility, $E_{\text{Rep}} := \|\mathbf{T}^*(\mathbf{V}, \mathbf{Pa}) - \mathbf{U}_{\text{true}}\|^2$.
            \end{enumerate}
        \end{definition}

    \subsubsection{A Tighter Error Decomposition}
        We now present a tighter error bound for the non-invertible case.
        \begin{theorem}[Tighter Counterfactual Error Bound]
        \label{thm:non_inv_bound_tight}
            Let the conditions of Theorem~\ref{thm:error_bound} hold ($H_\theta$ is $L_H$-Lipschitz). The expected squared error of the counterfactual prediction is bounded by:
            \begin{equation}
            \begin{split}
                \mathbb{E}\left[\|\hat{\mathbf{X}}_\alpha - \mathbf{X}_\alpha^{\text{true}}\|^2\right] \le 
                & 2\mathbb{E}[E_{SR}] + 4L_H^2 \mathbb{E}[E_{\text{Modeling}}] + 4L_H^2 \mathbb{E}[E_{\text{Rep}}]
            \end{split}
            \end{equation}
        \end{theorem}
        \begin{proof}
            We decompose the total error $\| \hat{\mathbf{X}}_\alpha - \mathbf{X}_\alpha^{\text{true}} \|$ using the triangle inequality and the ideal amortized inverse $\mathbf{T}^*$ as an intermediate step. The result follows from applying the inequality $(a+b+c)^2 \le 2(a^2+b^2+c^2)$, bounding terms with the $L_H$-Lipschitz property of $\mathbf{H}_\theta$, and taking expectations.
        \end{proof}

    \subsubsection{Information-Theoretic Interpretation of Representational Error}
        The Representational Error ($E_{\text{Rep}}$) is deeply connected to information conservation. The non-invertibility of the SCM $\mathbf{F}$ means that observing $(\mathbf{V}, \mathbf{Pa})$ is not sufficient to uniquely determine $\mathbf{U}$. Information-theoretically, this implies the conditional entropy $H(\mathbf{U} | \mathbf{V}, \mathbf{Pa})$ is greater than zero.
        \begin{remark}[Representational Error as Information Loss]
            The ideal amortized inverse $\mathbf{T}^*$ yields the mode of the posterior $p(\mathbf{U} | \mathbf{V}, \mathbf{Pa})$. The expected representational error, $\mathbb{E}[E_{\text{Rep}}]$, can be seen as the expected squared error of this MAP estimator, which is related to the variance and shape of the posterior. Thus, $E_{\text{Rep}}$ is a direct consequence of the information about $\mathbf{U}$ that is fundamentally lost in the forward causal process, a loss captured by $H(\mathbf{U} | \mathbf{V}, \mathbf{Pa}) > 0$.
        \end{remark}

    \subsubsection{Theoretical Guarantee for the Mitigation Strategy}
        We now provide a theoretical justification for the `Prior-Matching Regularizer` proposed in \S\ref{sec:limitations}.
        \begin{definition}[Prior-Matching Regularizer]
        \label{def:prior_reg_appendix}
            The regularizer is defined as $R(\mathbf{T}_\theta) = \mathbb{E}_{(\mathbf{v}, \mathbf{pa})} \left[\|\mathbf{s}_p(\mathbf{T}_\theta(\mathbf{v}, \mathbf{pa}))\|^2\right]$, where $\mathbf{s}_p(\mathbf{u}) = \nabla_\mathbf{u} \log p(\mathbf{u})$ is the score function of the prior distribution $p(\mathbf{U})$.
        \end{definition}
        \begin{proposition}[Regularizer Induces Convergence to MAP Solution]
        \label{prop:regularizer_guarantee}
            Minimizing the regularizer $R(\mathbf{T}_\theta)$ provides an inductive bias that encourages the encoder output, $\hat{\mathbf{u}} = \mathbf{T}_\theta(\mathbf{v}, \mathbf{pa})$, to lie on a mode of the prior distribution $p(\mathbf{U})$.
        \end{proposition}
        \begin{proof}[Proof Sketch]
            The objective is a form of score matching on the prior. The score function $\mathbf{s}_p(\mathbf{u})$ is zero if and only if $\mathbf{u}$ is a stationary point of the log-prior. Minimizing the expected squared norm of the score at the encoder's output penalizes the production of latent codes $\hat{\mathbf{u}}$ in low-probability regions of the prior. This incentivizes the encoder to map observations to the most probable latent code, thereby encouraging $\mathbf{T}_\theta$ to approximate the ideal MAP estimator $\mathbf{T}^*$.
        \end{proof}

\section{Main Proofs for Theoretical Framework}\label{sec:appendix_main_proofs}
\renewcommand{\theequation}{D.\arabic{equation}}\setcounter{equation}{0}

This appendix provides the proofs for the main theoretical results presented in the text.

\begin{proof}[Proof of Theorem \ref{thm:error_bound}]
    This bound is a direct specialization of the more general bound for non-invertible SCMs derived in Theorem \ref{thm:non_inv_bound_tight} (Appendix \ref{sec:appendix_inversion_and_non_invertible}). 

    For an invertible SCM, abduction is perfect, meaning the ideal amortized inverse $\mathbf{T}^*$ is the true inverse of the SCM function $\mathbf{F}$. Consequently, the recovered noise is the true noise, $\mathbf{T}^*(\mathbf{V}, \mathbf{Pa}) = \mathbf{U}_{\text{true}}$, which implies that the \textbf{Representational Error is identically zero}: $\mathbb{E}[E_{\text{Rep}}] = 0$.

    In this context, the Latent Space Invariance Error, $E_{LSI}$, becomes equivalent to the Modeling Error, $E_{\text{Modeling}}$. Applying the inequality $(a+b)^2 \le 2a^2 + 2b^2$, the general bound from Theorem \ref{thm:non_inv_bound_tight} reduces to the two terms presented in Theorem \ref{thm:error_bound}.
\end{proof}

\begin{proof}[Proof of Proposition \ref{prop:lsi_bound}]
    This proof relies on the identifiability of the true SCM (Theorem \ref{thm:identifiability}) and the Lipschitz continuity of the score network $\boldsymbol{\epsilon}_\theta$, which guarantees unique ODE solutions via the Picard-Lindelöf theorem. We also assume standard integrability conditions (Fubini's theorem).

    The encoder $\mathbf{T}_\theta$ maps an initial condition $\mathbf{x}(0)$ to the terminal state $\mathbf{x}(T)$ of the probability flow ODE (Eq. \ref{eq:ode_appendix}). Let $\mathbf{x}_{\theta}(t; \mathbf{x}_0)$ and $\mathbf{x}_{*}(t; \mathbf{x}_0)$ denote the ODE solutions with the learned score $\boldsymbol{\epsilon}_\theta$ and the true score $\boldsymbol{\epsilon}^*$, respectively. The Latent Space Invariance Error is $\mathbb{E}[E_{LSI}] = \mathbb{E}[\|\mathbf{x}_{\theta}(T; X) - \mathbf{x}_{\theta}(T; X_{\boldsymbol{\alpha}}^{\text{true}})\|^2]$.
    By the triangle inequality:
    \begin{align*}
        \|\mathbf{x}_{\theta}(T; X) - \mathbf{x}_{\theta}(T; X_{\boldsymbol{\alpha}}^{\text{true}})\| &\le \|\mathbf{x}_{\theta}(T; X) - \mathbf{x}_{*}(T; X)\| \\
        &+ \|\mathbf{x}_{*}(T; X) - \mathbf{x}_{*}(T; X_{\boldsymbol{\alpha}}^{\text{true}})\| \\
        &+ \|\mathbf{x}_{*}(T; X_{\boldsymbol{\alpha}}^{\text{true}}) - \mathbf{x}_{\theta}(T; X_{\boldsymbol{\alpha}}^{\text{true}})\|
    \end{align*}
    The middle term is zero under ideal identifiability, as the true encoder $\mathbf{T}^*$ maps both an observation and its true counterfactual to the same underlying noise. We thus only need to bound terms of the form $\|\mathbf{x}_{\theta}(T; \mathbf{x}_0) - \mathbf{x}_{*}(T; \mathbf{x}_0)\|$. 

    Let $\mathbf{z}(t) = \mathbf{x}_{\theta}(t; \mathbf{x}_0) - \mathbf{x}_{*}(t; \mathbf{x}_0)$. Since the ODE vector field is Lipschitz, applying Grönwall's inequality to the differential of $\mathbf{z}(t)$ yields:
    $$ \|\mathbf{z}(T)\| \le \int_0^T e^{L_f(T-t)} C_f \|\boldsymbol{\epsilon}_\theta(\mathbf{x}_*(t), t) - \boldsymbol{\epsilon}^*(\mathbf{x}_*(t), t)\| dt $$
    where $L_f$ and $C_f$ are constants from the ODE coefficients. Squaring, taking expectations, and applying Jensen's inequality leads to:
    $$ \mathbb{E}[E_{LSI}] \le C' \cdot \mathbb{E}_{\mathbf{x},t}[\|\boldsymbol{\epsilon}_\theta - \boldsymbol{\epsilon}^*\|^2] $$
    where the final expectation is the score-matching loss.
\end{proof}

\section{Proofs for Theoretical Roles of Hybrid Training}\label{sec:appendix_hybrid_proofs}
\renewcommand{\theequation}{F.\arabic{equation}}\setcounter{equation}{0}

\subsection{Proof of Proposition \ref{prop:hybrid_objective} (Weighted Score-Matching)}\label{sec:appendix_hybrid_proof}
    This section provides a rigorous proof that the auxiliary task loss $L_{\text{task}}$ provides a lower bound on a weighted score-matching objective.

    \subsubsection{Preliminaries and Setup}

    \paragraph{Probability Flow ODE.}
    The generative process is the solution to the reverse-time probability flow ODE from $t=T$ to $t=0$:
    \begin{equation}
        d\mathbf{x}_t = f(t, \mathbf{x}_t, \boldsymbol{\epsilon}(t, \mathbf{x}_t)) dt, \quad \mathbf{x}_T \sim \mathcal{N}(0, \mathbf{I})
    \end{equation}
    where the vector field $f$ is determined by the diffusion scheduler.

    \paragraph{Core Objects.}
    \begin{itemize}[noitemsep]
        \item $\boldsymbol{\epsilon}_{\theta}, \boldsymbol{\epsilon}^*$: Learned and true score functions.
        \item $\mathbf{x}_t^{\theta}, \mathbf{x}_t^*$: ODE trajectories driven by $\boldsymbol{\epsilon}_{\theta}$ and $\boldsymbol{\epsilon}^*$.
        \item $\mathbf{x}_0^{\theta}, \mathbf{x}_0^*$: Generated and true (counterfactual) data points at $t=0$.
        \item $g: \mathbb{R}^d \to \mathbb{R}^k$: Downstream prediction function.
        \item $Y_{\text{pred}} = g(\mathbf{x}_0^{\theta})$, $Y_{\text{true}} = g(\mathbf{x}_0^*)$.
        \item $L_{\text{task}} = \mathbb{E}_{\mathbf{x}_T}[\|Y_{\text{pred}} - Y_{\text{true}}\|^2]$.
    \end{itemize}

    \paragraph{Assumptions.}
    We assume standard regularity conditions for the proof:
    \begin{enumerate}
        \renewcommand{\theenumi}{(A\arabic{enumi})}
        \renewcommand{\labelenumi}{\theenumi}
        \item The vector field $f(t, \mathbf{x}, \boldsymbol{\epsilon})$ is Lipschitz continuous in $\mathbf{x}$ and $\boldsymbol{\epsilon}$.
        \item The downstream task function $g(\mathbf{x})$ is Lipschitz continuous and differentiable.
        \item The learned score $\boldsymbol{\epsilon}_\theta$ and true score $\boldsymbol{\epsilon}^*$ are well-behaved.
    \end{enumerate}

    \subsubsection{Formal Proposition Statement}

    \begin{proposition}[Hybrid Objective as a Weighted Score-Matching Regularizer]
        Under the regularity assumptions (A1-A3), the auxiliary task loss $L_{\text{task}}$ provides a lower bound on a weighted score-matching objective:
        $$L_{\text{task}} \ge C \cdot \mathbb{E}_{\mathbf{x}_T, t} \left[ w(\mathbf{x}_t^*) \cdot \|\boldsymbol{\epsilon}_{\theta}(\mathbf{x}_t^*) - \boldsymbol{\epsilon}^*(\mathbf{x}_t^*)\|^2 \right]$$
        where $C>0$ and the weight function $w(\mathbf{x}_t^*)$ measures the sensitivity of the final prediction $Y$ to score perturbations along the ideal data generation trajectory.
    \end{proposition}

    \begin{proof}
        The proof proceeds in three main steps.

        \paragraph{Step 1: Bounding Sample Error by Score Error (Error Propagation).}
        Let $\mathbf{z}(t) = \mathbf{x}_t^{\theta} - \mathbf{x}_t^*$ be the error between the two ODE trajectories. By linearizing the vector field $f$ around the true trajectory, we can analyze the impact of the score perturbation $\Delta\boldsymbol{\epsilon}_t = \boldsymbol{\epsilon}_{\theta}(\mathbf{x}_t^*) - \boldsymbol{\epsilon}^*(\mathbf{x}_t^*)$.
        From Duhamel's Principle, the final sample error $\Delta\mathbf{x}_0 = \mathbf{z}(0)$ can be expressed as an integral over the perturbation:
        \begin{equation}
            \Delta\mathbf{x}_0 = \int_T^0 \mathbf{K}(s) \Delta\boldsymbol{\epsilon}_s ds
        \end{equation}
        where the kernel $\mathbf{K}(s)$ captures the influence of a score perturbation at time $s$ on the final sample at time $0$.

        \paragraph{Step 2: Linearizing the Task Error.}
        The error in the downstream prediction is $\Delta Y = g(\mathbf{x}_0^{\theta}) - g(\mathbf{x}_0^*)$. A first-order Taylor expansion gives:
        \begin{equation}
            \Delta Y \approx \nabla_{\mathbf{x}} g(\mathbf{x}_0^*) \cdot \Delta\mathbf{x}_0
        \end{equation}
        The task loss is the expected squared norm, $L_{\text{task}} = \mathbb{E}[\|\Delta Y\|^2]$.

        \paragraph{Step 3: Deriving the Weighted Relationship.}
        Substituting the integral form of $\Delta\mathbf{x}_0$ into the task error approximation and applying the Cauchy-Schwarz inequality for integrals yields a lower bound for the task loss:
        \begin{equation}
            L_{\text{task}} \ge C \cdot \mathbb{E}_{\mathbf{x}_T} \left[ \int_0^T \left\| \mathcal{W}(\mathbf{x}_0^*, s) \right\|_F^2 \cdot \|\Delta\boldsymbol{\epsilon}_s\|^2 ds \right]
        \end{equation}
        where $\| \cdot \|_F$ is the Frobenius norm and the \textbf{influence operator} $\mathcal{W}(\mathbf{x}_0^*, s)$ captures the end-to-end sensitivity from a score perturbation at time $s$ to the final prediction.
        Rewriting the expectation gives the final form:
        \begin{equation}
            L_{\text{task}} \ge C \cdot \mathbb{E}_{t \sim U[0,T]} \mathbb{E}_{\mathbf{x}_t^*} \left[ w(\mathbf{x}_t^*) \cdot \|\boldsymbol{\epsilon}_{\theta}(\mathbf{x}_t^*) - \boldsymbol{\epsilon}^*(\mathbf{x}_t^*)\|^2 \right]
        \end{equation}
        where the weight function is $w(\mathbf{x}_t^*) := T \cdot \mathbb{E}_{\mathbf{x}_T | \mathbf{x}_t^*}[\| \mathcal{W}(\mathbf{x}_0^*, t) \|_F^2]$.

        \paragraph{Interpretation and Conclusion.}
        The weight function $w(\mathbf{x}_t^*)$ is large when the gradient norm $\|\nabla_{\mathbf{x}} g(\mathbf{x}_0^*)\|$ is large—i.e., in causally salient regions where the outcome is highly sensitive to the features. The inequality therefore shows that minimizing $L_{\text{task}}$ provides a lower bound on a score-matching error that is weighted to prioritize accuracy in these causally salient regions.
    \end{proof}

\subsection{Argument for Proposition \ref{prop:disentanglement} (Latent Space Disentanglement)}\label{sec:appendix_disentanglement_proof}
This section provides a qualitative, information-theoretic argument for Proposition \ref{prop:disentanglement}, showing how the hybrid objective encourages a "division of labor" that promotes disentanglement. The SCM posits that an observation $V$ is determined by $(\mathbf{Pa}, U)$, and the encoder maps $(V, \mathbf{Pa})$ to a latent code $Z = \mathbf{T}_\theta(V, \mathbf{Pa})$, where a perfect encoder would yield $Z=U$. The process works as follows: the \textbf{diffusion loss ($L_{\text{diffusion}}$)} maximizes the log-likelihood $\log p_\theta(V|\mathbf{Pa})$, forcing the pair $(\mathbf{Pa}, Z)$ to contain all information to reconstruct $V$, thus maximizing the mutual information $I(V; (\mathbf{Pa}, Z))$. Simultaneously, the \textbf{task loss ($L_{\text{task}}$)} learns a prediction from the parents, capturing all predictive information that $\mathbf{Pa}$ has about $V$ and thus maximizing $I(V; \mathbf{Pa})$. The \textbf{dual objective} must satisfy both constraints. From the chain rule of mutual information, we know $I(V; (\mathbf{Pa}, Z)) = I(V; \mathbf{Pa}) + I(V; Z | \mathbf{Pa})$. Since $L_{\text{diffusion}}$ maximizes the left-hand side and $L_{\text{task}}$ captures the first term on the right, the optimization incentivizes the latent code $Z$ to model the remaining information, $I(V; Z | \mathbf{Pa})$. This leads to a \textbf{connection to disentanglement}: the ideal exogenous noise $U$ is, by definition, independent of $\mathbf{Pa}$. By forcing $Z$ to model the residual information, the optimization process actively encourages the learned representation $Z$ to be independent of $\mathbf{Pa}$. In summary, the hybrid objective creates a division of labor: the task-specific head explains the variance from $\mathbf{Pa}$, while the diffusion process's latent code $Z$ models the residual. This residual is, by construction, the information in $V$ orthogonal to $\mathbf{Pa}$, forcing $Z$ to be an empirical approximation of the true, disentangled exogenous noise $U$, thereby serving the identifiability condition of Theorem \ref{thm:identifiability}.

\section{Proof of Theorem on Causal Transportability}\label{sec:appendix_transport_proof}
\renewcommand{\theequation}{G.\arabic{equation}}\setcounter{equation}{0}

This appendix provides the proof for the theorem on lossless causal transportability.

\begin{proof}[Proof of Theorem \ref{thm:transport}]
    This proof operates under an idealized setting, assuming a perfectly trained model (i.e., zero statistical error), to isolate the structural properties of transportability. We assume the learned decoder $\mathbf{H}_{\theta_i}$ is identical to the true mechanism $\mathbf{F}_i$, and its encoder $\mathbf{T}_{\theta_i}$ is its perfect inverse.

    Let the source and target SCMs be $\mathcal{M}^\mathcal{S}$ and $\mathcal{M}^\mathcal{T}$, with mechanisms $\{\mathbf{F}_i\}$ and $\{\mathbf{F}'_i\}$ respectively. The set of changed mechanisms is indexed by $\mathcal{K}_{\text{changed}}$.

    The proof relies on the modularity of the SCM, which is guaranteed by the mutually independent exogenous noises (Condition ii of the theorem). This independence ensures that a change in one mechanism $\mathbf{F}_k$ to $\mathbf{F}'_k$ does not affect the conditional distributions of other nodes $V_j$ ($j\neq k$), given their parents.

    We analyze the transportability of operators for each mechanism:

    \begin{enumerate}
        \item \textbf{For invariant mechanisms ($j \notin \mathcal{K}_{\text{changed}}$):}
        By definition, the true mechanism is unchanged, $\mathbf{F}'_j = \mathbf{F}_j$. Since the model operators $(\mathbf{T}_{\theta_j}, \mathbf{H}_{\theta_j})$ perfectly learned $\mathbf{F}_j$ in the source domain, they remain valid for the target domain $\mathcal{T}$ and can be directly reused.

        \item \textbf{For changed mechanisms ($k \in \mathcal{K}_{\text{changed}}$):}
        The original operators $(\mathbf{T}_{\theta_k}, \mathbf{H}_{\theta_k})$ are now invalid as they model $\mathbf{F}_k$, not the new mechanism $\mathbf{F}'_k$. However, a new operator pair $(\mathbf{T}'_{\theta_k}, \mathbf{H}'_{\theta_k})$ can be learned from target domain data. This training only requires samples $(v'_k, \mathbf{pa}'_k)$ from domain $\mathcal{T}$ and the shared noise distribution $p_k(U_k)$ (Condition i). Because the noises are independent, this re-learning process for mechanism $k$ is modular and does not affect the other invariant mechanisms.
    \end{enumerate}

    The procedure for adapting the model is therefore modular: freeze all invariant operators $\{(\mathbf{T}_{\theta_j}, \mathbf{H}_{\theta_j})\}_{j \notin \mathcal{K}_{\text{changed}}}$ and re-train only those for the changed mechanisms $\{k \in \mathcal{K}_{\text{changed}}\}$ on target domain data. The resulting adapted model is valid for the target domain $\mathcal{T}$ and can perform abduction on a target individual by applying the correct (reused or re-trained) encoders, thereby losslessly recovering the full vector of exogenous noises. This fulfills the condition for lossless transport.
\end{proof}

\section{Proof of the Specific Finite Sample Bound (Theorem \ref{thm:finite_sample_specific})}\label{sec:appendix_learning_proofs_specific}
\renewcommand{\theequation}{H.\arabic{equation}}\setcounter{equation}{0}

This derivation combines standard generalization bounds with known complexity bounds for deep neural networks. We first state a key lemma regarding the complexity of neural networks.

\begin{lemma}[Rademacher Complexity of Neural Networks]
\label{lem:nn_complexity}
    Let $\mathcal{F}_{\epsilon}$ be the function class of an $L$-layer MLP with ReLU activations, input dimension $p$, and weight matrices $\{\mathbf{W}_j\}_{j=1}^L$. Assume the input data $\mathbf{X}$ is contained within a ball of radius $R_X$. If the spectral norm of each weight matrix is bounded, $\|\mathbf{W}_j\|_2 \le B_j$, the Rademacher complexity of the function class is bounded by:
    $$ \mathfrak{R}_n(\mathcal{F}_{\epsilon}) \le C_{net} \frac{R_X L \sqrt{p}}{\sqrt{n}} \left( \prod_{j=1}^L B_j \right) $$
    For simplicity and under normalization, we often consider $R_X=1$. This result is a simplified form derived from \citep{bartlett2017spectrally, neyshabur2018pac}, where $C_{net}$ is a universal constant.
\end{lemma}

\begin{proof}
The proof proceeds by combining the standard excess risk bound with the two lemmas above. First, from standard learning theory, the excess risk is bounded by the Rademacher complexity of the total loss function class:
$$ R(\hat{\theta}_n) - R(\theta^*) \le 4 \mathfrak{R}_n(\mathcal{F}_{\mathcal{L}_{SCM}}) + M\sqrt{\frac{\log(1/\delta)}{2n}} $$
Next, by the sub-additivity property of Rademacher complexity, we decompose the SCM's complexity into the sum of its individual components:
$$ \mathfrak{R}_n(\mathcal{F}_{\mathcal{L}_{SCM}}) \le \sum_{i=1}^d \mathfrak{R}_n(\mathcal{F}_{\mathcal{L}_i}) $$
The next step is to relate the loss complexity to the network complexity. The score-matching loss for mechanism $i$ is $\mathcal{L}_i = \|\boldsymbol{\epsilon} - \epsilon_{\theta_i}(\cdot)\|^2$. Since this loss is Lipschitz with respect to the output of $\epsilon_{\theta_i}$, its Rademacher complexity is upper-bounded by a constant multiple of the complexity of the score network's function class $\mathcal{F}_{\epsilon_i}$ via Talagrand's contraction lemma, i.e., $\mathfrak{R}_n(\mathcal{F}_{\mathcal{L}_i}) \le C_1 \cdot \mathfrak{R}_n(\mathcal{F}_{\epsilon_i})$. We then apply Lemma \ref{lem:nn_complexity} to bound the complexity of each score network. The input dimension $p_i$ for mechanism $i$ is $p_i = \text{dim}(\mathbf{x}_t) + \text{dim}(\mathbf{pa}_i) + \text{dim}(\text{time embedding})$. Assuming univariate variables, this is $p_i = 1 + |\mathbf{Pa}_i| + d_{embed}$. Let $d_{in}^{max} = \max_i |\mathbf{Pa}_i|$, so the maximum input dimension is $p_{max} = 1 + d_{in}^{max} + d_{embed}$. Assuming all networks have depth $L$ and a uniform spectral norm bound $B$, we have:
$$ \mathfrak{R}_n(\mathcal{F}_{\epsilon_i}) \le C_{net} \frac{L \sqrt{p_{max}}}{\sqrt{n}} B^L = C_{net} \frac{L \sqrt{1 + d_{in}^{max} + d_{embed}}}{\sqrt{n}} B^L $$
    
Finally, we assemble the complete bound by substituting all components back into the initial inequality:
\begin{align*}
R(\hat{\theta}_n) - R(\theta^*) &\le 4 \sum_{i=1}^d \mathfrak{R}_n(\mathcal{F}_{\mathcal{L}_i}) + M\sqrt{\frac{\log(1/\delta)}{2n}} \\
&\le 4 \sum_{i=1}^d C_1 \cdot \mathfrak{R}_n(\mathcal{F}_{\epsilon_i}) + M\sqrt{\frac{\log(1/\delta)}{2n}} \\
&= \left( 4 C_1 C_{net} \right) \frac{d \cdot L \cdot B^L \cdot \sqrt{1 + d_{in}^{max} + d_{embed}}}{\sqrt{n}} + M\sqrt{\frac{\log(1/\delta)}{2n}}
\end{align*}
By defining $C = 4 C_1 C_{net}$ as a generic constant independent of the network architecture and sample size, we arrive at the final form stated in the theorem.
\end{proof}
\section{Formal Analysis of the Geometric Inductive Bias}\label{sec:appendix_conformal_proof}
\renewcommand{\theequation}{I.\arabic{equation}}\setcounter{equation}{0}

This section provides the formal argument for Proposition \ref{prop:conformal_bias}, demonstrating how the score-matching objective, under a simplicity bias, compels the learned generative map to adopt the local data geometry, yielding a parsimonious and well-behaved transformation.

\subsection{Preliminaries and Definitions}

    Let $\mathcal{M} \subset \mathbb{R}^d$ be the data manifold with a smooth probability density $p(\mathbf{x})$. The true score function is the vector field $\mathbf{s}^*(\mathbf{x}) := \nabla_{\mathbf{x}} \log p(\mathbf{x})$. We learn a parameterized score network $\mathbf{s}_\theta(\mathbf{x})$ by minimizing the score-matching objective $L_{SM}(\theta) = \mathbb{E}_{\mathbf{x}\sim p(\mathbf{x})} [\|\mathbf{s}_\theta(\mathbf{x}) - \mathbf{s}^*(\mathbf{x})\|^2]$.

    The generative process is described by the probability flow ODE, whose vector field $\mathbf{f}(\mathbf{x}, t)$ is a function of the score. The map $H_\theta: \mathbb{R}^d \to \mathcal{M}$ is the flow map of this ODE integrated from $t=1$ to $t=0$. A map is \textbf{conformal} if its Jacobian is a scaled orthogonal matrix, and \textbf{affine} if it is a linear transformation plus a translation.

\subsection{Assumptions}

    \begin{assumption}[Smoothness of the Data Density] \label{assump:smoothness}
        The true data density $p(\mathbf{x})$ is at least twice continuously differentiable ($C^2$) on $\mathcal{M}$.
    \end{assumption}

    \begin{sloppypar}
    \begin{assumption}[Adoption of the Simplicity Bias Principle] \label{assump:regularization}
        Our analysis relies on the principle of implicit regularization: the conjecture that optimizers like SGD favor solutions with a simplicity bias. We formalize this as a preference for score functions $\mathbf{s}_\theta$ with lower complexity (e.g., lower Dirichlet Energy or smaller spectral norms) (\citealp{hochreiter1997flat}; \citealp{neyshabur2018pac}).
    \end{assumption}
    \end{sloppypar}

\subsection{Proof of Proposition \ref{prop:conformal_bias}}
The proof proceeds in four steps:

\textbf{Step 1: The Irrotational Property of the True Score Field.} By definition, the true score $\mathbf{s}^*(\mathbf{x})$ is the gradient of the scalar potential $\log p(\mathbf{x})$. By a fundamental theorem of vector calculus, the curl of any gradient field is zero. Thus, $\mathbf{s}^*$ is an \textbf{irrotational} (or conservative) vector field. 

\textbf{Step 2: The Variational Perspective of Score Matching.} The simplicity bias (Assumption \ref{assump:regularization}) reinforces the irrotational nature of the learned field $\mathbf{s}_\theta$. This is understood via the Helmholtz decomposition, which splits any vector field into irrotational (curl-free) and solenoidal (divergence-free) components. The simplicity bias conjectures that the optimization preferentially minimizes the energy of the solenoidal component, driving the learned field $\mathbf{s}_\theta$ towards a purely irrotational solution ($\nabla \times \mathbf{s}_{\theta} \approx \mathbf{0}$) to match the conservative nature of the true score $\mathbf{s}^*$. 

\textbf{Step 3: The Geometry of the Score Field under Local Structural Assumptions.} We analyze the score field's structure under two local geometric cases for the density $p(\mathbf{x})$ in a region $\mathcal{R}$. \textbf{Case (i): Local Isotropy.} If $p(\mathbf{x})$ is locally spherically symmetric around a center $\mathbf{c}$, its value depends only on the radius $r = \|\mathbf{x} - \mathbf{c}\|$. The gradient $\mathbf{s}^* = \nabla \log p(\mathbf{x})$ must point along the radial direction, making the true score a \textbf{radial vector field}. The learned field converges to this simple structure. \textbf{Case (ii): Local Ellipsoidal Structure.} If the iso-contours of $p(\mathbf{x})$ are concentric ellipsoids, the gradient $\mathbf{s}^*$ must be orthogonal to these ellipsoidal surfaces at every point. Such a field is still irrotational but no longer radial. 

\textbf{Step 4: The Geometry of the Flow Map.} The ODE vector field $\mathbf{f}(\mathbf{x}, t)$ is a linear combination of the score field $\mathbf{s}_{\theta^*}(\mathbf{x})$ and the position vector $\mathbf{x}$. The geometry of the resulting flow map $H_{\theta^*}$ depends on this vector field's geometry. \textbf{Result for Case (i):} In the isotropic case, $\mathbf{f}(\mathbf{x}, t)$ is also a radial vector field. A flow generated by a radial vector field is, by its rotational symmetry, necessarily a \textbf{conformal map}, as it preserves angles. \textbf{Result for Case (ii):} In the ellipsoidal case, the flow must transform the isotropic latent space into the anisotropic data space. The simplest such transformation is a \textbf{local affine transformation}, involving direction-dependent scaling and rotation. While this geometric argument is standard, a fully rigorous proof would require directly computing the Jacobian of the integrated flow map $H_{\theta^*}$ to verify it satisfies the required mathematical conditions. The model's inductive bias thus compels it to learn the most parsimonious geometric transformation necessary to explain the local data geometry, promoting the well-behaved, locally invertible mappings that are crucial for abduction.

\section{Experimental Details}\label{sec:appendix-exp-details}
\renewcommand{\theequation}{J.\arabic{equation}}\setcounter{equation}{0}

This appendix provides comprehensive details for all experiments presented in Section~\ref{sec:experiments}, ensuring full transparency and reproducibility.

\subsection{General Setup}

\paragraph{Software Environment.}
All experiments were conducted in a unified software environment to ensure full reproducibility. Key library versions used were: \texttt{dowhy} (0.12), \texttt{econml} (0.16.0), \texttt{numpy} (1.26.4), \texttt{pandas} (1.3.5), \texttt{scikit-learn} (1.6.1), \texttt{torch} (1.10.0), \texttt{lightgbm} (4.6.0), and \texttt{networkx} (3.2.1). All stochastic processes were controlled with a fixed global random seed (42), except for the ensemble runs, which used a set of distinct seeds for training.

\paragraph{Baseline Estimator Configurations.}
All baseline estimators were implemented using the \texttt{dowhy} library. For machine learning-based methods like Causal Forest and DML, we utilized their robust implementations from the \texttt{econml} library. To ensure a fair and reproducible comparison against established benchmarks, we used their default hyperparameter settings, which are widely recognized and have been optimized by the library authors to provide strong performance across a broad range of tasks. Our proposed model, in turn, underwent a systematic grid search; the final parameters and a detailed analysis are provided in Appendix \ref{sec:appendix-hyperparam-analysis}.

\paragraph{Details for PSM Failure Scenario (Act I).}
The Data Generation Process (DGP) for this experiment ($N=5000$, true ATE $\tau=5000$) follows the graph in Figure~\ref{fig:sub_dag_psm_fail} and is defined by:
\begin{align*}
    W_1, W_2 &\sim \mathcal{N}(0, 1) \\
    C_1 &\sim \text{Categorical}(\text{softmax}(\mathbf{z})), \quad \text{where} \quad
    \begin{cases}
        z_A = W_1 - W_2 \\
        z_B = \cos(\pi W_1) + \sin(\pi W_2) \\
        z_C = W_1^2 - W_2^2
    \end{cases} \\
    \text{logit}(P(T=1)) &= 2\sin(\pi W_1) + 1.5 W_2^2 + 2W_1 W_2 - 1.5 \cdot \mathbb{I}[C_1=A] + 2.5 \cdot \mathbb{I}[C_1=B] + U_T \\
    Y &= 5000 \cdot T + 60 \cdot (15 W_1 - 25 W_2 + 10 W_1 W_2) \\
    &\quad + 60 \cdot (-40 \cdot \mathbb{I}[C_1=A] + 50 \cdot \mathbb{I}[C_1=C]) + U_Y
\end{align*}
with exogenous noises $U_T \sim \text{Logistic}(0, 1)$ and $U_Y \sim \mathcal{N}(0, 6000^2)$.

\paragraph{Details for Lalonde Benchmark (Act I).}
To evaluate performance on real-world data, we used the canonical Lalonde dataset, analyzing the effect of the NSW job training program (\texttt{treat}) on 1978 real earnings (\texttt{re78}). The assumed causal structure is the standard confounding model (Figure~\ref{fig:sub_dag_lalonde}), a DAG structure that reflects the broad consensus in the causal inference community for this benchmark. In line with our Targeted Modeling principle, we applied the expressive \texttt{CausalDiffusionModel} only to the key \texttt{treat} and \texttt{re78} nodes, modeling all confounders non-parametrically via their empirical distributions.

\paragraph{Details for Semi-Synthetic Analysis (Act II).}
This dataset uses the real-world covariates from the Lalonde dataset as a foundation. The outcome $Y$ and the ground-truth Individual Treatment Effect ($ITE_{\text{true}}$) are then synthetically generated according to the following structural equations:
\begin{align*}
    Y_{\text{base}} = & \ 2 \cdot X_{\text{re74}} + 1.5 \cdot X_{\text{re75}} + 100 \cdot X_{\text{educ}} - 50 \cdot X_{\text{age}} \\ 
    & + 2000 \cdot X_{\text{black}} - 1000 \cdot X_{\text{hisp}} + U_{\text{base}} \\
    \text{ITE}_{\text{true}} = & \ 1500 + 350 \log(1+X_{\text{educ}}) - 3(X_{\text{age}} - 40)^2 \\
    & + 1200 \cdot (1 - X_{\text{nodegr}}) \cdot (1 - X_{\text{black}}) \\
    & - 1000 \tanh\left(\frac{X_{\text{re74}} - \mu_{\text{re74}}}{1000}\right) \\
    Y = & \ Y_{\text{base}} + \text{ITE}_{\text{true}} \cdot X_{\text{treat}}
\end{align*}
where $U_{\text{base}} \sim \mathcal{N}(0, 500^2)$ and $\mu_{\text{re74}}$ is the mean of the \texttt{re74} covariate.

\paragraph{Details for the Stress Test on Non-Invertible SCMs (Act IV).}
This experiment's primary objective is to evaluate the framework's robustness when the core theoretical assumption of SCM invertibility is explicitly violated. The DGP ($N=2000$) is defined as follows:
\begin{align*}
    W &\sim \mathcal{U}(-2, 2) \\
    P(T=1|W) &= \sigma(W + 0.5W^2) \quad (\text{where } \sigma(\cdot) \text{ is the sigmoid function}) \\
    U_Y &\sim \mathcal{N}(0, 1.5^2),\quad 
    Y = 5T + 2W + U_Y^2
\end{align*}
The true ATE is exactly $\tau=5.0$. The SCM was structured with \texttt{W} as an \texttt{Empirical\-Distribution}, and \texttt{T} and \texttt{Y} as \texttt{Causal\-Diffusion\-Model}.

\paragraph{Details for the Ablation Study (Act IV).}
The ablation study was conducted on a challenging synthetic mediation dataset ($N=4000$) designed to highlight the benefits of generative models. The causal graph is shown in Figure~\ref{fig:sub_dag_advanced_mediation}, with the following data generation process:
\begin{align*}
    X_1 &\sim \mathcal{N}(0, 1), \quad X_2 \sim \mathcal{U}(-2, 2), \quad Z \sim \text{Bernoulli}(0.5) \\
    T &\sim \text{Bernoulli}\left(\sigma\left(2.0 \sin(\pi X_1) X_2 - 1.5 Z + U_T\right)\right) \\
    M &= 5 \tanh(X_2) + \mathbb{I}[T=1](15 \cos(2\pi X_2) + 5X_1) + \mathbb{I}[T=0](-10|X_1|) + U_M \\
    Y &= 25 M + 10 \operatorname{sinc}(2 X_1) + U_Y
\end{align*}
where $\sigma(\cdot)$ is the sigmoid function, $U_T \sim \text{Logistic}(0, 0.3)$, $U_M \sim \mathcal{N}(0, 1.5^2)$, and $U_Y$ is drawn from a Gaussian Mixture Model conditioned on $Z$. The true ATE for this DGP is approximately $202.29$.
\begin{description}[style=sameline, leftmargin=0pt]
    \item[BELM-MDCM (Full Model):] The complete proposed framework. The nodes \texttt{T}, \texttt{M}, and \texttt{Y} are all modeled by a \texttt{CausalDiffusionModel} with \texttt{sampler\_type='belm'} and a hybrid objective weight of $\lambda=5.0$.

    \item[w/o Analytical Invertibility:] Identical to the full model, but the \texttt{sampler\_type} for all diffusion models was set to \texttt{'ddim'}.

    \item[w/o Hybrid Objective:] Identical to the full model, but the hybrid objective weight $\lambda$ for all diffusion models was set to $0.0$.

    \item[w/o Targeted Modeling:] The key mediator node \texttt{M} was modeled with a simpler \texttt{gcm.\hspace{0pt}AdditiveNoiseModel} (backed by an \texttt{LGBM\-Regressor}), while \texttt{T} and \texttt{Y} remained as \texttt{Causal\-Diffusion\-Model}s with $\lambda=5.0$ and \texttt{sampler\_type='belm'}.
\end{description}

\subsection{Model Hyperparameter Justification}\label{sec:appendix-hyperparam-analysis}
The hyperparameters for our BELM-MDCM model, reported in Table~\ref{tab:all-hyperparams}, were identified through a systematic grid search for each experiment. The variation in these parameters reflects principled adaptations to different data characteristics, guided by the trade-off between generative fidelity and discriminative accuracy, as well as the signal-to-noise ratio (SNR) of the underlying causal relationships. Below, we provide a holistic analysis for each scenario.

\subsubsection{Hyperparameter Search Details}
To ensure full reproducibility, we detail the hyperparameter search space used to arrive at the configurations in Table \ref{tab:all-hyperparams}. The final values for each experiment were selected based on the best performance on a held-out validation set, typically comprising 20-30\% of the training data. The primary selection criterion was the lowest PEHE score for experiments with a ground-truth ITE (e.g., Semi-Synthetic), or the best balance of low absolute ATE error and low estimation variance for observational data scenarios. For the Lalonde experiment in particular, we prioritized a configuration that demonstrated superior stability, as detailed in the analysis below.

\begin{table}[htbp]
    \small
    \centering
    \caption{Hyperparameter Search Space and Selection Criteria.}
    \label{tab:hyperparam_search}
    \begin{tabularx}{\linewidth}{@{} p{0.3\linewidth} >{\raggedright\arraybackslash}X m{0.25\linewidth} @{}}
        \toprule
        \textbf{Hyperparameter} & \textbf{Search Space (Grid Search)} & \textbf{Selection Criterion} \\
        \midrule
        
        Hybrid Weight $\lambda$ & 
        \{0.0, 0.1, 0.3, 1.0, 2.0, 5.0, 10.0\} & 
        \multirow{4}{=}{Best balance of low error and low variance on the validation set.} \\ 
        \addlinespace 
        
        Guidance Weight $w$ & 
        \{0.0, 0.1, 0.2, 0.3, 1.0, 5.0, 10.0\} & 
        \\ 
        \addlinespace 
        
        Diffusion Timesteps $T$ & 
        \{50, 100, 200, 500\} & 
        \\
        \addlinespace

        Learning Rate &
        \{5e-5, 1e-4, 1.1e-4, 1.2e-4, 2e-4\} &
        \\
        
        \bottomrule
    \end{tabularx}
\end{table}

\begin{table}[htbp]
    \small
    \centering
    \caption{Consolidated BELM-MDCM Hyperparameters for All Key Experiments. Column headers correspond to the following experiments: \textbf{PSM} (PSM Failure), \textbf{Lalonde}, \textbf{Semi-Synth} (Semi-Synthetic), \textbf{Ablation} (Ablation Study), and \textbf{Stress Test} (Act IV).}
    \label{tab:all-hyperparams}
    \begin{tabular*}{\linewidth}{@{\extracolsep{\fill}}lccccc@{}}
        \toprule
        \textbf{Hyperparameter} & 
        \textbf{PSM} & 
        \textbf{Lalonde} & 
        \textbf{Semi-Synth} &
        \textbf{Ablation} &
        \textbf{Stress Test} \\
        \cmidrule(lr){2-6}
        Number of Epochs & 1500 & 1000 & 1200 & 700 & 500 \\
        Batch Size & 128 & 64 & 64 & 128 & 128 \\
        Hidden Dimension & 512 & 512 & 768 & 768 & 256 \\
        Learning Rate & $1 \times 10^{-4}$ & $1 \times 10^{-4}$ & $1.1 \times 10^{-4}$ & $1 \times 10^{-4}$ & $1 \times 10^{-4}$ \\
        Diffusion Timesteps ($T$) & 200 & 200 & 50 & 200 & 200 \\ 
        Hybrid Weight $\lambda$ & 0.1 & 2.0 & 2.0 & 5.0 & 0.5 \\
        Guidance Weight ($w$) & 0.0 & 1.0 & 0.1 & 0.2 & 0.0 \\
        \bottomrule
    \end{tabular*}
\end{table}

\paragraph{Scenario 1: The "Perfectionist Learner" (PSM Failure Experiment)}
\textbf{Data Profile:} This synthetic dataset features complex, non-linear functions (e.g., \texttt{sin}, \texttt{cos}) but is characterized by a \textbf{pure, high-SNR signal}. The main challenge is to perfectly learn this intricate generative process. \textbf{Hyperparameter Strategy:} The strategy prioritizes generative modeling. A low hybrid weight ($\lambda=0.1$) directs the model to focus on the diffusion loss. Since the conditional signal is strong, classifier-free guidance is disabled ($w=0.0$). A large model capacity and extended training are employed to capture the ground-truth functions.

\paragraph{Scenario 2: The "Pragmatic Signal Extractor" (Lalonde Experiment)}
\textbf{Data Profile:} This real-world data is characterized by a \textbf{weak, noisy signal} and a small sample size. The primary challenge is to robustly extract the causal signal. \textbf{Hyperparameter Strategy:} The priority shifts from pure accuracy to a balance of accuracy and stability. A high hybrid weight ($\lambda=2.0$) provides a strong inductive bias towards the predictive task. Crucially, our systematic grid search revealed that high guidance weights ($w > 1.0$) dramatically increased estimation variance, leading to unreliable results. We therefore selected a moderate guidance weight of $w=1.0$. This configuration provides a stabilizing effect, guiding the model towards the causal signal without amplifying noise, thereby achieving the optimal balance between accuracy and the robustness required for reliable inference on real-world data.

\paragraph{Scenario 3: The "Precision Artist" (Semi-Synthetic Experiment)}
\textbf{Data Profile:} This setup presents a \textbf{hybrid-SNR} environment, using noisy real-world covariates but a pure, synthetic outcome function. The task demands high precision. \textbf{Hyperparameter Strategy:} The diffusion timesteps are reduced ($T=50$), plausibly because a shorter generative path better preserves the fine-grained details of the synthetic ITE function. The hybrid weight is high ($\lambda=2.0$) to focus on the estimation task, while guidance is light ($w=0.1$) as the ITE signal is cleaner than in the purely observational Lalonde case.

\paragraph{Scenario 4: The "Component Analyst" (Ablation Study)}
\textbf{Data Profile:} This is a complex, synthetic mediation structure with a \textbf{clean, high-SNR signal}, specifically designed to isolate the performance impact of individual framework components. \textbf{Hyperparameter Strategy:} The strategy is to create a strong, stable baseline. A high hybrid weight ($\lambda=5.0$) heavily orients the model towards the predictive task, making any performance degradation from ablations more pronounced. A large model capacity (\texttt{Hidden Dim=768}) ensures the model can learn the complex functions, while light guidance ($w=0.2$) provides a minor stabilizing effect.

\paragraph{Scenario 5: The "Robustness Tester" (Stress Test)}
\textbf{Data Profile:} A simple DGP featuring a non-invertible causal mechanism ($Y \propto U^2$), designed to test the framework's behavior when its core assumption is violated. \textbf{Hyperparameter Strategy:} The goal is stable learning of a simple function. A moderate hybrid weight ($\lambda=0.5$) balances generative and predictive learning, while a smaller model (\texttt{Hidden Dim=256}) is sufficient for the simpler DGP and helps prevent overfitting. Guidance is turned off ($w=0.0$) as the signal is clean.

\paragraph{Holistic Comparison}
\begin{table}[H]
    \centering\small
    \caption{Summary of Adaptive Hyperparameter Strategies.}
    \label{tab:strategy_summary}
    \begin{tabularx}{\linewidth}{@{} l *{5}{>{\centering\arraybackslash}X} @{}}
        \toprule
        \textbf{Aspect} & \textbf{PSM Failure} & \textbf{Lalonde} & \textbf{Semi-Synthetic} & \textbf{Ablation Study} & \textbf{Stress Test} \\
        \cmidrule(r){2-6}
        Data Signal & Pure \& Complex & Noisy \& Weak & Hybrid \& Complex & Pure \& Mediated & Pure \& Non-Invertible \\ 
        \midrule
        Guidance ($w$) & 0.0 (Off) & 1.0 (Balanced Guidance) & 0.1 (Slight Nudge) & 0.2 (Light) & 0.0 (Off) \\ 
        \midrule
        Timesteps ($T$) & 200 (Standard) & 200 (Standard) & 50 (Short Path) & 200 (Standard) & 200 (Standard) \\ 
        \midrule
        Hybrid ($\lambda$) & 0.1 (Generative) & 2.0 (Predictive) & 2.0 (Predictive) & 5.0 (Strongly Predictive) & 0.5 (Balanced) \\
        \midrule
        \textbf{Core Strategy} & \textbf{Perfect Generation} & \textbf{Robust Prediction} & \textbf{Precision Estimation} & \textbf{Component Isolation} & \textbf{Stable Learning} \\
        \bottomrule
    \end{tabularx}
\end{table}

In conclusion, the variability in hyperparameters demonstrates the flexibility of our framework. It showcases the model's ability to deploy different tools—such as prioritizing the generative loss or amplifying weak signals with guidance—to optimally adapt to the specific challenges posed by diverse causal inference problems.

\subsection{CMF Metric Implementation Details}
    To ensure the rigor and reproducibility of our evaluation, this section provides the implementation details for the CMF scores.

    \paragraph{CMI-Score Estimation Method.}
    For Conditional Mutual Information (CMI) estimation, we adopt the widely-used k-Nearest Neighbors (k-NN) based estimator from Kraskov et al. \citep{kraskov2004estimating}. This method is chosen for its strong theoretical properties and practical robustness, particularly for continuous and high-dimensional data, as it avoids explicit density estimation. Its key advantages include:
    \begin{itemize}[noitemsep]
        \item \textbf{Non-parametric:} It makes no assumptions about the underlying data distributions.
        \item \textbf{Data-adaptive:} The estimation is based on local distances, adapting to the data manifold's geometry.
        \item \textbf{Robustness:} It is more robust than methods that rely on fixed binning, which can be sensitive to bin size.
    \end{itemize}
    Following common practice, we set the number of neighbors to $k=5$ for all our experiments to ensure a stable and reliable estimation.

    \paragraph{KMD-Score Kernel Parameter Selection.}
    The performance of the MMD test is sensitive to kernel parameter choices. For the KMD-Score, we employed a standard Radial Basis Function (RBF) kernel, $k(x, y) = \exp(-\|x-y\|^2 / (2\sigma^2))$. The bandwidth parameter $\sigma$ is critical. Following best practices, we set the bandwidth using the median heuristic, a robust and common data-driven approach. For the analysis on the Lalonde dataset's \texttt{re78} mechanism, this heuristic yielded a bandwidth of $\sigma=0.1$, a value confirmed as effective in preliminary experiments. All KMD-Scores are reported using this configuration.

\section{Empirical Validation of Proposed Evaluation Metrics}\label{sec:appendix_metric_validation}

To empirically validate the reliability, sensitivity, and complementary nature of our proposed evaluation metrics (CIC-Score, CMI-Score, and KMD-Score), we conducted a controlled micro-simulation study assessing how each metric responds to a spectrum of increasingly severe model errors.

\paragraph{Experimental Setup.}
    We designed a simple ground-truth Structural Causal Model (SCM) and created five simulated models (A through E) representing a clear "degradation gradient" of model quality:
    \begin{itemize}[noitemsep, topsep=0pt]
        \item \textbf{Model A (Oracle):} Represents a theoretically perfect model with zero SRE and a perfectly learned causal mechanism. This serves as our gold standard, with expected scores of 1.0.
        \item \textbf{Model B (Lossy Inverter):} Simulates a model with a non-zero SRE. It uses the correct causal mechanism but introduces a systematic error during the reconstruction cycle, mimicking the core flaw of DDIM-based approaches.
        \item \textbf{Model C (Wrong Mechanism):} Represents a model that fails to learn the correct functional form of the causal mechanism (e.g., learning a linear instead of a sinusoidal relationship).
        \item \textbf{Model D (Maximal Error):} Simulates a more severe failure where both the causal mechanism and the inferred noise distribution are fundamentally incorrect.
        \item \textbf{Model E (Total Mismatch):} Represents the worst-case scenario where the model ignores the causal graph and inputs, generating outputs from an unrelated distribution.
    \end{itemize}

\begin{figure}[htbp]
    \centering
    \includegraphics[width=\textwidth]{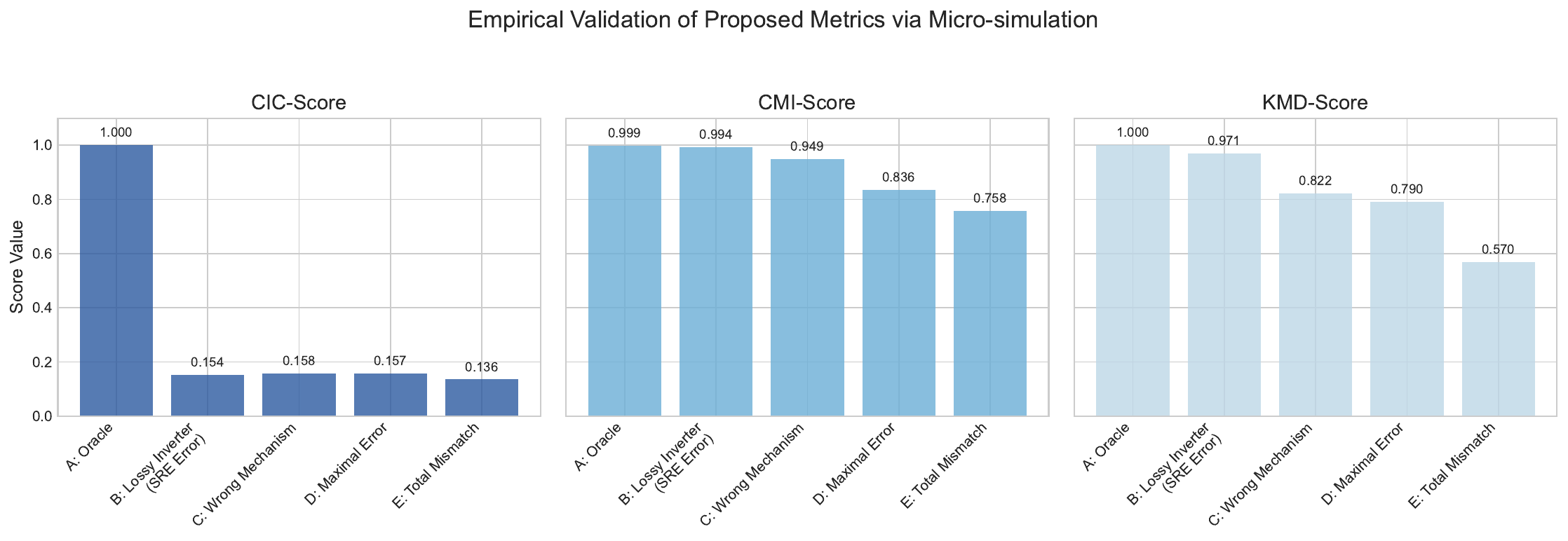}
    \caption{Results of the micro-simulation study for metric validation. The three plots show the response of the CIC-Score, CMI-Score, and KMD-Score to five models (A-E) of progressively decreasing quality. The scores demonstrate a clear monotonic degradation, confirming their ability to reliably track model fidelity. Note the CIC-Score's sharp drop from Model A to B, highlighting its specific sensitivity to the Structural Reconstruction Error (SRE).}
    \label{fig:metric_validation}
\end{figure}

\paragraph{Analysis of Results.}
    The results (Figure \ref{fig:metric_validation}) demonstrate the distinct and complementary roles of our proposed metrics:

    \begin{enumerate}[noitemsep, topsep=0pt]
        \item \textbf{The CIC-Score acts as a high-sensitivity "SRE detector."} It exhibits a dramatic drop from a perfect 1.0 (Model A) to approximately 0.23 (Model B) the moment SRE is introduced, while showing less sensitivity to the specific form of mechanism error. This confirms its primary role as a diagnostic for adherence to the Causal Information Conservation principle.
        
        \item \textbf{The CMI-Score serves as a robust "mechanism association tracker."} It degrades gracefully and monotonically as the learned causal mechanism deviates from the ground truth (from ~0.99 to ~0.76). This demonstrates its utility in quantifying the fidelity of learned parent-child conditional dependencies.
        
        \item \textbf{The KMD-Score functions as the "final arbiter" of distributional fidelity.} Possessing the widest dynamic range, it is sensitive to all forms of error and provides a holistic judgment of the similarity between the generated and true counterfactual distributions. As the most rigorous metric, it correctly assigns the lowest score to the completely mismatched Model E.
    \end{enumerate}
    This simulation thus validates our proposed metrics as a reliable and nuanced evaluation framework. They work in synergy to diagnose specific model failings (CIC-Score), assess relational accuracy (CMI-Score), and provide an overall quality judgment (KMD-Score), offering a more insightful assessment than traditional metrics alone.

\paragraph{Discussion on the Non-Zero Lower Bound of Scores.}
    Notably, even for the worst-performing model (Model E), the CMI and KMD scores do not fall to zero. This behavior is not a limitation but a desirable feature that reflects their ability to capture "residual statistical structure" in the evaluation setting.

    \begin{itemize}[noitemsep, topsep=0pt]
        \item \textbf{For the KMD-Score:} The MMD compares the joint distributions $P_{\text{model}}(Y, W, T)$ and $P_{\text{oracle}}(Y, W, T)$. Crucially, as all models operate on the same observed parent data, the marginal parent distribution $P(W, T)$ is identical between the two. The difference lies only in the conditional $P(Y|W, T)$. Because the joint distributions share a substantial common subspace, their MMD will be finite, preventing the KMD-Score from reaching zero. Furthermore, using a `StandardScaler` maps both distributions to a similar feature space, which is necessary for a fair, scale-invariant comparison.
        
        \item \textbf{For the CMI-Score:} This metric quantifies the $I(Y; \text{parent} | \text{other parents})$. Even an incorrect mechanism (like in Model C or D) still generates an output $Y$ as a deterministic function of its parents, resulting in a non-zero CMI. For Model E, where $Y$ is independent of its parents, the theoretically zero CMI is not observed due to the inherent finite-sample variance of the non-parametric k-NN estimator.
    \end{itemize}
    This behavior is advantageous, ensuring the metrics provide a meaningful, continuous gradient of failure rather than a simplistic binary judgment. This enhances their diagnostic power, allowing for fine-grained distinctions between types and degrees of model imperfection.
\section{Algorithm for ATE Estimation}\label{sec:appendix_algorithm}

This appendix provides the detailed pseudo-code for the counterfactual imputation procedure used to estimate the Average Treatment Effect (ATE) in our experiments.

\begin{algorithm}[htbp]
    \small
    \caption{ATE Estimation with Invertible SCMs via Counterfactual Imputation}\label{alg:ate_estimation_appendix}
    
    \KwIn{Observational data $\mathcal{D} = \{v^{(j)}\}_{j=1}^N$ where $v^{(j)} \in \mathbb{R}^d$; Causal graph $\mathcal{G}$.}
    \KwOut{Estimated Average Treatment Effect ($\hat{\text{ATE}}$).}
    
    \SetKwProg{Def}{Define}{:}{}
    \Def{Invertible SCM $\mathcal{M}_\theta$}{
        $\mathcal{M}_\theta = \{f_i(\cdot; \theta_i)\}_{i=1}^d$ based on $\mathcal{G}$, where $v_i = f_i(\mathbf{pa}_i, u_i)$\;
    }

    \tcp{1. Train the invertible SCM on observational data}
    $\hat{\theta} \leftarrow \arg\min_{\theta} \sum_{j=1}^{N} \sum_{i=1}^{d} \mathcal{L}_i\left(f_i(\mathbf{pa}_i^{(j)}; \theta_i), v_i^{(j)}\right)$\;

    \tcp{2. Generate counterfactual outcomes for each individual}
    \For{$j \leftarrow 1$ to $N$}{
        $t_j \leftarrow v_T^{(j)}$ \tcp*{Observed treatment}
        $\mathbf{u}^{(j)} \leftarrow \mathcal{M}_{\hat{\theta}}^{-1}(v^{(j)})$ \tcp*{Abduction via invertible BELM encoder}
        $y_j(1-t_j) \leftarrow \operatorname{Predict}\left(\mathcal{M}_{\hat{\theta}}, \mathbf{u}^{(j)}, \operatorname{do}(T:=1-t_j)\right)$ \tcp*{Action \& Prediction}
        
        \tcp{Store factual and counterfactual outcomes}
        $\mathbf{Y}_j(t_j) \leftarrow v_Y^{(j)}$\;
        $\mathbf{Y}_j(1-t_j) \leftarrow y_j(1-t_j)$\;
    }

    \tcp{3. Compute the ATE from factual and counterfactual outcomes}
    $\hat{\text{ATE}} \leftarrow \frac{1}{N} \sum_{j=1}^{N} \left( \mathbf{Y}_j(1) - \mathbf{Y}_j(0) \right)$\;
    \BlankLine

    \Return{$\hat{\text{ATE}}$}\;
\end{algorithm}

\vskip 0.2in
\bibliography{sample}

@book{angrist2008mostly,
  title={Mostly Harmless Econometrics: An Empiricist's Companion},
  author={Angrist, Joshua D. and Pischke, J{\"o}rn-Steffen},
  year={2008},
  publisher={Princeton University Press}
}

@inproceedings{bartlett2017spectrally,
  title     = {Spectrally-normalized margin bounds for neural networks},
  author    = {Bartlett, Peter L and Foster, Dylan J and Telgarsky, Matus},
  booktitle = {Advances in Neural Information Processing Systems},
  volume    = {30},
  year      = {2017}
}

@article{chao2023interventional,
  title   = {Interventional and Counterfactual Inference with Diffusion Models},
  author  = {Chao, Patrick and Bl{\"o}baum, Patrick and Kasiviswanathan, Shiva Prasad},
  journal = {arXiv preprint arXiv:2302.00860},
  year    = {2023}
}

@article{dinh2014nice,
  title   = {{NICE}: Non-Linear Independent Components Estimation},
  author  = {Dinh, Laurent and Krueger, David and Bengio, Yoshua},
  journal = {arXiv preprint arXiv:1410.8516},
  year    = {2014}
}

@inproceedings{dinh2017density,
  title     = {Density Estimation using Real {NVP}},
  author    = {Dinh, Laurent and Sohl-Dickstein, Jascha and Bengio, Samy},
  booktitle = {International Conference on Learning Representations (ICLR)},
  year      = {2017}
}

@inproceedings{goodfellow2014generative,
  title     = {Generative Adversarial Nets},
  author    = {Goodfellow, Ian and Pouget-Abadie, Jean and Mirza, Mehdi and Xu, Bing and Warde-Farley, David and Ozair, Sherjil and Courville, Aaron and Bengio, Yoshua},
  booktitle = {Advances in Neural Information Processing Systems},
  volume    = {27},
  year      = {2014}
}

@article{gretton2012kernel,
  title   = {A Kernel Two-Sample Test},
  author  = {Gretton, Arthur and Borgwardt, Karsten M. and Rasch, Malte J. and Sch{\"o}lkopf, Bernhard and Smola, Alexander},
  journal = {Journal of Machine Learning Research},
  volume  = {13},
  number  = {25},
  pages   = {723--773},
  year    = {2012}
}

@book{hairer2006solving,
  title     = {Solving Ordinary Differential Equations II: Stiff and Differential-Algebraic Problems},
  author    = {Hairer, Ernst and Wanner, Gerhard},
  year      = {2006},
  publisher = {Springer Science \& Business Media},
  series    = {Springer Series in Computational Mathematics},
  volume    = {14},
  edition   = {2nd}
}

@inproceedings{he2016deep,
  title     = {Deep Residual Learning for Image Recognition},
  author    = {He, Kaiming and Zhang, Xiangyu and Ren, Shaoqing and Sun, Jian},
  booktitle = {Proceedings of the IEEE Conference on Computer Vision and Pattern Recognition (CVPR)},
  pages     = {770--778},
  year      = {2016}
}

@article{heckman2001micro,
  title={Micro data, heterogeneity, and the evaluation of public policy: Nobel lecture},
  author={Heckman, James J},
  journal={Journal of Political Economy},
  volume={109},
  number={4},
  pages={673--748},
  year={2001},
  publisher={The University of Chicago Press}
}

@inproceedings{ho2020denoising,
  title     = {Denoising Diffusion Probabilistic Models},
  author    = {Ho, Jonathan and Jain, Ajay and Abbeel, Pieter},
  booktitle = {Advances in Neural Information Processing Systems},
  volume    = {33},
  pages     = {6840--6851},
  year      = {2020}
}

@article{hochreiter1997flat,
  title     = {Flat Minima},
  author    = {Hochreiter, Sepp and Schmidhuber, J{\"u}rgen},
  journal   = {Neural Computation},
  volume    = {9},
  number    = {1},
  pages     = {1--42},
  year      = {1997},
  publisher = {MIT Press}
}

@inproceedings{hoyer2009nonlinear,
  title     = {Nonlinear Causal Discovery with Additive Noise Models},
  author    = {Hoyer, Patrik O. and Janzing, Dominik and Mooij, Joris M. and Peters, Jonas and Sch{\"o}lkopf, Bernhard},
  booktitle = {Advances in Neural Information Processing Systems},
  volume    = {21},
  year      = {2009}
}

@inproceedings{kingma2014autoencoding,
  title     = {Auto-Encoding Variational {B}ayes},
  author    = {Kingma, Diederik P. and Welling, Max},
  booktitle = {International Conference on Learning Representations (ICLR)},
  year      = {2014}
}

@inproceedings{kingma2018glow,
  title     = {{Glow}: Generative Flow with Invertible 1x1 Convolutions},
  author    = {Kingma, Diederik P. and Dhariwal, Prafulla},
  booktitle = {Advances in Neural Information Processing Systems},
  volume    = {31},
  year      = {2018}
}

@article{kraskov2004estimating,
  title   = {Estimating mutual information},
  author  = {Kraskov, Alexander and St{\"o}gbauer, Harald and Grassberger, Peter},
  journal = {Physical Review E},
  volume  = {69},
  number  = {6},
  pages   = {066138},
  year    = {2004},
  publisher = {APS}
}

@article{lalonde1986evaluating,
  title     = {Evaluating the Econometric Evaluations of Training Programs with Experimental Data},
  author    = {Lalonde, Robert J.},
  journal   = {The American Economic Review},
  volume    = {76},
  number    = {4},
  pages     = {604--620},
  year      = {1986},
  publisher = {JSTOR}
}

@article{liu2022pseudo,
  title   = {Pseudo Numerical Methods for Diffusion Models on Manifolds},
  author  = {Liu, Luping and Ren, Yi and Lin, Zhijie and Zhao, Zhou},
  journal = {arXiv preprint arXiv:2202.09778},
  year    = {2022}
}

@inproceedings{liu2024belm,
  title     = {{BELM}: Bidirectional Explicit-Form Linear Multi-Step Samplers for Diffusion Models},
  author    = {Liu, Zimeng and Li, Ziqiong and Liu, Jian and Su, Wei},
  booktitle = {The Twelfth International Conference on Learning Representations (ICLR)},
  year      = {2024}
}

@book{mohri2018foundations,
  title     = {Foundations of Machine Learning},
  author    = {Mohri, Mehdi and Rostamizadeh, Afshin and Talwalkar, Ameet},
  publisher = {The MIT Press},
  year      = {2018},
  edition   = {Second}
}

@inproceedings{neyshabur2018pac,
  title     = {A {PAC-Bayesian} approach to spectrally-normalized margin bounds for neural networks},
  author    = {Neyshabur, Behnam and Bhojanapalli, Srinadh and McAllester, David and Srebro, Nati},
  booktitle = {International Conference on Learning Representations (ICLR)},
  year      = {2018}
}

@book{pearl2009causality,
  title     = {Causality: Models, Reasoning and Inference},
  author    = {Pearl, Judea},
  year      = {2009},
  publisher = {Cambridge University Press},
  address   = {Cambridge},
  edition   = {Second}
}

@inproceedings{pearl2014transportability,
  title     = {Transportability of Causal and Statistical Relations: A Formal Approach},
  author    = {Pearl, Judea and Bareinboim, Elias},
  booktitle = {Proceedings of the Twenty-Eighth AAAI Conference on Artificial Intelligence},
  pages     = {1149--1157},
  year      = {2014}
}

@book{peters2017elements,
  title     = {Elements of Causal Inference: Foundations and Learning Algorithms},
  author    = {Peters, Jonas and Janzing, Dominik and Sch{\"o}lkopf, Bernhard},
  year      = {2017},
  publisher = {The MIT Press},
  series    = {Adaptive Computation and Machine Learning series}
}

@article{rubin1974estimating,
  title     = {Estimating Causal Effects of Treatments in Randomized and Nonrandomized Studies},
  author    = {Rubin, Donald B.},
  journal   = {Journal of Educational Psychology},
  volume    = {66},
  number    = {5},
  pages     = {688--701},
  year      = {1974},
  publisher = {American Psychological Association}
}

@article{sanchez2022dcms,
  title   = {{DCM}s: Diffusion Causal Models for Counterfactual Estimation},
  author  = {Sanchez, Pedro and Tsaftaris, Sotirios A.},
  journal = {arXiv preprint arXiv:2202.10166},
  year    = {2022}
}

@inproceedings{sharma2022dowhy,
  title     = {{DoWhy}: An End-to-End Library for Causal Inference},
  author    = {Sharma, Amit and Kiciman, Emre},
  booktitle = {Proceedings of the First Conference on Causal Learning and Reasoning},
  series    = {Proceedings of Machine Learning Research},
  volume    = {177},
  pages     = {734--747},
  year      = {2022},
  publisher = {PMLR}
}

@article{shimizu2006linear,
  title   = {A Linear, Non-Gaussian Acyclic Model for Causal Discovery},
  author  = {Shimizu, Shohei and Hoyer, Patrik O. and Hyv{\"a}rinen, Aapo and Kerminen, Antti},
  journal = {Journal of Machine Learning Research},
  volume  = {7},
  pages   = {2003--2030},
  year    = {2006}
}

@inproceedings{song2021denoising,
  author    = {Song, Jiaming and Meng, Chenlin and Ermon, Stefano},
  booktitle = {International Conference on Learning Representations (ICLR)},
  title     = {Denoising Diffusion Implicit Models},
  year      = {2021}
}

@inproceedings{song2021score,
  title     = {Score-Based Generative Modeling through Stochastic Differential Equations},
  author    = {Song, Yang and Sohl-Dickstein, Jascha and Kingma, Diederik P. and Kumar, Abhishek and Ermon, Stefano and Poole, Ben},
  booktitle = {International Conference on Learning Representations (ICLR)},
  year      = {2021}
}

@book{wooldridge2010econometric,
  title={Econometric Analysis of Cross Section and Panel Data},
  author={Wooldridge, Jeffrey M},
  year={2010},
  publisher={MIT press}
}

\end{document}